\documentclass{article}
\usepackage{times}
\usepackage{microtype}
\usepackage{graphicx}
\usepackage{subcaption}
\usepackage{booktabs}
\usepackage{hyperref}

\usepackage{configuration/iclr2026}

\iclrfinalcopy

\usepackage{amsmath,amssymb,amsthm}
\newtheorem{theorem}{Theorem}[section]

\newtheorem{definition}[theorem]{Definition}

\newtheorem{remark}{Remark}[section]
\newtheorem{assumption}{Assumption}

\DeclareMathOperator*{\argmin}{arg\,min\,}

\newenvironment{talign*}
{\csname align*\endcsname}
{\endalign}

\usepackage[utf8]{inputenc}
\usepackage[T1]{fontenc}
\usepackage{url}
\usepackage{booktabs}
\usepackage{amsfonts}
\usepackage{nicefrac}
\usepackage{microtype}
\usepackage{xcolor}
\usepackage{algorithm}
\usepackage{algorithmic}
\usepackage{graphicx}
\usepackage{subcaption}
\usepackage[flushleft]{threeparttable}
\usepackage{float}
\usepackage{multirow}
\usepackage{xspace}
\usepackage{natbib}
\usepackage{enumitem}
\usepackage{multicol}
\usepackage[font=small]{caption}
\usepackage{autobreak}
\usepackage{sidecap}
\usepackage{wrapfig}
\usepackage{bbding}
\usepackage[toc, page, header]{appendix}
\usepackage{tikz}
\usepackage{xcolor}
\usepackage{pifont}
\usepackage{mdframed}
\usepackage{xurl}
\usepackage{mathtools}
\usepackage{bm}
\usepackage{colortbl}

\usepackage[capitalize,noabbrev]{cleveref}
\usepackage{tcolorbox}
\tcbuselibrary{breakable, skins}
\newtcolorbox{takeaway}{
	enhanced jigsaw,
	breakable,
	colback=black!8,
	colframe=black!8,
	boxrule=0pt,
	borderline west={2.5pt}{0pt}{black!20},
	sharp corners,
	left=5pt,right=5pt,top=3pt,bottom=3pt,
	before skip=10pt,after skip=10pt,
}
\newtcolorbox{taskbox}{
	breakable,
	colback=blue!2,
	colframe=blue!20,
	boxrule=0.5pt,
	arc=0pt,
	outer arc=0pt,
	left=5pt,right=5pt,top=5pt,bottom=5pt,
	boxsep=0pt,
	width=0.98\columnwidth,
	center,
}

\tcolorboxenvironment{remark}{
	enhanced jigsaw,
	breakable,
	colback=blue!8,
	colframe=blue!8,
	boxrule=0pt,
	borderline west={2.5pt}{0pt}{blue!20},
	sharp corners,
	left=3pt,right=3pt,top=1pt,bottom=1pt,
	before skip=4pt,after skip=4pt,
}

\tcolorboxenvironment{definition}{
	enhanced jigsaw,
	breakable,
	colback=purple!8,
	colframe=purple!8,
	boxrule=0pt,
	borderline west={2.5pt}{0pt}{purple!20},
	sharp corners,
	left=3pt,right=3pt,top=1pt,bottom=1pt,
	before skip=4pt,after skip=4pt,
}

\makeatletter
\def\thm@space@setup{\thm@preskip=1pt plus 1pt minus 1pt
\thm@postskip=\thm@preskip
}
\makeatother

\hypersetup{
	colorlinks=true,
	linkcolor=blue,
	citecolor=blue,
	urlcolor=blue
}

\definecolor{coral}{RGB}{255,127,80}
\definecolor{darkgreen}{RGB}{0,100,0}
\definecolor{darkyellow}{RGB}{204,153,0}
\definecolor{salmon}{RGB}{250,128,114}
\definecolor{active_principles}{RGB}{133, 20, 44}
\definecolor{working_set}{RGB}{115, 23, 37}
\definecolor{anomalies}{RGB}{237, 71, 100}
\definecolor{coherent_augmentation}{RGB}{37, 87, 175}
\definecolor{true_principle}{RGB}{131, 107, 183}

\definecolor{superior}{HTML}{7D2AD1}
\definecolor{inferior}{HTML}{D14D4D}
\definecolor{headercolor}{gray}{0.90}
\definecolor{rowcolor}{gray}{0.97}

\newcommand{\method}{\textsc{CoEvolve}\xspace}
\newcommand{\methodm}{\textnormal{\scshape CoEvolve}}

\title{Collaborative Principle Evolution via Evidence Transfer for Scientific Discovery}

\author{Yingming Pu$^{1,2}$ \quad Hongyu Chen$^{2,*}$ \quad Tao Lin$^{2,*}$\\[3pt]
  \mdseries $^{1}$Zhejiang University, Hangzhou, Zhejiang, China\\
  \mdseries $^{2}$Westlake University, Hangzhou, Zhejiang, China\\
  \mdseries $^{*}$Corresponding authors \quad
  \texttt{\{puyingming, lintao\}@westlake.edu.cn}
}

\begin{document}

\maketitle

\begin{abstract}
Large Language Model (LLM)-based agents promise to automate scientific discovery, yet exploring the vast hypothesis space remains costly.
Existing principle-evolution methods accelerate this loop, but operate sequentially, which caps exploration breadth and wastes wall-clock time on challenging problems.
To address this, we formulate collaborative scientific discovery as \emph{evidence transfer} between parallel principle-evolution branches.
We present \textbf{\method}, which realizes this transfer through a coordination core over parallel branches. By integrating value-of-information-gated routing and context-discounted likelihood injection, \method enables branches to collaborate through shared measurements while keeping their principle posteriors separate.
Across six scientific-discovery tasks under a matched evaluation budget, \method attains a mean solution quality of $66.5\%$ versus $57.0\%$ for single-branch principle evolution, with a $1.80\times$ mean wall-clock speedup on the GPT-5.6-Terra backbone; on five auto-research tasks delegated to an autonomous research harness, it is the only arm whose mean stays above the published SOTA anchor on every task.
These results establish when evidence sharing accelerates parallel discovery and when transfer safeguards are necessary to limit negative or inert transfers.
\end{abstract}

\section{Introduction}
\label{sec:introduction}

\begin{wrapfigure}{r}{0.38\textwidth}
  \vspace{-15pt}
  \centering
  \includegraphics[width=\linewidth]{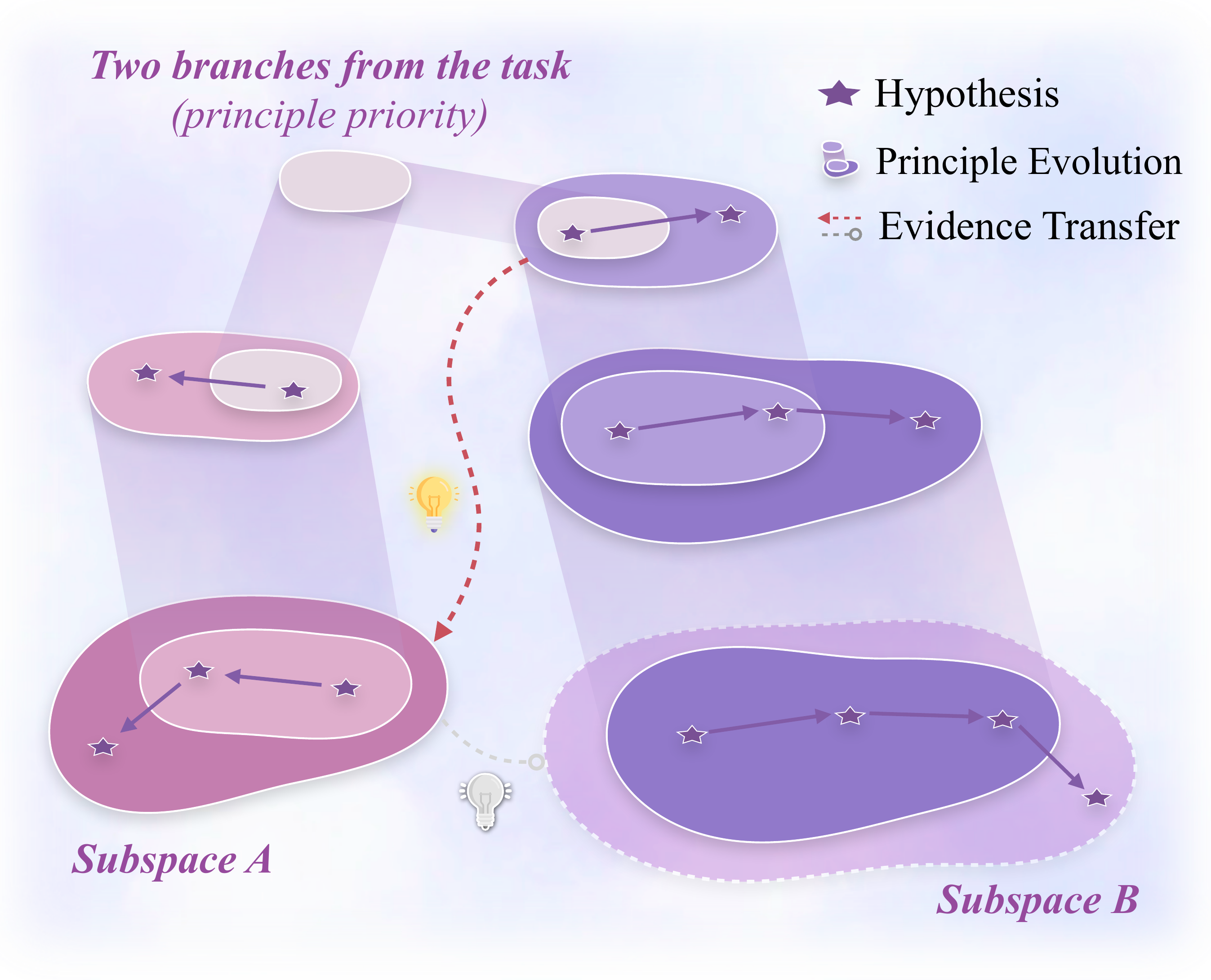}
  \vspace{-15pt}
  \caption{\textbf{Principle co-evolution.} Independent branches explore different principle subspaces through a hypothesis-testing loop. \method shares evidence, allowing the branches to collaborate while retaining separate search histories.}
  \label{fig:wrap_intro}
  \vspace{-18pt}
\end{wrapfigure}

The large language model (LLM)-based AI Scientist is often realized as a hypothesis-driven loop: an agent proposes a hypothesis, evaluates it, and designs the next from observations~\citep{Wei2025FromAF,Gridach2025AgenticAF,Herron2026FromRT,Boiko2023AutonomousCR,Mitchener2025KosmosAA}. While these systems automate discovery~\citep{Ghareeb2025RobinAM,Ghafarollahi2024SciAgentsAS,li2026agenticfusionlargeatomic}, each still operates alone. Yet real discovery is rarely solitary, so how could AI Scientists collaborate on the same problem?

Existing systems coordinate plans, roles, and natural-language states~\citep{Su2024ManyHA,Ghareeb2025RobinAM,Lyu2026EvoScientistTM,Feng2026InternAgent15AU,Xin2026EurekAgentAE,Tang2025AIResearcherAS}, leaving evidence-level transfer rules unspecified, whereas principle-aware methods~\citep{Pu2026PrincipleEvolvableSD,Pu2025PiFlowPS} adapt scientific principles as a structured medium, maintaining one context-specific evidence stream per cycle.
Naively extending this principle optimization to many branches invites failure: one branch's observation may duplicate an existing measurement, come from a different context, or need replication before its relevance is clear; principles then grow overconfident or noisy, and performance degrades. 
This creates a concrete requirement for existing systems: \emph{parallel AI Scientists must share measurements without treating redundant or context-shifted evidence as independent local evidence}.

These limitations culminate in three challenges: (a) concurrent branches must provide \textbf{complementary coverage}, not repetition; (b) transfers must reflect \textbf{source compatibility}; and (c) \textbf{wall-clock efficiency}: any speedup must exceed the cost of screening.

To bridge this gap, we propose \textbf{\method}, treating collaborative scientific discovery as an evidence transfer problem via \emph{principle co-evolution}~(Definition~\ref{def:principle_coevolution}), as shown in Figure~\ref{fig:wrap_intro}. Each branch retains its own principle posterior~\citep{Pu2026PrincipleEvolvableSD}. A \emph{coordination core} collects branch-level evidence, deduplicates and discounts each record, and routes them to the target branch for principle posterior updates. This \textsc{Collect--Merge--Score--Route--Inject} loop fosters a parallel exploration that maximizes the information gained from each branch based on information theory. Consequently, \textbf{\method} connects concurrent exploration to local updates while preserving branch autonomy.

\begin{definition}[Principle co-evolution]
\label{def:principle_coevolution}
    In this paper, we use the term \emph{principle co-evolution} to describe a parallel hypothesis-testing protocol in which $K\geq 2$ branches run independent principle-evolution from different sub-domains, over a shared Universal Principle Space~\citep{Pu2026PrincipleEvolvableSD}. The branches are coupled exclusively through a shared pool of tested evidence. Evidence may update a target branch only as a likelihood factor~(Section~\ref{subsec:theoretical_framework}). Branches thus co-evolve through shared measurements only, no global posterior or belief is formed.
\end{definition}

We evaluate \textbf{\method} on six scientific-discovery tasks and five autoresearch tasks. Empirical results show that \textbf{\method}:
\textbf{(a)} \textbf{attains the best average solution quality}, a mean of $66.5\%$ versus $57.0\%$ for PiEvo~\citep{Pu2026PrincipleEvolvableSD};
\textbf{(b)} \textbf{converts parallel throughput into quality}, completing the same evaluation budget $1.80\times$ faster than PiEvo on the GPT-5.6-Terra backbone with a $+8.3$ equal-time solution quality gap at a $0.97\times$ per-token conversion rate; and
\textbf{(c)} \textbf{generalizes to open-ended autoresearch}, standing as the only method above the published SOTA anchor on all five harness tasks (mean improvement $+16.7\%$ vs.\ PiEvo's $+5.8\%$), with the margin growing with task explorativeness: on the most explorative task, \method nearly triples the strongest baseline~\citep{Pu2026PrincipleEvolvableSD}'s improvement over the anchor ($+45.4\%$ vs.\ PiEvo's $+15.4\%$).

In summary, our contributions are:
\begin{enumerate}[nosep, leftmargin=16pt]
\item[\textbf{(a)}] We introduce the concept of \textbf{Principle Co-evolution}, transforming principle-evolving agents to parallel scientific discovery, widening exploration with high efficiency.
\item[\textbf{(b)}] We propose \textbf{\method}, a framework that evolves scientific principles via evidence transfer in parallel, sharing branch-level evidence and speeding up discovery.
\item[\textbf{(c)}] We conduct \textbf{extensive evaluations} across six closed-world scientific tasks and five autoresearch tasks, showing that \method attains the best average solution quality and is the only system above the published SOTA anchor on all autoresearch tasks, with the gains traced to disciplined evidence transfer rather than parallelism alone.
\end{enumerate}

\section{Related Work}
\label{sec:related_work}
\noindent \textbf{Agentic AI for scientific discovery.}
Recent surveys describe a shift from workflows to \emph{agentic science}, where LLMs plan studies, invoke domain tools, and iterate over hypotheses~\citep{Wei2025FromAF,Gridach2025AgenticAF,Herron2026FromRT}. 
End-to-end systems instantiate this loop at increasing autonomy and horizon: AI Researcher and DeepScientist explore solution spaces, Co-Scientist organizes deliberative agents with experimental validation, and EvoScientist, EvoSci, and InternAgent-1.5 maintain evolving populations or long-horizon research state~\citep{Tang2025AIResearcherAS,Weng2025DeepScientistAF,Gottweis2025AcceleratingSD,Lyu2026EvoScientistTM,Xiong2026EvoSciAB,Feng2026InternAgent15AU}. 
Other systems redesign the loop through scientific-method controllers, stateful workflows, environment engineering, or outer-loop optimization~\citep{Smith2026TheLS,Zhao2026ScienceFlowAL,Xin2026EurekAgentAE,Qu2026BilevelAM}. Specifically, PiEvo~\citep{Pu2026PrincipleEvolvableSD} casts discovery as Bayesian optimization over an expanding set of principles; its single evidence stream, however, confines exploration to one evolution. \method lifts this confinement: evidence validated in one branch becomes evidence for all.

\noindent \textbf{Language model-based collaboration.}
Multi-agent discovery systems distribute roles, plans, or reasoning across specialized LLMs: SciAgents~\citep{Ghafarollahi2024SciAgentsAS} and Robin~\citep{Ghareeb2025RobinAM} automate workflow-level coordination, \citet{Su2024ManyHA} improves idea generation through agent plurality, and domain systems such as GenoMAS~\citep{Liu2025GenoMASAM} and modular materials agents couple collaboration to code or task structure~\citep{Chaudhari2026ModularLL}. PiFlow~\citep{Pu2025PiFlowPS} further frames search as principle-guided exploration-exploitation balance, while EDAgent addresses concurrent orchestration and incremental state synchronization~\citep{Yuan2026EDAgentAC}. 
These systems show that coordination improves discovery diversity, but communication remains prompt-level: which records may cross agents, and at what weight, is left unspecified. \method leverages a record-level protocol: what crosses agents is a tested observation whose weight is computed.

\noindent \textbf{Evidence transfer in parallel search.}
Classical parallel search provides analytical primitives for controlling information flow. Batch and asynchronous Bayesian optimization reduce redundant acquisitions and provide regret guarantees~\citep{Ren2024TSRSRAP,Sugiura2026RandomizedKB}; safe optimization constrains admissible actions~\citep{Wei2024SafeBO}; physics-informed composite models structure prediction targets~\citep{Dong2026BayesianOO}; multi-output Gaussian processes characterize conditions for avoiding negative transfer~\citep{Li2020OnNT}; and island-model evolution uses controlled migration to preserve population diversity~\citep{ychowski2024CultivatingAO}. In LLM-based search, DeltaEvolve records structured semantic deltas, but only along the successive nodes of one lineage~\citep{Jiang2026DeltaEvolveAS}. 
These primitives control information flow but leave the experiment-level protocol open: an import may be dependent or require replication. \method supplies that protocol: observations cross branches only as gated, discounted imports, certified before they write into a posterior.

\section{Methodology}
\label{sec:methodology}

\subsection{Problem Setup and Notations}
\label{subsec:problem_setup}
We formulate collaborative scientific discovery as a \emph{principle co-evolution} problem (Definition~\ref{def:principle_coevolution}). 
Each branch $j$ explores a hypothesis space $\mathcal H_j$ and maintains a posterior over a shared finite principle universe $\bar{\mathcal P}$, where a principle $P$ models how hypotheses produce outcomes and $P^\star$ denotes the true principle~\citep{Pu2026PrincipleEvolvableSD}. At round $t$, branch $j$ selects $h\in\mathcal H_j$, evaluates it with the task validator, and observes $y\sim p_j(\cdot\mid h,P^\star)$; the tested pairs $(h,y)$ form the local evidence set $\mathcal D_{j,t}$. 

\noindent \textbf{\method leaves the branch-level loop unchanged and adds a \emph{coordination core}.} As shown in Figure~\ref{fig:overview}, each branch owns its tested pairs, publishes them as candidate records, and may admit other branches' records as imports; Definition~\ref{def:evidence_objects} delimits this \emph{evidence-only} interface, and \cref{subsec:theoretical_framework} formulates these rules.

\begin{definition}[Evidence objects and system history]
    \label{def:evidence_objects}
    A \emph{local observation} is a pair $(h,y)\in\mathcal D_{j,t}$ tested by branch $j$ itself. To share one, source branch $i$ publishes a \emph{candidate record}
    $
    C=(i,h,y,E_i,m)
    $
    carrying source context $E_i$ and metadata $m$. An \emph{import} for target branch $j$ is a record that passes $j$'s routing decision, with the discount $\alpha_{j\leftarrow i,t}(h)$ attached at scoring time; those accepted by round $t$ form $\mathcal I_{j,t}$. 
    Let $H_t^{\mathrm{sys}}$ denote the branch histories and evidence pool available at round $t$; discounts and routing decisions use only the history available before an import is injected.
\end{definition}

\begin{figure}[t!]
    \centering
    \includegraphics[width=0.954\linewidth]{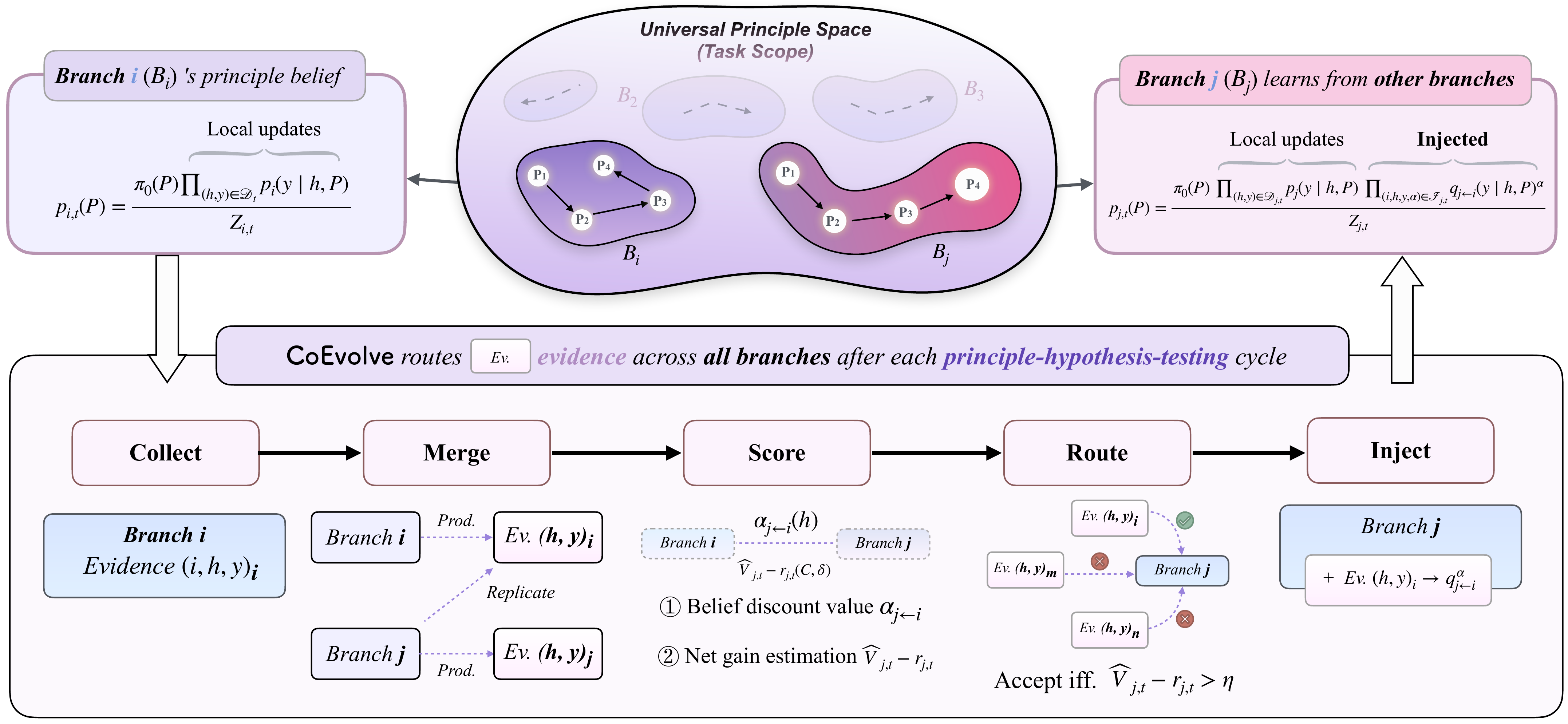}
    \vspace{-4pt}
    \caption{\textbf{Overview of \method.} After every hypothesis-testing cycle, the \emph{coordination core} (bottom) collects hypothesis--outcome pairs as evidence, merges them into a global evidence pool, scores each source--target pair, and routes and injects only evidence satisfying the gate. An accepted record enters the target branch as the discounted likelihood factor $q_{j\leftarrow i}(y\mid h,P)^\alpha$; rejected records remain in the evidence pool.}
    \label{fig:overview}
\end{figure}

\subsection{Theoretical Framework}
\label{subsec:theoretical_framework}
In \method, a target-specific router in the coordination core decides whether and with what weight a pooled record may update one branch. An accepted record enters the target posterior as the discounted tuple $(i,h,y,\alpha)$. Below, we derive in turn how a record is \emph{scored}, \emph{routed}, and \emph{injected} into the target posterior.

\noindent \textbf{Score: information value of a candidate record.} From first principles, the synergy between branches resides in their differences: \emph{an observation is worth transferring to branch $j$ precisely to the extent that it carries information branch $j$ does not already have}. The natural measure of this worth is the entropy reduction its update would induce in the target's principle posterior~\citep{Pu2026PrincipleEvolvableSD}, net of the transfer cost. Let $H(p)=-\sum_{P\in\bar{\mathcal P}}p(P)\log p(P)$ and let $p^+_{j,t}$ denote the posterior after the update for $C$. For an already observed candidate record $C$, with the transfer cost $\Delta^{\mathrm{imp}}_{j,t}(C)$ measured on the same normalized entropy scale, its conditional (post-outcome) net value is
\begin{equation}
    V_{j,t}(C)=
    \mathbb E\!\left[H(p_{j,t})-H(p^+_{j,t})\mid H_t^{\mathrm{sys}},C\right]
    -\lambda\,\Delta^{\mathrm{imp}}_{j,t}(C).
    \label{eq:true_value}
\end{equation}
The first term prices the information $C$ carries about the governing principle; the second charges the transfer at rate $\lambda$, so a record must be expected to pay for its own transfer. Only the expectation is certified: a realized import can still raise the target's entropy.

\noindent \textbf{Route: certified admission and ranking.} As Eq.~\eqref{eq:true_value} makes explicit, routing is scored after the source record, outcome included, is published; the surviving uncertainty is the verification and fitting randomness conditional on the observed $C$, not the outcome itself (a pre-outcome variant would additionally integrate over the predictive distribution of $y$). Admission is therefore decided from a computable estimate $\widehat V_{j,t}(C)$, and the estimation error must be accounted for explicitly. We require the estimation to come with a calibrated radius $r_{j,t}(C,\delta)$ such that
$
\Pr\!\left(\left|\widehat V_{j,t}(C)-V_{j,t}(C)\right|\le r_{j,t}(C,\delta)\mid H_t^{\mathrm{sys}},C\right)\ge 1-\delta.
$
Admission then uses a lower-confidence rule with a margin $\eta>0$ and a per-candidate confidence budget $\delta_{j,t,C}$: the router accepts a candidate only if
\begin{equation}
    \widehat V_{j,t}(C)-r_{j,t}(C,\delta_{j,t,C})>\eta,
    \label{eq:safe_voi_gate}
\end{equation}
so a candidate that merely looks valuable on average, but whose value is uncertain, is refused. Among the candidates that pass, the router ranks by
\begin{equation}
    C^\star_{j,t}
    =
    \argmin_{C\in\mathrm{passed}_j}
    \nicefrac{\Delta^{\mathrm{imp}}_{j,t}(C)^2}
    {\max\{\widehat V_{j,t}(C)-r_{j,t}(C,\delta_{j,t,C}),0\}+\varepsilon},
    \label{eq:ids_ratio}
\end{equation}
which prefers candidates offering high certified value per unit transfer cost. Admission is further capped per round and target branch by a routing quota $B$ (default $3$), keeping the top-$B$ candidates in ranked order, and by a redundancy cap $R_{\max}$ that drops a candidate whose text is a near-duplicate (token-Jaccard similarity ${\ge}\,R_{\max}$, default $0.3$) of one already admitted in the same round; both are specified in Appendix~\ref{app:loop_implementation}. The gate and the ratio decide \emph{which} records may flow, and in what order; \emph{how} an accepted record enters the target posterior is specified by the injection step that follows.

\noindent \textbf{Inject: discounted likelihood update.} We assume that branch contexts may be heterogeneous. A cross-branch observation is therefore valid for target branch $j$ only through an explicit likelihood $q_{j\leftarrow i}(y\mid h,P)$; to keep this transfer disciplined, we introduce a source--target discount $\alpha_{j\leftarrow i,t}(h)\in[0,1]$ that downweights the contribution of a source observation to a target posterior, formulated as:
\begin{equation}
    \alpha_{j\leftarrow i,t}(h)=
    \operatorname{clip}_{[0,1]}\!\left(
    \rho_{i,t}(h)s_{ij,t}(h)v_{i,t}(h)
    \right),
    \label{eq:discount_factor}
\end{equation}
where $\rho_{i,t}$ scores the source's accuracy in predicting its own outcomes, $s_{ij,t}$ scores how well the source's setting matches the target's, and the optional $v_{i,t}$ scores whether the observation has been independently replicated. 
Therefore, branch $j$'s posterior after local evidence and accepted imports is
\begin{equation}
    p_{j,t}(P)=
    \nicefrac{
    \pi_0(P)
    \prod_{(h,y)\in\mathcal D_{j,t}}\underbrace{p_j(y\mid h,P)}_{\scriptsize \text{Local updates}}
    \prod_{(i,h,y,\alpha)\in\mathcal I_{j,t}}\underbrace{q_{j\leftarrow i}(y\mid h,P)^\alpha}_{\scriptsize \text{Evidence transfer}}
    }{Z_{j,t}},
    \label{eq:power_posterior}
\end{equation}
where $\pi_0(P)>0$ is the prior, $\mathcal D_{j,t}$ contains locally tested observations, $\mathcal I_{j,t}$ contains target-specific accepted imports, and $Z_{j,t}$ is the normalizer. When the observations are conditionally independent and each $q_{j\leftarrow i}$ is the correct law for an observation generated in source context $i$ and interpreted by target $j$, this is an ordinary posterior; otherwise it is a generalized-Bayes or composite-likelihood posterior, and an import must not be read as a full independent sample.

Each injection multiplies the target posterior by exactly one factor of Eq.~\eqref{eq:power_posterior}. The three concerns thus stay separate by construction: the shared pool governs system-level coverage, the discounted posterior governs branch-level concentration, and the calibrated gate governs negative-transfer control.

\noindent \textbf{Instantiation of \method.}
Every quantity above is computed from logged statistics and record metadata alone. The value estimate $\widehat V_{j,t}(C)$ is evaluated exactly: clone the target posterior, apply the candidate at weight $\alpha$, and measure the induced entropy change. This is tractable because the principle universe is finite. The transfer cost $\Delta^{\mathrm{imp}}_{j,t}(C)$ combines fixed bookkeeping rates with a surcharge that grows with the same setting mismatch that downweights $\alpha$, so a less compatible source is discounted and charged more at once. Each accepted import takes effect only once its acceptance is recorded, so no posterior update is silent. Exact enumeration certifies the numerical update ($r=0$) but not Gaussian Process misspecification or residual context shift (refer to the $p_j(y\mid h, P)$ modeling in PiEvo~\citep{Pu2026PrincipleEvolvableSD}); the gate therefore remains a conservative heuristic rather than a deployed guarantee (Remark~\ref{rem:gate_scope}). Exact forms and the per-round loop are given in Appendix~\ref{app:loop_implementation}.

\begin{theorem}[Concentration and negative-transfer control, informal]
    \label{thm:informal_guarantees}
    Under a finite principle universe with positive prior (Appendix~\ref{subsec:guarantees}), \method guarantees:
    (a) \textbf{concentration under discounted imports}: if the discounted log-likelihood contrast in favor of $P^\star$ grows linearly in the effective evidence count $N^{\mathrm{eff}}_{j,T}$, i.e., local observations entering at unit weight, imports at their discount $\alpha$, then each branch's belief $p_{j,T}(P^\star)$ converges to one exponentially fast in $N^{\mathrm{eff}}_{j,T}$ (\cref{thm:concentration});
    and (b) \textbf{negative-transfer control}: if the value estimates are calibrated, the gate of Eq.~\eqref{eq:safe_voi_gate} bounds the probability of accepting any import whose net value falls below the margin $\eta$ by the calibration budget $\sum_{j,t,C}\delta_{j,t,C}$ (\cref{thm:negative_transfer}).
    \vspace{-0.08cm}
\end{theorem}

In summary, we detail the implementation of \method in Algorithm~\ref{algo:method}, and \method provides theoretical guarantees (\cref{thm:informal_guarantees}, formal statements are in \cref{subsec:guarantees}, with proofs in Appendix~\ref{app:methodology_proofs}) and empirical evidence, as validated in \cref{subsec:efficiency_analysis,subsec:ablation_study} (Figures~\ref{fig:efficiency} and~\ref{fig:ablation_module_knockouts}).

\begin{remark}[Evidence-level privacy scope]
    \label{rem:evidence_scope}
    \method coordinates branches through \emph{scientific evidence}: it keeps each branch's posterior and search history private by construction. For the sensitive evidence, privacy-preserving variants of the shared pool compose with the routing and gating machinery unchanged, and we leave them as future work.
\end{remark}

\section{Experiments}
\label{sec:experiments}
\subsection{Experiment Setup}
\label{subsec:experiment_setup}

\noindent \textbf{Benchmarks.}
We evaluate \method on six scientific tasks, spanning Molecular Bio-activity Optimization (MBO)~\citep{Gaulton2011ChEMBLAL}, Antimicrobial Peptide Design (AMP)~\citep{SantosJnior2019MacrelAP}, Promoter Expression Optimization (Promoter)~\citep{deAlmeida2022GenerationOS} in biology, Nanohelix Optimization (NHO)~\citep{wu2025machine} and Transition-Metal-Complex Design (TMC)~\citep{Song2025EvaluatingLL} in materials science, and Superconductor Critical-Temperature Optimization (SPO)~\citep{Hamidieh2018ADS} in physics. Each task utilizes surrogates~\citep{Pu2026PrincipleEvolvableSD} and released databases, and gates candidates by the standard validity rules of the field. Details refer to Appendix~\ref{sec:benchmark}.

\noindent \textbf{Baselines.}
We compare \method against three families of baselines. (i)~\emph{Vanilla MAS}: a plain multi-agent system without any hypothesis-selection or principle-evolution mechanism, which anchors the per-task reference. (ii)~\emph{Autonomous research agents}: The AI Scientist v1~\citep{Lu2024TheAS}, The AI Scientist v2~\citep{Yamada2025TheAS}, AI Researcher~\citep{Tang2025AIResearcherAS}, EvoScientist~\citep{Lyu2026EvoScientistTM}, and InternAgent-1.5~\citep{Feng2026InternAgent15AU}; to ensure a fair comparison on hypothesis discovery, each system is restricted to its core hypothesis-generation loop (no literature retrieval or report drafting). (iii)~\emph{Principle-guided search}: PiFlow~\citep{Pu2025PiFlowPS}, which performs principle-aware exploration-exploitation balance, and PiEvo~\citep{Pu2026PrincipleEvolvableSD}, which evolves its principle space.

\noindent \textbf{Implementation details.}
We use Gemma-4-31B-IT~\citep{Abd2026Gemma4T} and GPT-5.6-Terra~\citep{openai2026gpt56systemcard} for evaluation, with three independent seeds, under a total evaluation budget of $72$ queries without any web or retrieval access; \method runs $K{=}3$ branches by default ($24$ queries per branch). Per-task wall-clock time is measured from launcher start to final evaluation on identical shared hardware, using the same LLM endpoint.

\subsection{Evaluation Metrics}
\label{subsec:evaluation_metrics}
Following ~\citet{Pu2026PrincipleEvolvableSD}, our evaluation of \method against baselines focuses on (a) objective attainment, (b) exploration breadth, and (c) per-evaluation anytime efficiency. The elapsed-time and cost metrics (wall-clock speedup, the equal-time quality gap $\Delta\mathrm{SQ}@T_c$, the token conversion ratio, and token-normalized AUOC) are defined in \cref{subsec:efficiency_analysis}, where they are used.

\noindent \textbf{a. Solution Quality (SQ).}
$
	\mathrm{SQ} = \nicefrac{\max \{ y_k \mid (h_k, y_k) \in \mathcal{T} \} - y_{\mathrm{lo}}}{y_{\mathrm{hi}} - y_{\mathrm{lo}}} \times 100\% \,
$ is the maximum outcome along a trajectory $\mathcal{T}$, normalized to the domain reference scale $[y_{\mathrm{lo}}, y_{\mathrm{hi}}]$ (see Appendix~\ref{sec:benchmark}). Here $100\%$ marks the task's domain reference scale: a definitional bound on four tasks and a domain-anchored target on MBO and SPO, never approached in our runs (Appendix~\ref{sec:benchmark}); this keeps the cross-task average an arithmetic mean of comparable quantities. For a multi-branch system such as \method, $\mathcal{T}$ is the union of the run's branch trajectories (best-of-seed), and APD and AUOC below are computed over the same pooled trajectory, so every metric prices the run as a whole.

\noindent \textbf{b. Average Pairwise Distance (APD).}
$
    \mathrm{APD} = \nicefrac{2}{M(M-1)} \sum_{1 \le i < j \le M} d\left(\phi(h_i), \phi(h_j)\right) \,
$ is the mean pairwise distance among the $M$ valid hypotheses under the task-specific feature map $\phi(h)$ (see Appendix~\ref{sec:benchmark}). Higher APD indicates a wider spread of exploration.

\noindent \textbf{c. Area Under the Optimization Curve (AUOC).}
$
    \mathrm{AUOC} = \nicefrac{1}{T} \sum_{t=1}^{T} \frac{\max_{k \le t} \{ y_k \} - y_{\mathrm{lo}}}{y_{\mathrm{hi}} - y_{\mathrm{lo}}}.
$ measures the cumulative running-best performance over the fixed evaluation budget. Higher AUOC indicates that a method reaches strong running-best values earlier in the evaluation sequence.

\begin{table*}[t]
	\centering
	\caption{\textbf{Performance comparison with baselines on 6 benchmarks under the Solution Quality (SQ \%) metric.}
        Categorized by LLMs: \texttt{Gemma-4-31B-IT} and \texttt{GPT-5.6-Terra} with both non-thinking mode.
        We report mean $\pm$ std over three seeds. On the GPT-5.6-Terra backbone, \method's paired per-seed SQ gain over PiEvo is positive in 15 of 18 task--seed pairs (two-sided sign test $p{<}0.01$; Wilcoxon $p{=}0.001$; full tests in \cref{app:stats}). \textbf{Bold} marks the best mean per task within each backbone block and the best Average. The colored ratio after each cell highlights \textcolor[HTML]{7D2AD1}{superior} ($\uparrow$) or \textcolor[HTML]{D14D4D}{inferior} ($\downarrow$) SQ relative to the worst non-zero mean on the same backbone and task, which is \underline{underlined}.}
	\label{tab:main_compare_8tasks}
	\resizebox{\textwidth}{!}{
		\begin{tabular}{l|cccccc|l}
			\toprule[1.5pt]
			\textbf{Model / Method}                      & \bfseries MBO & \bfseries NHO & \bfseries SPO & \bfseries TMC & \bfseries AMP & \bfseries Promoter & \bfseries Average \\
			\midrule
			\multicolumn{8}{l}{\textbf{Gemma-4-31B-IT}} \\
			\midrule
			\rowcolor{rowcolor}
			Vanilla MAS                                 & \underline{39.2} {\scriptsize \color{gray} $\pm$ 0.8} & \underline{77.2} {\scriptsize \color{gray} $\pm$ 1.4} & 0.0 {\scriptsize \color{gray} $\pm$ 0.0}\, {\scriptsize \textcolor[HTML]{D14D4D}{$\downarrow$0.0x}} & 74.4 {\scriptsize \color{gray} $\pm$ 0.0}\, {\scriptsize \textcolor[HTML]{7D2AD1}{$\uparrow$1.3x}} & 27.4 {\scriptsize \color{gray} $\pm$ 2.3}\, {\scriptsize \textcolor[HTML]{7D2AD1}{$\uparrow$1.1x}} & 84.4 {\scriptsize \color{gray} $\pm$ 0.3}\, {\scriptsize \textcolor[HTML]{7D2AD1}{$\uparrow$4.5x}} & 50.4\, {\scriptsize \textcolor[HTML]{7D2AD1}{$\uparrow$1.2x}} \\
			The AI Scientist v1~\citep{Lu2024TheAS}     & 48.6 {\scriptsize \color{gray} $\pm$ 2.2}\, {\scriptsize \textcolor[HTML]{7D2AD1}{$\uparrow$1.2x}} & 78.3 {\scriptsize \color{gray} $\pm$ 4.2}\, {\scriptsize \textcolor[HTML]{7D2AD1}{$\uparrow$1.0x}} & 0.0 {\scriptsize \color{gray} $\pm$ 0.0}\, {\scriptsize \textcolor[HTML]{D14D4D}{$\downarrow$0.0x}} & 84.5 {\scriptsize \color{gray} $\pm$ 0.0}\, {\scriptsize \textcolor[HTML]{7D2AD1}{$\uparrow$1.5x}} & \underline{24.1} {\scriptsize \color{gray} $\pm$ 1.5} & \underline{18.7} {\scriptsize \color{gray} $\pm$ 32.4} & 42.4\, {\scriptsize \textcolor[HTML]{7D2AD1}{$\uparrow$1.0x}} \\
			\rowcolor{rowcolor}
			The AI Scientist v2~\citep{Yamada2025TheAS} & 48.5 {\scriptsize \color{gray} $\pm$ 5.1}\, {\scriptsize \textcolor[HTML]{7D2AD1}{$\uparrow$1.2x}} & 84.3 {\scriptsize \color{gray} $\pm$ 3.9}\, {\scriptsize \textcolor[HTML]{7D2AD1}{$\uparrow$1.1x}} & \underline{8.2} {\scriptsize \color{gray} $\pm$ 1.6} & 83.7 {\scriptsize \color{gray} $\pm$ 0.8}\, {\scriptsize \textcolor[HTML]{7D2AD1}{$\uparrow$1.5x}} & 27.7 {\scriptsize \color{gray} $\pm$ 3.0}\, {\scriptsize \textcolor[HTML]{7D2AD1}{$\uparrow$1.2x}} & 0.0 {\scriptsize \color{gray} $\pm$ 0.0}\, {\scriptsize \textcolor[HTML]{D14D4D}{$\downarrow$0.0x}} & \underline{42.1} \\
			AI Researcher~\citep{Tang2025AIResearcherAS}& 50.0 {\scriptsize \color{gray} $\pm$ 2.4}\, {\scriptsize \textcolor[HTML]{7D2AD1}{$\uparrow$1.3x}} & 88.3 {\scriptsize \color{gray} $\pm$ 1.8}\, {\scriptsize \textcolor[HTML]{7D2AD1}{$\uparrow$1.1x}} & 24.6 {\scriptsize \color{gray} $\pm$ 2.0}\, {\scriptsize \textcolor[HTML]{7D2AD1}{$\uparrow$3.0x}} & 81.5 {\scriptsize \color{gray} $\pm$ 6.1}\, {\scriptsize \textcolor[HTML]{7D2AD1}{$\uparrow$1.5x}} & 26.7 {\scriptsize \color{gray} $\pm$ 14.6}\, {\scriptsize \textcolor[HTML]{7D2AD1}{$\uparrow$1.1x}} & 19.3 {\scriptsize \color{gray} $\pm$ 33.5}\, {\scriptsize \textcolor[HTML]{7D2AD1}{$\uparrow$1.0x}} & 48.4\, {\scriptsize \textcolor[HTML]{7D2AD1}{$\uparrow$1.2x}} \\
			\rowcolor{rowcolor}
			EvoScientist~\citep{Lyu2026EvoScientistTM}  & 45.3 {\scriptsize \color{gray} $\pm$ 1.5}\, {\scriptsize \textcolor[HTML]{7D2AD1}{$\uparrow$1.2x}} & 93.3 {\scriptsize \color{gray} $\pm$ 0.7}\, {\scriptsize \textcolor[HTML]{7D2AD1}{$\uparrow$1.2x}} & 9.2 {\scriptsize \color{gray} $\pm$ 0.7}\, {\scriptsize \textcolor[HTML]{7D2AD1}{$\uparrow$1.1x}} & 85.4 {\scriptsize \color{gray} $\pm$ 0.0}\, {\scriptsize \textcolor[HTML]{7D2AD1}{$\uparrow$1.5x}} & 29.7 {\scriptsize \color{gray} $\pm$ 2.6}\, {\scriptsize \textcolor[HTML]{7D2AD1}{$\uparrow$1.2x}} & 54.0 {\scriptsize \color{gray} $\pm$ 7.8}\, {\scriptsize \textcolor[HTML]{7D2AD1}{$\uparrow$2.9x}} & 52.8\, {\scriptsize \textcolor[HTML]{7D2AD1}{$\uparrow$1.3x}} \\
			InternAgent 1.5~\citep{Feng2026InternAgent15AU} & 53.6 {\scriptsize \color{gray} $\pm$ 2.1}\, {\scriptsize \textcolor[HTML]{7D2AD1}{$\uparrow$1.4x}} & 90.5 {\scriptsize \color{gray} $\pm$ 5.5}\, {\scriptsize \textcolor[HTML]{7D2AD1}{$\uparrow$1.2x}} & 10.5 {\scriptsize \color{gray} $\pm$ 0.1}\, {\scriptsize \textcolor[HTML]{7D2AD1}{$\uparrow$1.3x}} & 76.8 {\scriptsize \color{gray} $\pm$ 1.6}\, {\scriptsize \textcolor[HTML]{7D2AD1}{$\uparrow$1.4x}} & 28.7 {\scriptsize \color{gray} $\pm$ 1.7}\, {\scriptsize \textcolor[HTML]{7D2AD1}{$\uparrow$1.2x}} & 86.2 {\scriptsize \color{gray} $\pm$ 0.6}\, {\scriptsize \textcolor[HTML]{7D2AD1}{$\uparrow$4.6x}} & 57.7\, {\scriptsize \textcolor[HTML]{7D2AD1}{$\uparrow$1.4x}} \\
			\rowcolor{rowcolor}
			PiFlow~\citep{Pu2025PiFlowPS}               & 64.5 {\scriptsize \color{gray} $\pm$ 5.7}\, {\scriptsize \textcolor[HTML]{7D2AD1}{$\uparrow$1.6x}} & 91.0 {\scriptsize \color{gray} $\pm$ 3.7}\, {\scriptsize \textcolor[HTML]{7D2AD1}{$\uparrow$1.2x}} & 15.7 {\scriptsize \color{gray} $\pm$ 7.9}\, {\scriptsize \textcolor[HTML]{7D2AD1}{$\uparrow$1.9x}} & \underline{56.0} {\scriptsize \color{gray} $\pm$ 48.5} & 32.7 {\scriptsize \color{gray} $\pm$ 4.5}\, {\scriptsize \textcolor[HTML]{7D2AD1}{$\uparrow$1.4x}} & 72.2 {\scriptsize \color{gray} $\pm$ 9.1}\, {\scriptsize \textcolor[HTML]{7D2AD1}{$\uparrow$3.9x}} & 55.3\, {\scriptsize \textcolor[HTML]{7D2AD1}{$\uparrow$1.3x}} \\
			PiEvo~\citep{Pu2026PrincipleEvolvableSD}    & 61.8 {\scriptsize \color{gray} $\pm$ 2.7}\, {\scriptsize \textcolor[HTML]{7D2AD1}{$\uparrow$1.6x}} & \textbf{94.9 {\scriptsize \color{gray} $\pm$ 0.1}}\, {\scriptsize \textcolor[HTML]{7D2AD1}{$\uparrow$1.2x}} & 10.6 {\scriptsize \color{gray} $\pm$ 0.0}\, {\scriptsize \textcolor[HTML]{7D2AD1}{$\uparrow$1.3x}} & 84.7 {\scriptsize \color{gray} $\pm$ 0.0}\, {\scriptsize \textcolor[HTML]{7D2AD1}{$\uparrow$1.5x}} & 31.4 {\scriptsize \color{gray} $\pm$ 5.8}\, {\scriptsize \textcolor[HTML]{7D2AD1}{$\uparrow$1.3x}} & 54.4 {\scriptsize \color{gray} $\pm$ 6.3}\, {\scriptsize \textcolor[HTML]{7D2AD1}{$\uparrow$2.9x}} & 56.3\, {\scriptsize \textcolor[HTML]{7D2AD1}{$\uparrow$1.3x}} \\
			\rowcolor{rowcolor}
			\textbf{\method} (ours)                     & \textbf{77.1 {\scriptsize \color{gray} $\pm$ 4.6}}\, {\scriptsize \textcolor[HTML]{7D2AD1}{$\uparrow$2.0x}} & \textbf{94.9 {\scriptsize \color{gray} $\pm$ 0.1}}\, {\scriptsize \textcolor[HTML]{7D2AD1}{$\uparrow$1.2x}} & \textbf{30.0 {\scriptsize \color{gray} $\pm$ 0.6}}\, {\scriptsize \textcolor[HTML]{7D2AD1}{$\uparrow$3.7x}} & \textbf{88.3 {\scriptsize \color{gray} $\pm$ 4.2}}\, {\scriptsize \textcolor[HTML]{7D2AD1}{$\uparrow$1.6x}} & \textbf{37.6 {\scriptsize \color{gray} $\pm$ 2.6}}\, {\scriptsize \textcolor[HTML]{7D2AD1}{$\uparrow$1.6x}} & \textbf{86.4 {\scriptsize \color{gray} $\pm$ 0.4}}\, {\scriptsize \textcolor[HTML]{7D2AD1}{$\uparrow$4.6x}} & \textbf{69.1}\, {\scriptsize \textcolor[HTML]{7D2AD1}{$\uparrow$1.6x}} \\
			\midrule
			\multicolumn{8}{l}{\textbf{GPT-5.6-Terra}} \\
			\midrule
			\rowcolor{rowcolor}
			Vanilla MAS                                 & \underline{51.0} {\scriptsize \color{gray} $\pm$ 2.6} & \underline{73.2} {\scriptsize \color{gray} $\pm$ 2.6} & 0.0 {\scriptsize \color{gray} $\pm$ 0.0}\, {\scriptsize \textcolor[HTML]{D14D4D}{$\downarrow$0.0x}} & \underline{66.6} {\scriptsize \color{gray} $\pm$ 1.7} & 22.1 {\scriptsize \color{gray} $\pm$ 3.2}\, {\scriptsize \textcolor[HTML]{7D2AD1}{$\uparrow$1.0x}} & 70.8 {\scriptsize \color{gray} $\pm$ 14.3}\, {\scriptsize \textcolor[HTML]{7D2AD1}{$\uparrow$1.1x}} & \underline{47.3} \\
			The AI Scientist v1~\citep{Lu2024TheAS}     & 60.0 {\scriptsize \color{gray} $\pm$ 6.5}\, {\scriptsize \textcolor[HTML]{7D2AD1}{$\uparrow$1.2x}} & 89.8 {\scriptsize \color{gray} $\pm$ 4.7}\, {\scriptsize \textcolor[HTML]{7D2AD1}{$\uparrow$1.2x}} & 0.0 {\scriptsize \color{gray} $\pm$ 0.0}\, {\scriptsize \textcolor[HTML]{D14D4D}{$\downarrow$0.0x}} & 85.4 {\scriptsize \color{gray} $\pm$ 0.0}\, {\scriptsize \textcolor[HTML]{7D2AD1}{$\uparrow$1.3x}} & 23.8 {\scriptsize \color{gray} $\pm$ 2.6}\, {\scriptsize \textcolor[HTML]{7D2AD1}{$\uparrow$1.1x}} & \underline{62.1} {\scriptsize \color{gray} $\pm$ 16.7} & 53.5\, {\scriptsize \textcolor[HTML]{7D2AD1}{$\uparrow$1.1x}} \\
			\rowcolor{rowcolor}
			The AI Scientist v2~\citep{Yamada2025TheAS} & 51.5 {\scriptsize \color{gray} $\pm$ 5.8}\, {\scriptsize \textcolor[HTML]{7D2AD1}{$\uparrow$1.0x}} & 84.9 {\scriptsize \color{gray} $\pm$ 9.7}\, {\scriptsize \textcolor[HTML]{7D2AD1}{$\uparrow$1.2x}} & 10.8 {\scriptsize \color{gray} $\pm$ 0.3}\, {\scriptsize \textcolor[HTML]{7D2AD1}{$\uparrow$1.1x}} & 83.6 {\scriptsize \color{gray} $\pm$ 3.0}\, {\scriptsize \textcolor[HTML]{7D2AD1}{$\uparrow$1.3x}} & 26.7 {\scriptsize \color{gray} $\pm$ 0.0}\, {\scriptsize \textcolor[HTML]{7D2AD1}{$\uparrow$1.2x}} & 76.4 {\scriptsize \color{gray} $\pm$ 15.6}\, {\scriptsize \textcolor[HTML]{7D2AD1}{$\uparrow$1.2x}} & 55.6\, {\scriptsize \textcolor[HTML]{7D2AD1}{$\uparrow$1.2x}} \\
			AI Researcher~\citep{Tang2025AIResearcherAS}& 61.2 {\scriptsize \color{gray} $\pm$ 5.7}\, {\scriptsize \textcolor[HTML]{7D2AD1}{$\uparrow$1.2x}} & 93.6 {\scriptsize \color{gray} $\pm$ 2.5}\, {\scriptsize \textcolor[HTML]{7D2AD1}{$\uparrow$1.3x}} & 11.8 {\scriptsize \color{gray} $\pm$ 0.1}\, {\scriptsize \textcolor[HTML]{7D2AD1}{$\uparrow$1.2x}} & 89.3 {\scriptsize \color{gray} $\pm$ 3.4}\, {\scriptsize \textcolor[HTML]{7D2AD1}{$\uparrow$1.3x}} & 23.4 {\scriptsize \color{gray} $\pm$ 2.5}\, {\scriptsize \textcolor[HTML]{7D2AD1}{$\uparrow$1.1x}} & 73.6 {\scriptsize \color{gray} $\pm$ 8.3}\, {\scriptsize \textcolor[HTML]{7D2AD1}{$\uparrow$1.2x}} & 58.8\, {\scriptsize \textcolor[HTML]{7D2AD1}{$\uparrow$1.2x}} \\
			\rowcolor{rowcolor}
			EvoScientist~\citep{Lyu2026EvoScientistTM}  & 53.7 {\scriptsize \color{gray} $\pm$ 5.6}\, {\scriptsize \textcolor[HTML]{7D2AD1}{$\uparrow$1.1x}} & 86.8 {\scriptsize \color{gray} $\pm$ 6.6}\, {\scriptsize \textcolor[HTML]{7D2AD1}{$\uparrow$1.2x}} & \underline{10.2} {\scriptsize \color{gray} $\pm$ 0.2} & 78.5 {\scriptsize \color{gray} $\pm$ 6.0}\, {\scriptsize \textcolor[HTML]{7D2AD1}{$\uparrow$1.2x}} & \underline{21.5} {\scriptsize \color{gray} $\pm$ 2.9} & 72.6 {\scriptsize \color{gray} $\pm$ 11.7}\, {\scriptsize \textcolor[HTML]{7D2AD1}{$\uparrow$1.2x}} & 53.9\, {\scriptsize \textcolor[HTML]{7D2AD1}{$\uparrow$1.1x}} \\
			InternAgent 1.5~\citep{Feng2026InternAgent15AU} & 61.0 {\scriptsize \color{gray} $\pm$ 6.0}\, {\scriptsize \textcolor[HTML]{7D2AD1}{$\uparrow$1.2x}} & 91.8 {\scriptsize \color{gray} $\pm$ 3.3}\, {\scriptsize \textcolor[HTML]{7D2AD1}{$\uparrow$1.3x}} & 11.4 {\scriptsize \color{gray} $\pm$ 0.1}\, {\scriptsize \textcolor[HTML]{7D2AD1}{$\uparrow$1.1x}} & 91.8 {\scriptsize \color{gray} $\pm$ 2.2}\, {\scriptsize \textcolor[HTML]{7D2AD1}{$\uparrow$1.4x}} & 24.1 {\scriptsize \color{gray} $\pm$ 2.1}\, {\scriptsize \textcolor[HTML]{7D2AD1}{$\uparrow$1.1x}} & 74.1 {\scriptsize \color{gray} $\pm$ 10.7}\, {\scriptsize \textcolor[HTML]{7D2AD1}{$\uparrow$1.2x}} & 59.0\, {\scriptsize \textcolor[HTML]{7D2AD1}{$\uparrow$1.2x}} \\
			\rowcolor{rowcolor}
			PiFlow~\citep{Pu2025PiFlowPS}               & 52.3 {\scriptsize \color{gray} $\pm$ 14.6}\, {\scriptsize \textcolor[HTML]{7D2AD1}{$\uparrow$1.0x}} & 76.3 {\scriptsize \color{gray} $\pm$ 5.9}\, {\scriptsize \textcolor[HTML]{7D2AD1}{$\uparrow$1.0x}} & 15.0 {\scriptsize \color{gray} $\pm$ 7.6}\, {\scriptsize \textcolor[HTML]{7D2AD1}{$\uparrow$1.5x}} & 77.6 {\scriptsize \color{gray} $\pm$ 2.9}\, {\scriptsize \textcolor[HTML]{7D2AD1}{$\uparrow$1.2x}} & 22.8 {\scriptsize \color{gray} $\pm$ 1.7}\, {\scriptsize \textcolor[HTML]{7D2AD1}{$\uparrow$1.1x}} & 76.3 {\scriptsize \color{gray} $\pm$ 8.2}\, {\scriptsize \textcolor[HTML]{7D2AD1}{$\uparrow$1.2x}} & 53.4\, {\scriptsize \textcolor[HTML]{7D2AD1}{$\uparrow$1.1x}} \\
			PiEvo~\citep{Pu2026PrincipleEvolvableSD}    & 54.3 {\scriptsize \color{gray} $\pm$ 1.7}\, {\scriptsize \textcolor[HTML]{7D2AD1}{$\uparrow$1.1x}} & 86.2 {\scriptsize \color{gray} $\pm$ 1.4}\, {\scriptsize \textcolor[HTML]{7D2AD1}{$\uparrow$1.2x}} & 10.4 {\scriptsize \color{gray} $\pm$ 0.1}\, {\scriptsize \textcolor[HTML]{7D2AD1}{$\uparrow$1.0x}} & 91.5 {\scriptsize \color{gray} $\pm$ 1.3}\, {\scriptsize \textcolor[HTML]{7D2AD1}{$\uparrow$1.4x}} & 23.8 {\scriptsize \color{gray} $\pm$ 0.0}\, {\scriptsize \textcolor[HTML]{7D2AD1}{$\uparrow$1.1x}} & 79.4 {\scriptsize \color{gray} $\pm$ 5.8}\, {\scriptsize \textcolor[HTML]{7D2AD1}{$\uparrow$1.3x}} & 57.6\, {\scriptsize \textcolor[HTML]{7D2AD1}{$\uparrow$1.2x}} \\
			\rowcolor{rowcolor}
			\textbf{\method} (ours)                     & \textbf{71.3 {\scriptsize \color{gray} $\pm$ 19.1}}\, {\scriptsize \textcolor[HTML]{7D2AD1}{$\uparrow$1.4x}} & \textbf{93.9 {\scriptsize \color{gray} $\pm$ 0.7}}\, {\scriptsize \textcolor[HTML]{7D2AD1}{$\uparrow$1.3x}} & \textbf{15.3 {\scriptsize \color{gray} $\pm$ 8.3}}\, {\scriptsize \textcolor[HTML]{7D2AD1}{$\uparrow$1.5x}} & \textbf{93.0 {\scriptsize \color{gray} $\pm$ 0.0}}\, {\scriptsize \textcolor[HTML]{7D2AD1}{$\uparrow$1.4x}} & \textbf{27.7 {\scriptsize \color{gray} $\pm$ 1.7}}\, {\scriptsize \textcolor[HTML]{7D2AD1}{$\uparrow$1.3x}} & \textbf{82.1 {\scriptsize \color{gray} $\pm$ 3.0}}\, {\scriptsize \textcolor[HTML]{7D2AD1}{$\uparrow$1.3x}} & \textbf{63.9}\, {\scriptsize \textcolor[HTML]{7D2AD1}{$\uparrow$1.4x}} \\
			\bottomrule[1.5pt]
		\end{tabular}
	}
	\vspace{-7pt}
\end{table*}

\begin{figure*}[t]
	\centering
    \vspace{-0.18cm}
	\includegraphics[width=\textwidth]{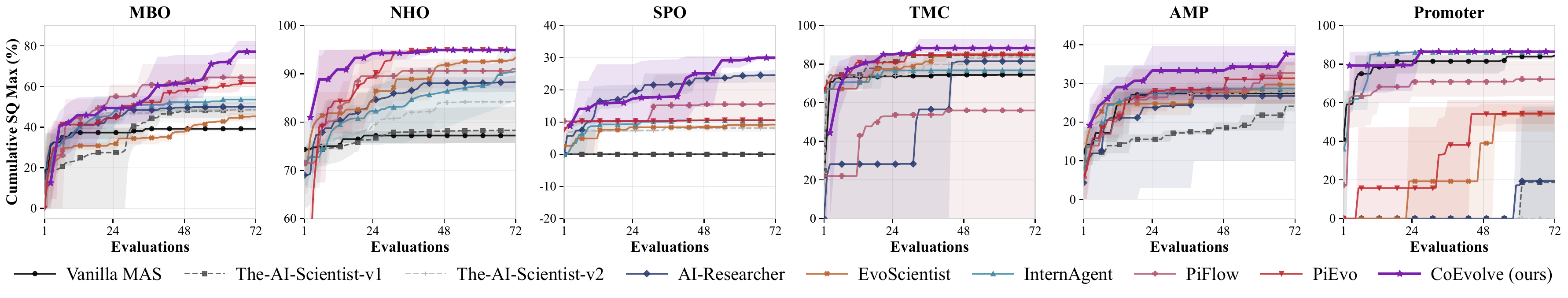}\vspace*{-0.3cm}
	\caption{\textbf{Running-best SQ w/ Gemma-4-31B-IT.} We plot the cumulative best SQ over evaluations for each task per Table~\ref{tab:main_compare_8tasks}. \method attains the highest final running best on five of the six tasks and ties PiEvo~\citep{Pu2026PrincipleEvolvableSD} on NHO ($94.9\%$). }
	\label{fig:main_anytime_evals}
    \vspace{-12pt}
\end{figure*}

\begin{table*}[t]
	\centering
	\caption{\textbf{Exploration (APD $\uparrow$) and exploitation (AUOC $\uparrow$) comparison with baselines on 6 benchmarks.}
        Both over the same three seeds as Table~\ref{tab:main_compare_8tasks}.
        \textbf{Bold} marks the best mean per column within each backbone and \underline{underline} the second-best; the colored ratio after each Average highlights \textcolor[HTML]{7D2AD1}{superior} ($\uparrow$) or \textcolor[HTML]{D14D4D}{inferior} ($\downarrow$) performance relative to Vanilla on the same backbone. \method attains the best exploration--exploitation balance.}
	\label{tab:apd_auoc}
    \vspace{-6pt}
	\resizebox{\textwidth}{!}{
		\begin{tabular}{l|cccccccccccc|cc}
			\toprule[1.5pt]
			\textbf{Model / Method} & \multicolumn{2}{c}{\textbf{MBO}} & \multicolumn{2}{c}{\textbf{NHO}} & \multicolumn{2}{c}{\textbf{SPO}} & \multicolumn{2}{c}{\textbf{TMC}} & \multicolumn{2}{c}{\textbf{AMP}} & \multicolumn{2}{c}{\textbf{Promoter}} & \multicolumn{2}{c}{\textbf{Average}} \\
			 & \textbf{APD} $\uparrow$ & \textbf{AUOC} $\uparrow$ & \textbf{APD} $\uparrow$ & \textbf{AUOC} $\uparrow$ & \textbf{APD} $\uparrow$ & \textbf{AUOC} $\uparrow$ & \textbf{APD} $\uparrow$ & \textbf{AUOC} $\uparrow$ & \textbf{APD} $\uparrow$ & \textbf{AUOC} $\uparrow$ & \textbf{APD} $\uparrow$ & \textbf{AUOC} $\uparrow$ & \textbf{Avg APD} $\uparrow$ & \textbf{Avg AUOC} $\uparrow$ \\
			\midrule
			\multicolumn{15}{l}{\textbf{Gemma-4-31B-IT}} \\
			\midrule
			\rowcolor{rowcolor}
			Vanilla MAS                                 & \textbf{79.1 {\scriptsize \color{gray} $\pm$ 1.1}} & 37.5 {\scriptsize \color{gray} $\pm$ 0.9} & 24.0 {\scriptsize \color{gray} $\pm$ 1.1} & 76.5 {\scriptsize \color{gray} $\pm$ 1.0} & 0.0 {\scriptsize \color{gray} $\pm$ 0.0} & 0.0 {\scriptsize \color{gray} $\pm$ 0.0} & 29.8 {\scriptsize \color{gray} $\pm$ 4.8} & 73.7 {\scriptsize \color{gray} $\pm$ 0.5} & 76.1 {\scriptsize \color{gray} $\pm$ 0.9} & 25.7 {\scriptsize \color{gray} $\pm$ 2.6} & 36.9 {\scriptsize \color{gray} $\pm$ 1.1} & 80.2 {\scriptsize \color{gray} $\pm$ 4.4} & 41.0 & 48.9 \\
			The AI Scientist v1~\citep{Lu2024TheAS}     & 21.7 {\scriptsize \color{gray} $\pm$ 17.0} & 38.2 {\scriptsize \color{gray} $\pm$ 3.4} & 5.9 {\scriptsize \color{gray} $\pm$ 3.4} & 76.8 {\scriptsize \color{gray} $\pm$ 3.6} & 0.0 {\scriptsize \color{gray} $\pm$ 0.0} & 0.0 {\scriptsize \color{gray} $\pm$ 0.0} & 53.1 {\scriptsize \color{gray} $\pm$ 1.1} & \underline{81.5} {\scriptsize \color{gray} $\pm$ 1.0} & 58.5 {\scriptsize \color{gray} $\pm$ 3.0} & 17.2 {\scriptsize \color{gray} $\pm$ 1.1} & 0.0 {\scriptsize \color{gray} $\pm$ 0.0} & 18.7 {\scriptsize \color{gray} $\pm$ 32.4} & 23.2\, {\scriptsize \textcolor[HTML]{D14D4D}{$\downarrow$0.6x}} & 38.7\, {\scriptsize \textcolor[HTML]{D14D4D}{$\downarrow$0.8x}} \\
			\rowcolor{rowcolor}
			The AI Scientist v2~\citep{Yamada2025TheAS} & 34.7 {\scriptsize \color{gray} $\pm$ 12.5} & 43.6 {\scriptsize \color{gray} $\pm$ 3.5} & 12.9 {\scriptsize \color{gray} $\pm$ 2.2} & 80.5 {\scriptsize \color{gray} $\pm$ 3.0} & 4.6 {\scriptsize \color{gray} $\pm$ 2.9} & 7.7 {\scriptsize \color{gray} $\pm$ 1.0} & 47.4 {\scriptsize \color{gray} $\pm$ 4.2} & 78.5 {\scriptsize \color{gray} $\pm$ 0.7} & 34.3 {\scriptsize \color{gray} $\pm$ 6.2} & 25.7 {\scriptsize \color{gray} $\pm$ 3.5} & 0.0 {\scriptsize \color{gray} $\pm$ 0.0} & 0.0 {\scriptsize \color{gray} $\pm$ 0.0} & 22.3\, {\scriptsize \textcolor[HTML]{D14D4D}{$\downarrow$0.5x}} & 39.3\, {\scriptsize \textcolor[HTML]{D14D4D}{$\downarrow$0.8x}} \\
			AI Researcher~\citep{Tang2025AIResearcherAS}& 46.5 {\scriptsize \color{gray} $\pm$ 9.0} & 46.1 {\scriptsize \color{gray} $\pm$ 2.1} & 16.6 {\scriptsize \color{gray} $\pm$ 4.8} & 84.5 {\scriptsize \color{gray} $\pm$ 1.4} & 4.7 {\scriptsize \color{gray} $\pm$ 0.1} & \textbf{21.2 {\scriptsize \color{gray} $\pm$ 1.9}} & 52.3 {\scriptsize \color{gray} $\pm$ 5.3} & 81.1 {\scriptsize \color{gray} $\pm$ 5.9} & 73.1 {\scriptsize \color{gray} $\pm$ 8.7} & 23.5 {\scriptsize \color{gray} $\pm$ 14.8} & 18.3 {\scriptsize \color{gray} $\pm$ 31.7} & 18.3 {\scriptsize \color{gray} $\pm$ 31.7} & 35.3\, {\scriptsize \textcolor[HTML]{D14D4D}{$\downarrow$0.9x}} & 45.8\, {\scriptsize \textcolor[HTML]{D14D4D}{$\downarrow$0.9x}} \\
			\rowcolor{rowcolor}
			EvoScientist~\citep{Lyu2026EvoScientistTM}  & 57.0 {\scriptsize \color{gray} $\pm$ 11.4} & 35.7 {\scriptsize \color{gray} $\pm$ 1.7} & 23.3 {\scriptsize \color{gray} $\pm$ 1.5} & 87.3 {\scriptsize \color{gray} $\pm$ 2.0} & 8.3 {\scriptsize \color{gray} $\pm$ 0.9} & 8.3 {\scriptsize \color{gray} $\pm$ 0.8} & 61.8 {\scriptsize \color{gray} $\pm$ 0.7} & 79.5 {\scriptsize \color{gray} $\pm$ 0.5} & 60.7 {\scriptsize \color{gray} $\pm$ 2.2} & 25.8 {\scriptsize \color{gray} $\pm$ 2.4} & 0.0 {\scriptsize \color{gray} $\pm$ 0.0} & 54.0 {\scriptsize \color{gray} $\pm$ 7.8} & 35.2\, {\scriptsize \textcolor[HTML]{D14D4D}{$\downarrow$0.9x}} & 48.4\, {\scriptsize \textcolor[HTML]{D14D4D}{$\downarrow$1.0x}} \\
			InternAgent 1.5~\citep{Feng2026InternAgent15AU} & 55.3 {\scriptsize \color{gray} $\pm$ 5.0} & 47.3 {\scriptsize \color{gray} $\pm$ 5.0} & 14.3 {\scriptsize \color{gray} $\pm$ 3.4} & 83.3 {\scriptsize \color{gray} $\pm$ 4.6} & 4.7 {\scriptsize \color{gray} $\pm$ 4.9} & 9.9 {\scriptsize \color{gray} $\pm$ 0.6} & 51.5 {\scriptsize \color{gray} $\pm$ 2.3} & 73.4 {\scriptsize \color{gray} $\pm$ 2.4} & 51.8 {\scriptsize \color{gray} $\pm$ 6.3} & 26.3 {\scriptsize \color{gray} $\pm$ 1.4} & 30.6 {\scriptsize \color{gray} $\pm$ 1.2} & \textbf{85.4 {\scriptsize \color{gray} $\pm$ 0.5}} & 34.7\, {\scriptsize \textcolor[HTML]{D14D4D}{$\downarrow$0.8x}} & \underline{54.3}\, {\scriptsize \textcolor[HTML]{7D2AD1}{$\uparrow$1.1x}} \\
			\rowcolor{rowcolor}
			PiFlow~\citep{Pu2025PiFlowPS}               & \underline{78.9} {\scriptsize \color{gray} $\pm$ 3.6} & \textbf{56.2 {\scriptsize \color{gray} $\pm$ 5.2}} & \underline{34.3} {\scriptsize \color{gray} $\pm$ 5.7} & \underline{87.5} {\scriptsize \color{gray} $\pm$ 3.1} & \textbf{15.4 {\scriptsize \color{gray} $\pm$ 0.7}} & 13.3 {\scriptsize \color{gray} $\pm$ 4.7} & 45.0 {\scriptsize \color{gray} $\pm$ 16.0} & 52.0 {\scriptsize \color{gray} $\pm$ 45.2} & \textbf{78.1 {\scriptsize \color{gray} $\pm$ 1.3}} & 26.1 {\scriptsize \color{gray} $\pm$ 0.6} & \underline{42.9} {\scriptsize \color{gray} $\pm$ 2.2} & 64.9 {\scriptsize \color{gray} $\pm$ 4.7} & 49.1\, {\scriptsize \textcolor[HTML]{7D2AD1}{$\uparrow$1.2x}} & 50.0\, {\scriptsize \textcolor[HTML]{7D2AD1}{$\uparrow$1.0x}} \\
			PiEvo~\citep{Pu2026PrincipleEvolvableSD}    & 77.3 {\scriptsize \color{gray} $\pm$ 5.1} & 50.1 {\scriptsize \color{gray} $\pm$ 2.7} & 27.3 {\scriptsize \color{gray} $\pm$ 3.8} & 86.7 {\scriptsize \color{gray} $\pm$ 6.3} & 6.1 {\scriptsize \color{gray} $\pm$ 0.2} & 10.3 {\scriptsize \color{gray} $\pm$ 0.1} & \underline{67.2} {\scriptsize \color{gray} $\pm$ 2.4} & \textbf{82.3 {\scriptsize \color{gray} $\pm$ 2.0}} & \underline{76.2} {\scriptsize \color{gray} $\pm$ 4.8} & \underline{26.8} {\scriptsize \color{gray} $\pm$ 3.1} & \textbf{46.9 {\scriptsize \color{gray} $\pm$ 7.5}} & 33.1 {\scriptsize \color{gray} $\pm$ 16.6} & \underline{50.2}\, {\scriptsize \textcolor[HTML]{7D2AD1}{$\uparrow$1.2x}} & 48.2\, {\scriptsize \textcolor[HTML]{D14D4D}{$\downarrow$1.0x}} \\
			\rowcolor{rowcolor}
			\textbf{\method} (ours)                     & 73.9 {\scriptsize \color{gray} $\pm$ 7.6} & \underline{54.7} {\scriptsize \color{gray} $\pm$ 6.5} & \textbf{39.1 {\scriptsize \color{gray} $\pm$ 0.1}} & \textbf{92.7 {\scriptsize \color{gray} $\pm$ 1.7}} & \underline{13.3} {\scriptsize \color{gray} $\pm$ 2.2} & \underline{20.7} {\scriptsize \color{gray} $\pm$ 6.9} & \textbf{68.1 {\scriptsize \color{gray} $\pm$ 5.0}} & 80.0 {\scriptsize \color{gray} $\pm$ 10.6} & 75.6 {\scriptsize \color{gray} $\pm$ 1.7} & \textbf{31.7 {\scriptsize \color{gray} $\pm$ 4.9}} & 42.2 {\scriptsize \color{gray} $\pm$ 9.5} & \underline{83.6} {\scriptsize \color{gray} $\pm$ 4.6} & \textbf{52.0}\, {\scriptsize \textcolor[HTML]{7D2AD1}{$\uparrow$1.3x}} & \textbf{60.6}\, {\scriptsize \textcolor[HTML]{7D2AD1}{$\uparrow$1.2x}} \\
			\midrule
			\multicolumn{15}{l}{\textbf{GPT-5.6-Terra}} \\
			\midrule
			\rowcolor{rowcolor}
			Vanilla MAS                                 & 57.4 {\scriptsize \color{gray} $\pm$ 6.0} & 47.7 {\scriptsize \color{gray} $\pm$ 1.8} & 22.0 {\scriptsize \color{gray} $\pm$ 1.9} & 70.7 {\scriptsize \color{gray} $\pm$ 4.8} & 0.0 {\scriptsize \color{gray} $\pm$ 0.0} & 0.0 {\scriptsize \color{gray} $\pm$ 0.0} & 49.2 {\scriptsize \color{gray} $\pm$ 2.4} & 65.0 {\scriptsize \color{gray} $\pm$ 1.4} & 83.8 {\scriptsize \color{gray} $\pm$ 0.2} & 20.4 {\scriptsize \color{gray} $\pm$ 1.5} & 33.6 {\scriptsize \color{gray} $\pm$ 0.6} & 63.2 {\scriptsize \color{gray} $\pm$ 13.8} & 41.0 & 44.5 \\
			The AI Scientist v1~\citep{Lu2024TheAS}     & 6.5 {\scriptsize \color{gray} $\pm$ 5.6} & \underline{54.6} {\scriptsize \color{gray} $\pm$ 4.5} & 13.7 {\scriptsize \color{gray} $\pm$ 1.4} & 80.8 {\scriptsize \color{gray} $\pm$ 12.6} & 0.0 {\scriptsize \color{gray} $\pm$ 0.0} & 0.0 {\scriptsize \color{gray} $\pm$ 0.0} & 56.3 {\scriptsize \color{gray} $\pm$ 2.3} & 75.9 {\scriptsize \color{gray} $\pm$ 6.5} & 38.3 {\scriptsize \color{gray} $\pm$ 11.0} & 20.3 {\scriptsize \color{gray} $\pm$ 1.4} & 19.1 {\scriptsize \color{gray} $\pm$ 4.2} & 53.9 {\scriptsize \color{gray} $\pm$ 9.0} & 22.3\, {\scriptsize \textcolor[HTML]{D14D4D}{$\downarrow$0.5x}} & 47.6\, {\scriptsize \textcolor[HTML]{7D2AD1}{$\uparrow$1.1x}} \\
			\rowcolor{rowcolor}
			The AI Scientist v2~\citep{Yamada2025TheAS} & 28.1 {\scriptsize \color{gray} $\pm$ 9.3} & 45.3 {\scriptsize \color{gray} $\pm$ 3.3} & 15.6 {\scriptsize \color{gray} $\pm$ 4.3} & 72.4 {\scriptsize \color{gray} $\pm$ 9.9} & 1.5 {\scriptsize \color{gray} $\pm$ 0.3} & 10.5 {\scriptsize \color{gray} $\pm$ 0.1} & 57.3 {\scriptsize \color{gray} $\pm$ 2.3} & 75.3 {\scriptsize \color{gray} $\pm$ 2.6} & 33.3 {\scriptsize \color{gray} $\pm$ 3.9} & \textbf{24.2 {\scriptsize \color{gray} $\pm$ 0.4}} & 29.6 {\scriptsize \color{gray} $\pm$ 2.7} & \underline{72.0} {\scriptsize \color{gray} $\pm$ 15.3} & 27.6\, {\scriptsize \textcolor[HTML]{D14D4D}{$\downarrow$0.7x}} & 49.9\, {\scriptsize \textcolor[HTML]{7D2AD1}{$\uparrow$1.1x}} \\
			AI Researcher~\citep{Tang2025AIResearcherAS}& 74.8 {\scriptsize \color{gray} $\pm$ 1.6} & 44.6 {\scriptsize \color{gray} $\pm$ 6.9} & 23.1 {\scriptsize \color{gray} $\pm$ 5.3} & \underline{88.6} {\scriptsize \color{gray} $\pm$ 5.7} & 6.6 {\scriptsize \color{gray} $\pm$ 2.3} & 10.8 {\scriptsize \color{gray} $\pm$ 0.3} & 57.1 {\scriptsize \color{gray} $\pm$ 1.2} & 82.8 {\scriptsize \color{gray} $\pm$ 1.4} & \underline{87.3} {\scriptsize \color{gray} $\pm$ 1.6} & 20.1 {\scriptsize \color{gray} $\pm$ 3.0} & 33.7 {\scriptsize \color{gray} $\pm$ 1.4} & 63.1 {\scriptsize \color{gray} $\pm$ 5.0} & 47.1\, {\scriptsize \textcolor[HTML]{7D2AD1}{$\uparrow$1.1x}} & 51.7\, {\scriptsize \textcolor[HTML]{7D2AD1}{$\uparrow$1.2x}} \\
			\rowcolor{rowcolor}
			EvoScientist~\citep{Lyu2026EvoScientistTM}  & 36.9 {\scriptsize \color{gray} $\pm$ 7.7} & 46.8 {\scriptsize \color{gray} $\pm$ 2.7} & 28.4 {\scriptsize \color{gray} $\pm$ 4.8} & 73.6 {\scriptsize \color{gray} $\pm$ 4.8} & \underline{16.4} {\scriptsize \color{gray} $\pm$ 1.6} & 9.7 {\scriptsize \color{gray} $\pm$ 0.3} & 60.6 {\scriptsize \color{gray} $\pm$ 2.6} & 75.2 {\scriptsize \color{gray} $\pm$ 4.0} & 84.8 {\scriptsize \color{gray} $\pm$ 3.9} & 19.5 {\scriptsize \color{gray} $\pm$ 3.5} & 35.3 {\scriptsize \color{gray} $\pm$ 5.4} & 65.7 {\scriptsize \color{gray} $\pm$ 15.1} & 43.7\, {\scriptsize \textcolor[HTML]{7D2AD1}{$\uparrow$1.1x}} & 48.4\, {\scriptsize \textcolor[HTML]{7D2AD1}{$\uparrow$1.1x}} \\
			InternAgent 1.5~\citep{Feng2026InternAgent15AU} & 35.3 {\scriptsize \color{gray} $\pm$ 23.7} & \textbf{55.2 {\scriptsize \color{gray} $\pm$ 6.6}} & 26.4 {\scriptsize \color{gray} $\pm$ 11.9} & 85.5 {\scriptsize \color{gray} $\pm$ 7.2} & 1.0 {\scriptsize \color{gray} $\pm$ 0.2} & 10.8 {\scriptsize \color{gray} $\pm$ 0.1} & 60.6 {\scriptsize \color{gray} $\pm$ 0.6} & 80.3 {\scriptsize \color{gray} $\pm$ 1.9} & 58.1 {\scriptsize \color{gray} $\pm$ 7.2} & 22.3 {\scriptsize \color{gray} $\pm$ 1.7} & 29.7 {\scriptsize \color{gray} $\pm$ 4.1} & 70.3 {\scriptsize \color{gray} $\pm$ 11.9} & 35.2\, {\scriptsize \textcolor[HTML]{D14D4D}{$\downarrow$0.9x}} & \underline{54.1}\, {\scriptsize \textcolor[HTML]{7D2AD1}{$\uparrow$1.2x}} \\
			\rowcolor{rowcolor}
			PiFlow~\citep{Pu2025PiFlowPS}               & \underline{77.4} {\scriptsize \color{gray} $\pm$ 7.5} & 45.2 {\scriptsize \color{gray} $\pm$ 10.6} & \underline{48.5} {\scriptsize \color{gray} $\pm$ 2.3} & 64.4 {\scriptsize \color{gray} $\pm$ 9.3} & \textbf{16.5 {\scriptsize \color{gray} $\pm$ 0.2}} & \textbf{12.3 {\scriptsize \color{gray} $\pm$ 4.4}} & \textbf{79.4 {\scriptsize \color{gray} $\pm$ 3.4}} & 74.5 {\scriptsize \color{gray} $\pm$ 0.7} & 86.9 {\scriptsize \color{gray} $\pm$ 1.4} & 19.2 {\scriptsize \color{gray} $\pm$ 1.8} & 37.0 {\scriptsize \color{gray} $\pm$ 1.1} & 68.5 {\scriptsize \color{gray} $\pm$ 4.3} & \underline{57.6}\, {\scriptsize \textcolor[HTML]{7D2AD1}{$\uparrow$1.4x}} & 47.3\, {\scriptsize \textcolor[HTML]{7D2AD1}{$\uparrow$1.1x}} \\
			PiEvo~\citep{Pu2026PrincipleEvolvableSD}    & 70.1 {\scriptsize \color{gray} $\pm$ 11.4} & 41.8 {\scriptsize \color{gray} $\pm$ 5.7} & 30.3 {\scriptsize \color{gray} $\pm$ 3.0} & 76.1 {\scriptsize \color{gray} $\pm$ 4.2} & 11.6 {\scriptsize \color{gray} $\pm$ 1.5} & 9.3 {\scriptsize \color{gray} $\pm$ 0.9} & 68.0 {\scriptsize \color{gray} $\pm$ 2.4} & \textbf{89.6 {\scriptsize \color{gray} $\pm$ 2.8}} & 84.8 {\scriptsize \color{gray} $\pm$ 1.0} & 22.0 {\scriptsize \color{gray} $\pm$ 0.9} & \underline{40.5} {\scriptsize \color{gray} $\pm$ 1.9} & \textbf{73.7 {\scriptsize \color{gray} $\pm$ 3.5}} & 50.9\, {\scriptsize \textcolor[HTML]{7D2AD1}{$\uparrow$1.2x}} & 52.1\, {\scriptsize \textcolor[HTML]{7D2AD1}{$\uparrow$1.2x}} \\
			\rowcolor{rowcolor}
			\textbf{\method} (ours)                     & \textbf{79.8 {\scriptsize \color{gray} $\pm$ 4.4}} & 45.5 {\scriptsize \color{gray} $\pm$ 8.8} & \textbf{54.7 {\scriptsize \color{gray} $\pm$ 1.6}} & \textbf{88.8 {\scriptsize \color{gray} $\pm$ 5.9}} & 9.4 {\scriptsize \color{gray} $\pm$ 6.7} & \underline{11.5} {\scriptsize \color{gray} $\pm$ 2.5} & \underline{75.5} {\scriptsize \color{gray} $\pm$ 0.7} & \underline{87.3} {\scriptsize \color{gray} $\pm$ 3.8} & \textbf{87.5 {\scriptsize \color{gray} $\pm$ 3.1}} & \underline{23.6} {\scriptsize \color{gray} $\pm$ 4.7} & \textbf{41.9 {\scriptsize \color{gray} $\pm$ 1.0}} & 70.8 {\scriptsize \color{gray} $\pm$ 7.6} & \textbf{58.1}\, {\scriptsize \textcolor[HTML]{7D2AD1}{$\uparrow$1.4x}} & \textbf{54.6}\, {\scriptsize \textcolor[HTML]{7D2AD1}{$\uparrow$1.2x}} \\
			\bottomrule[1.5pt]
		\end{tabular}
	}
	\vspace{-7pt}
\end{table*}

\subsection{Performance Analysis}
\label{subsec:performance_analysis}

Table~\ref{tab:main_compare_8tasks} compares \method against eight baselines on six benchmarks under two backbones, Figure~\ref{fig:main_anytime_evals} traces the running-best solution quality (SQ) over the evaluation sequence, and Table~\ref{tab:apd_auoc} profiles exploration (APD) against exploitation (AUOC). Three observations follow from these tables:

\noindent \textbf{Obs.\ding{202}: \method attains the best average SQ on these six benchmarks.}
\method attains the highest average SQ on both backbones, $63.9\% \sim 69.1\%$. This corresponds to an $8\% \sim 20\%$ relative gain over the strongest baseline, InternAgent-1.5~\citep{Feng2026InternAgent15AU}, and $11\% \sim 23\%$ over PiEvo~\citep{Pu2026PrincipleEvolvableSD} (paired per-seed gains positive on 15 of 18 task--seed pairs, $p{<}0.01$; Table~\ref{tab:main_compare_8tasks}), indicating that evidence transfer directs parallel exploration toward high-quality hypotheses.

\noindent \textbf{Obs.\ding{203}: \method's lead holds throughout evaluations on average.} 
In Figure~\ref{fig:main_anytime_evals}, \method's running-best curve attains the highest final value on five of six tasks. \method achieves the best average AUOC on both backbones (Table~\ref{tab:apd_auoc}). Strong candidates therefore arrive early and the advantage persists, which is what \cref{subsec:efficiency_analysis} prices in wall-clock time and tokens.

\noindent \textbf{Obs.\ding{204}: \method is the unique Pareto-optimal method.} 
\method attains the best average APD on both backbones together with the best average AUOC (Table~\ref{tab:apd_auoc}). Each baseline, for example, InternAgent-1.5~\citep{Feng2026InternAgent15AU} holds only the second-best average AUOC and explores $33\%\sim39\%$ more narrowly, while PiFlow nearly matches \method's exploration under GPT-5.6-Terra yet falls $7.3$ AUOC points behind.
This is consistent with the role of coordination core (\cref{subsec:theoretical_framework}): accepted cross-branch imports keep diversity while each branch's running best keeps improving.

\subsection{Efficiency Analysis}
\label{subsec:efficiency_analysis}
\begin{figure*}[t]
	\centering
    \vspace{-0.2cm}
	\includegraphics[width=\textwidth]{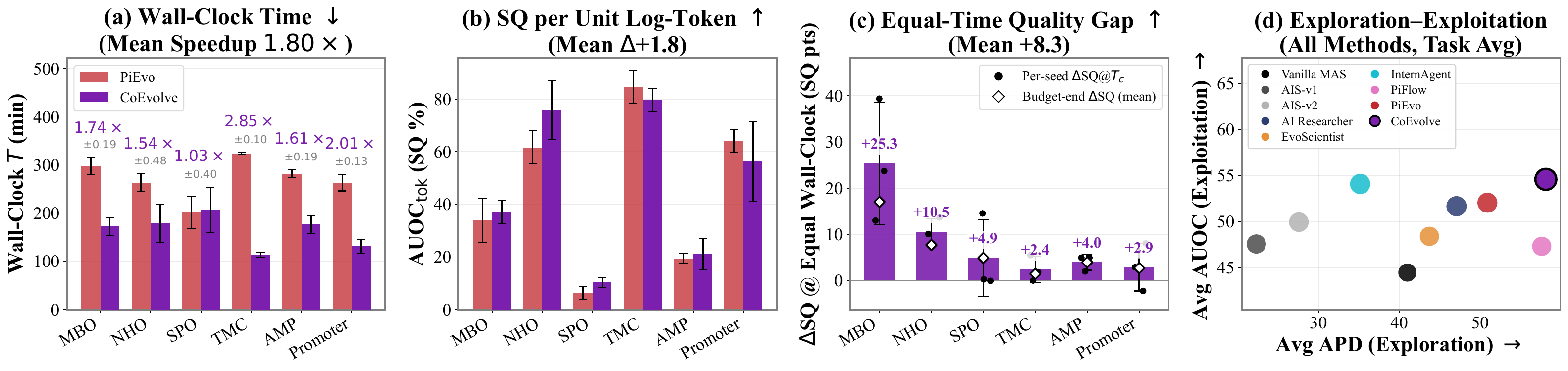}
    \vspace{-20pt}
	\caption{\textbf{Efficiency analysis on the GPT-5.6-Terra backbone (Table~\ref{tab:main_compare_8tasks} \& \ref{tab:apd_auoc})} 
    (a)~End-to-end wall-clock time per task, with the paired speedup $T_{\mathrm{PiEvo}}/T_{\methodm}$ annotated;  
    (b)~Token-normalized AUOC per task (AUOC$_{\mathrm{tok}}$); 
    (c)~Equal-time quality gap $\Delta\mathrm{SQ}@T_c$: PiEvo's running-best SQ read at \method's completion time $T_c$, subtracted from \method's final SQ; 
    (d)~\method is Pareto-optimal.
    }
	\label{fig:efficiency}
    \vspace{-0.38cm}
\end{figure*}

As \method is built upon PiEvo~\citep{Pu2026PrincipleEvolvableSD}, to measure wall-clock and token efficiency, we introduce three complementary measures (\textbf{Metric $1\sim3$}), then report two observations.

\noindent \textbf{Metric 1. Wall-clock speedup and equal-time quality gap.} We record the end-to-end wall-clock time $T$ under the same device environment and report the paired speedup $T_{\mathrm{PiEvo}}/T_{\methodm}$. We also compare the two systems after identical wall-clock investment: the \emph{equal-time quality gap} $\Delta\mathrm{SQ}@T_c = \mathrm{SQ}_{\methodm}(T_c) - \mathrm{SQ}_{\mathrm{PiEvo}}(T_c)$ reads PiEvo's running best at \method's completion time $T_c$, using the per-evaluation timestamps of both.

\noindent \textbf{Metric 2. Token conversion ratio $(\mathrm{SQ}_{\methodm}/\mathrm{SQ}_{\mathrm{PiEvo}})\,/\,(\mathrm{Tok}_{\methodm}/\mathrm{Tok}_{\mathrm{PiEvo}})$.} A value above $1$ means each \method token produces more final SQ than a PiEvo token. 

\noindent \textbf{Metric 3. Token-normalized area under the optimization curve (AUOC$_{\mathrm{tok}}$).} AUOC measures efficiency per \emph{evaluation}; AUOC$_{\mathrm{tok}}$ instead integrates the running-best curve over log-token-spend, normalized by the observed log-token window:
$
	\mathrm{AUOC}_{\mathrm{tok}} = \nicefrac{1}{\log \tau_{N} - \log \tau_{1}} \int_{\tau_{1}}^{\tau_{N}} \nicefrac{\max_{k \le \tau} \{ y_k \} - y_{\mathrm{lo}}}{y_{\mathrm{hi}} - y_{\mathrm{lo}}} \, \mathrm{d}\log\tau,
$
where $\tau$ denotes cumulative token spend. We approximate the integral by the trapezoidal rule on the observed evaluations.

\noindent \textbf{Obs.\ding{205}: Parallel advantage of \method costs a near-parity token premium.} 
Branch-summed, \method uses $1.06\times\sim1.62\times$ PiEvo's tokens per task.
In return, the conversion ratio ranges $0.89\times\sim1.07\times$ per task (mean $0.97\times$, above $1$ on MBO and NHO), and token-normalized AUOC (Figure~\ref{fig:efficiency}~(b)) favors \method on four of six tasks (mean $\Delta{+}1.8$). Together with Obs.\ding{206}: a $1.80\times$ wall-clock speedup (faster in 17 of 18 task--seed pairs, $p{<}10^{-3}$) at a $1.22\times$ token premium and a $0.97\times$ per-token rate means the parallel branches convert wall-clock into SQ.

\subsection{Case Study on Autonomous Research Harness}
\label{subsec:harness}

\noindent \textbf{Task configuration.} Per our goal in Section~\ref{sec:introduction} and following~\citet{fei2026autoresearcheval}, we port \method onto an autonomous research harness: a containerized coding agent must deliver an \emph{artifact} that a sealed evaluator scores against each task's published SOTA anchor. We sample five representative tasks from AutoresearchEval~\citep{fei2026autoresearcheval}: ALDE, D2D, and MolEdit from the \emph{exploiting} category, and Deconv and InvScat from the \emph{exploration} category (the exploiting category spans the free-oracle and recipe-saturated levels of the pre-registered explorativeness index; \cref{subsec:harness_tei}), and report anchor-relative improvement $\Delta$ under the same $72$-evaluation budget as Section~\ref{subsec:experiment_setup}. Task definitions, the substrate, the budget contract, and all controls are in Appendix~\ref{sec:harness_design}.

\begin{wraptable}{r}{0.52\textwidth}
	\centering
	\vspace{-8pt}
	\caption{\textbf{Performance comparison on five autoresearch tasks (anchor-relative improvement $\Delta\times100$, $\uparrow$).}
    Mean $\pm$ std over three seeds; \textbf{bold} = best per column. All five \method means are significantly above the anchor (one-sample $t$ against zero, $p{<}0.025$ per task; \cref{app:stats}).
}
	\label{tab:harness_compare}
	\resizebox{\linewidth}{!}{
		\begin{tabular}{l|ccccc|l}
			\toprule[1.5pt]
			\textbf{Arm}                                  & \bfseries ALDE & \bfseries Deconv & \bfseries InvScat & \bfseries D2D & \bfseries MolEdit & \bfseries Average \\
			\midrule

		\rowcolor{rowcolor}
		Claude Code & +2.3 {\scriptsize \color{gray} $\pm$ 0.2} & +14.0 {\scriptsize \color{gray} $\pm$ 7.3} & $-$5.8 {\scriptsize \color{gray} $\pm$ 7.6} & +9.0 {\scriptsize \color{gray} $\pm$ 0.3} & +0.2 {\scriptsize \color{gray} $\pm$ 1.7} & +3.9 \\

		Codex & +2.2 {\scriptsize \color{gray} $\pm$ 1.0} & +14.3 {\scriptsize \color{gray} $\pm$ 5.4} & $-$5.9 {\scriptsize \color{gray} $\pm$ 6.7} & +11.2 {\scriptsize \color{gray} $\pm$ 0.9} & +1.2 {\scriptsize \color{gray} $\pm$ 1.9} & +4.6 \\

		\rowcolor{rowcolor}
		Arbor & +1.3 {\scriptsize \color{gray} $\pm$ 1.9} & +11.7 {\scriptsize \color{gray} $\pm$ 23.8} & $-$6.5 {\scriptsize \color{gray} $\pm$ 5.9} & +15.3 {\scriptsize \color{gray} $\pm$ 0.1} & +0.5 {\scriptsize \color{gray} $\pm$ 0.6} & +4.5 \\

		PiEvo & +2.4 {\scriptsize \color{gray} $\pm$ 2.0} & +15.4 {\scriptsize \color{gray} $\pm$ 28.6} & $-$5.3 {\scriptsize \color{gray} $\pm$ 7.9} & +15.4 {\scriptsize \color{gray} $\pm$ 0.1} & +1.3 {\scriptsize \color{gray} $\pm$ 1.6} & +5.8 \\

		\rowcolor{rowcolor}
		\textbf{\method} & \textbf{+6.6 {\scriptsize \color{gray} $\pm$ 0.7}} & \textbf{+45.4 {\scriptsize \color{gray} $\pm$ 1.4}} & \textbf{+11.9 {\scriptsize \color{gray} $\pm$ 3.7}} & \textbf{+15.5 {\scriptsize \color{gray} $\pm$ 0.0}} & \textbf{+4.1 {\scriptsize \color{gray} $\pm$ 0.6}} & \textbf{+16.7} \\
			\bottomrule[1.5pt]
		\end{tabular}
	}
	\vspace{-8pt}
\end{wraptable}

We compare three types of harness: (a) \emph{agent-only} harnesses (Claude Code~\citep{anthropic2026claudecode}, Codex~\citep{openai2026unrolling}, Arbor~\citep{jin2026arbor}) act autonomously once given the task, (b) PiEvo~\citep{Pu2026PrincipleEvolvableSD} as a strategic layer inside Claude Code, injecting guidance per turn, and (c) \method runs $K{=}3$ guided PiEvo branches in parallel. With the same LLM backbone of \texttt{DeepSeek-v4-Flash-0731}, our analysis yields three observations:

\noindent \textbf{Obs.\ding{207}: \method surpasses all SOTA anchors~\citep{fei2026autoresearcheval}.}
\method alone keeps all five means above the anchor (one-sample $t$, $p{<}0.025$ per task), lifting the cross-task average from PiEvo's $+5.8$ to $+16.7$ (see Table~\ref{tab:harness_compare}).

\begin{wrapfigure}{r}{0.62\textwidth}
	\centering
	\vspace{-6pt}
	\begin{subcaptionblock}[t]{0.49\linewidth}
		\includegraphics[width=\linewidth]{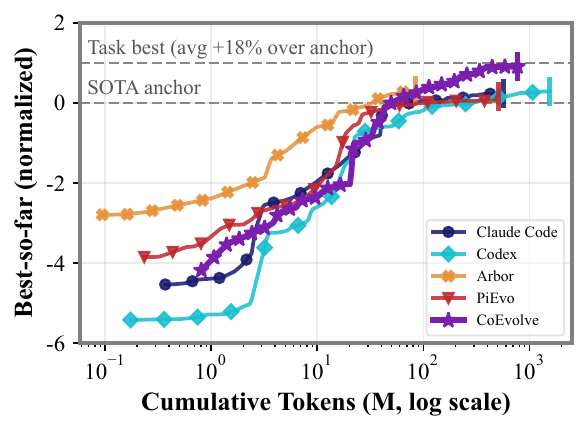}
		\label{fig:harness_anytime_tokens_merged}
	\end{subcaptionblock}\hfill
	\begin{subcaptionblock}[t]{0.49\linewidth}
		\includegraphics[width=\linewidth]{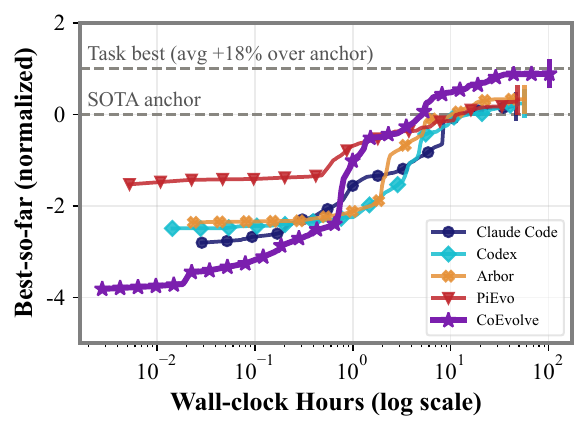}
		\label{fig:harness_anytime_wallclock_merged}
	\end{subcaptionblock}
	\vspace{-14pt}
	\caption{
     \method leads the token-spend curve from ${\sim}1$M tokens onward; the wall-clock view prices its parallel throughput.}
	\label{fig:harness_anytime_merged}
	\vspace{-8pt}
\end{wrapfigure}

The two margins decompose as follows (Table~\ref{tab:harness_compare}): stage guidance (i.e., PiEvo) adds only $+0.4$ percentage points, while the coordination core (i.e., \method) adds $+10.9$ on top; the worst \method case (mean minus std) also matches or exceeds the best case (mean plus std) of every other arm on every task.

\noindent \textbf{Obs.\ding{208}: \method converts breadth into top-tier quality.}
PiEvo leads mid-budget on InvScat and D2D (and edges out \method on MolEdit in full-budget AUOC; \cref{subsec:harness_trajectory}), yet \method gains complementary coverage at reduced single-line depth, and accepted imports enter the champion branch's posterior, at a $1.55\times$ mean billable-token spend (Figure~\ref{fig:harness_anytime_merged}, left; Table~\ref{tab:harness_tokens}) for a $1.31\times$ end-to-end wall-clock speedup (Table~\ref{tab:harness_ttt}).

\noindent \textbf{Obs.\ding{209}: The payoff is bounded by what parallel coverage can find.}
The runner-up margin reaches $+17.2\sim+30.0$ points on the open method-space tasks (Deconv, InvScat), against $+0.1\sim+4.2$ where a saturated recipe or a free oracle governs the score, matching the pre-registered explorativeness ordering (\cref{subsec:harness_tei}). \method gains most where branches observe what others miss (\cref{subsec:theoretical_framework}), as also discussed in \cref{subsec:ablation_study}.

\subsection{Ablation Study}
\label{subsec:ablation_study}

We ablate the collaboration layer at three levels: evidence sharing across branches, the transfer-safeguard modules, and the branch count and coordination hyperparameters. The sharing study spans all six tasks; the module and hyperparameter studies run on MBO under the protocol of Table~\ref{tab:main_compare_8tasks}.

\begin{wrapfigure}{r}{0.38\textwidth}
	\centering
	\vspace{-14pt}
	\includegraphics[width=\linewidth]{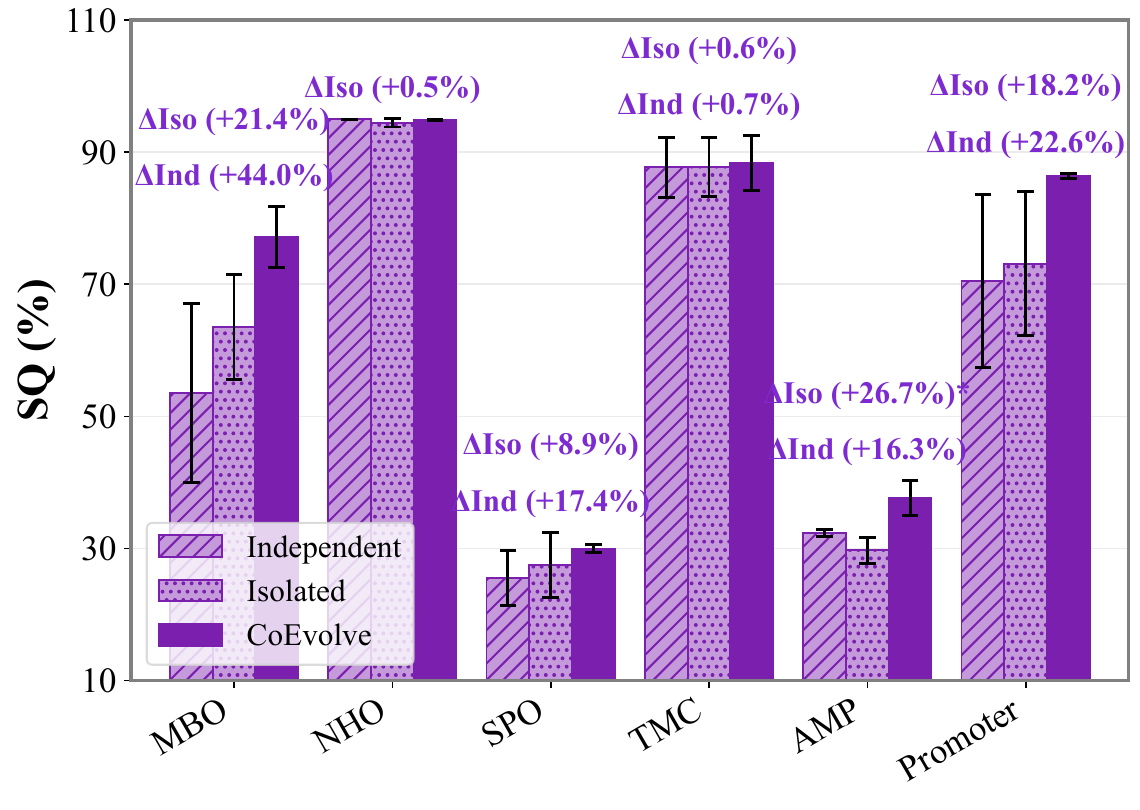}
	\vspace{-22pt}
	\caption{\textbf{Ablation of evidence transfer.} Per-task comparison from Independent (w/o coordination) through Isolated (w/o sharing) to \method.}
	\label{fig:ablation_sharing_arms}
    \vspace{-12pt}
\end{wrapfigure}

\noindent \textbf{Evidence sharing lifts the whole branch population over the Isolated setting.} Best-of-seed SQ of \method exceeds the coordination-only Isolated mode on all six tasks (Figure~\ref{fig:ablation_sharing_arms}), and removing coordination entirely (Independent) degrades MBO, SPO, and Promoter further. On these tasks the gain decomposes into coordination and sharing components. Moreover, on NHO, Independent matches \method ($95.0$ vs $94.9$) and coordination pays as convergence speed rather than as final quality ($1.5\times$ faster to the same ceiling; \cref{sec:sharing_diagnostics}). Worst-branch SQ (Table~\ref{tab:ablation_sharing_arms}) improves on five of six tasks (up to $+9.4$ on MBO), indicating that sharing lifts the entire branch population. The flow itself is real: $244$ accepted imports, $205$ net-unique after deduplication, with window-level, associational uplift of up to $+43.3$ SQ points (Table~\ref{tab:sharing_diagnostics}).

\begin{wrapfigure}{r}{0.38\textwidth}
	\centering
	\vspace{-10pt}
	\includegraphics[width=\linewidth]{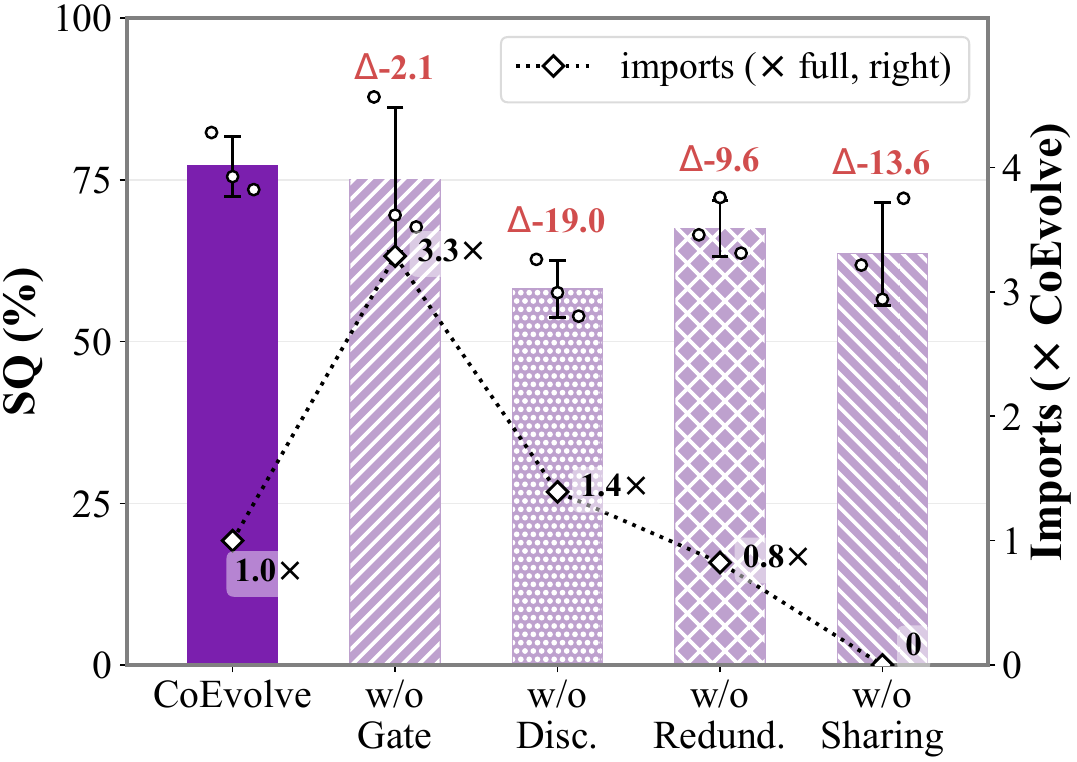}
	\vspace{-18pt}
	\caption{\textbf{Module-knockout ablation on MBO}. Accepted imports are normalized to the full system.}
    \vspace{-16pt}
	\label{fig:ablation_module_knockouts}
\end{wrapfigure}

\noindent \textbf{Discounting is load-bearing; the VoI gate is a tail-risk control.} Each knockout disables exactly one safeguard: import discounting ($\alpha{\equiv}1$), the safe-VoI gate ($\eta{\to}{-\infty}$, every routed import admitted), or the redundancy cap ($R_{\max}{=}1$). As shown in Figure~\ref{fig:ablation_module_knockouts} and Table~\ref{tab:ablation_module_knockouts}, removing discounting costs the most ($-19.0$ SQ), ahead of sharing itself ($-13.6$) and the redundancy cap ($-9.6$). Together with the shared mode's lead over Isolated on all six tasks (Figure~\ref{fig:ablation_sharing_arms}), this is the joint signature of \cref{thm:informal_guarantees}(a): calibrated sharing helps, and the same imports at full weight collapse. Removing the gate leaves the mean nearly intact ($-2.1$) but more than doubles the standard deviation ($4.6{\to}11.1$) while admitting over $3\times$ as many imports ($42{\to}137$), indicating that protection against rare harmful imports, exactly the tail-probability role \cref{thm:informal_guarantees}(b) and \cref{thm:negative_transfer} assign it.

\noindent \textbf{The collaboration gain grows with landscape width.} Varying the branch count at fixed total budget (Figure~\ref{fig:ksweep}) gives a task-dependent response: an inverted-U on MBO peaking at the default $K{=}3$ ($77.1$ SQ): more branches buy coverage at the price of per-branch depth, flat on AMP, and a rise into a plateau on TMC. The default $K{=}3$ is therefore not a universal optimum.

\noindent \textbf{The coordination defaults sit near the local optimum.} Varying one coordination hyperparameter at a time on MBO (Figure~\ref{fig:ablation_hparams}), attainment peaks at the default routing quota ($3$) and redundancy cap ($R_{\max}{=}0.3$, every off-value degrading by $8.3$--$16.4$ SQ), sits on a plateau in the import cost weight $\lambda{\in}[1,2]$, and shows no monotone response to the gate margin $\eta$, i.e., the single above-default point ($\eta{=}0.01$, $+4.3\pm18.5$) is within noise. Therefore, $\eta$ most strongly controls admission volume, whereas $R_{\max}$ barely moves import volume, indicating that benefit is content selectivity.

\section{Conclusion}
\label{sec:conclusion}
We present \method, a framework that formulates collaborative scientific discovery as evidence transfer among parallel principle-evolution branches: a coordination core routes observations as gated, discounted factors, keeping the branch-level posteriors separate.
\method attains a mean solution quality of $66.5\%$ against $57.0\%$ for the strong baseline PiEvo, with a $1.80\times$ wall-clock speedup, and is the only system above the published SOTA anchor on five representative auto-research tasks; removing the discount costs more than removing any other module, and the admission gate bounds rare harmful imports. These results make principle co-evolution a practical way for multiple AI Scientists to collaborate on one problem, with evidence transfer as the coupling mechanism.

\bibliography{sources/reference}
\bibliographystyle{configuration/iclr2026}
\nocite{*}

\newpage
\appendix

\section{Formal Assumptions and Guarantees}
\label{subsec:guarantees}

We state the assumptions used by the methodology before the corresponding results; complete proofs appear in Section~\ref{app:methodology_proofs}. Let $H_t^{\mathrm{sys}}$ be the complete system history after round $t$, let $\mathcal F_t=\sigma(H_t^{\mathrm{sys}})$, and define the information acquired by policy $\Pi$ as
\begin{equation}
\mathcal U_\Pi(T):=I_\Pi(P;H_T^{\mathrm{sys}}).
\label{eq:system_information}
\end{equation}
The four results below follow a common pattern: each states the assumptions it needs on the learning process, followed by the guarantee they buy. In order, they certify that parallel branches cover information faster (\cref{thm:parallel_coverage}), that discounted sharing does not break posterior concentration (\cref{thm:concentration}), that the VoI gate controls harmful imports in probability (\cref{thm:negative_transfer}), and that parallel search shortens the time to discovery (\cref{thm:discovery}).

\begin{assumption}[Finite support and positive prior]\label{ass:finite_positive}
The principle universe $\bar{\mathcal P}$ is finite and $\pi_0(P)>0$ for every $P\in\bar{\mathcal P}$.
\end{assumption}
\paragraph{Coverage.} The first guarantee asks that branches do not duplicate one another. Assumption~\ref{ass:complementarity} makes this precise through a redundancy factor $\underline\gamma_T$: if all branches observed the same evidence, the joint information would collapse to a single branch's share and $\underline\gamma_T$ would be near $1/K$; if their evidence were fully complementary, $\underline\gamma_T$ would approach $1$.

\begin{assumption}[Complementary information]\label{ass:complementarity}
On each active round, let $I_{j,t}:=I(P;Y_{j,t}\mid\mathcal F_{t-1},h_{j,t})$. For a deterministic $\underline\gamma_T\in[0,1]$,
\begin{equation}
\underbrace{I(P;Y_{1,t},\ldots,Y_{K,t}\mid\mathcal F_{t-1},h_{1:K,t})}_{\text{what the system learns jointly}}
\ge \underbrace{\underline\gamma_T}_{\text{complementarity}}\;\underbrace{\sum_{j=1}^K I_{j,t}}_{\text{sum of per-branch information}}.
\label{eq:gamma_complementarity}
\end{equation}
\end{assumption}
\begin{assumption}[Per-branch information rate]\label{ass:info_rate}
There are $\mu_j\ge0$ such that $I_{j,t}\ge\mu_j$ on active rounds. If $\mathcal A_T$ is the set of active rounds, then $|\mathcal A_T|\ge(1-\chi_T)T$ for a deterministic $\chi_T\in[0,1]$.
\end{assumption}
Assumption~\ref{ass:info_rate} adds two floors: every active round teaches branch $j$ at least $\mu_j$ nats, and at least a $(1-\chi_T)$ fraction of rounds are active. Together they yield the first guarantee.

\begin{theorem}[Parallel coverage]\label{thm:parallel_coverage}
Under Assumptions~\ref{ass:complementarity}--\ref{ass:info_rate}, before entropy saturation,
\[
\mathcal U_{\methodm}(T)\ge \underbrace{(1-\chi_T)T}_{\text{active rounds}}\;\underbrace{\underline\gamma_T\sum_{j=1}^K\mu_j}_{\text{complementary per-round rate}}.
\]
\end{theorem}
In words, the system accumulates mutual information at a rate proportional to the number of branches, discounted only by redundancy ($\underline\gamma_T$) and idle rounds ($\chi_T$): when branches explore complementary directions, $K$ parallel branches certify nearly $K$ times the information of one branch over the same wall-clock horizon.

\paragraph{Concentration.} Coverage alone does not say the system identifies the true principle; the second guarantee does. Assumption~\ref{ass:contrast} requires the accumulated log-likelihood contrast in favor of $P^\star$ to grow linearly in the number of effective samples $N^{\mathrm{eff}}_{j,T}$, where imported evidence enters with discounted (hence fractional) weights, up to a slack term $\mathfrak r_j$ that holds uniformly over time and over all wrong principles at once.

\begin{assumption}[Calibrated discounted contrast]\label{ass:contrast}
For each branch $j$ and $P\ne P^\star$, let $L_{j,T}(P)$ be the accumulated local and discounted-import log-likelihood ratio in favor of $P^\star$, and let $N^{\mathrm{eff}}_{j,T}$ be the sum of its predictable weights in $[0,1]$. There is a $\kappa_j(P)>0$ and a time-uniform radius $\mathfrak r_j(n,\delta)$ such that, with probability at least $1-\delta$, simultaneously over $T$ and $P\ne P^\star$,
\begin{equation}
L_{j,T}(P)\ge \underbrace{\kappa_j(P)N^{\mathrm{eff}}_{j,T}}_{\text{contrast accumulating at rate }\kappa_j(P)}
-\underbrace{\mathfrak r_j(N^{\mathrm{eff}}_{j,T},\delta)}_{\text{time-uniform confidence slack}}.
\label{eq:discounted_contrast_bound}
\end{equation}
We assume further that every likelihood ratio entering $L_{j,T}$ is finite on data generated under $P^\star$.
\end{assumption}

\begin{theorem}[Concentration under discounted imports]\label{thm:concentration}
Under Assumptions~\ref{ass:finite_positive} and~\ref{ass:contrast}, with probability at least $1-\delta$,
\[
p_{j,T}(P^\star)\ge 1-\underbrace{\frac{|\bar{\mathcal P}|-1}{\pi_0(P^\star)}}_{\text{prior scale}}\;\exp\{-\kappa_jN^{\mathrm{eff}}_{j,T}+\mathfrak r_j(N^{\mathrm{eff}}_{j,T},\delta)\},
\]
where $\kappa_j=\min_{P\ne P^\star}\kappa_j(P)$.
\end{theorem}
Since $\mathfrak r_j(n,\delta)$ grows sublinearly in $n$ (of order $\sqrt{n\log\log n}$ for bounded increments), the exponent is eventually dominated by $-\kappa_jN^{\mathrm{eff}}_{j,T}$, so the posterior mass on the true principle converges to $1$ exponentially fast in effective samples---even though part of those samples arrived as discounted imports from other branches.

\paragraph{Negative-transfer control.} The third guarantee concerns the gate that decides which imports to accept. It needs one ingredient: the gate's estimate $\widehat V_{j,t}(C)$ of a candidate's true net value $V_{j,t}(C)$ must come with a calibrated error radius.

\begin{assumption}[Calibrated value estimates]\label{ass:voi_calibration}
For every adaptively considered candidate,
\begin{equation}
|\widehat V_{j,t}(C)-V_{j,t}(C)|\le r_{j,t}(C,\delta)
\label{eq:voi_calibration}
\end{equation}
with conditional probability at least $1-\delta$.
\end{assumption}

\begin{theorem}[Negative-transfer control]\label{thm:negative_transfer}
Under Assumption~\ref{ass:voi_calibration}, the probability of accepting any candidate with $V_{j,t}(C)\le\eta$ is at most $\sum_{j,t,C}\delta_{j,t,C}$ when the gate in Eq.~\ref{eq:safe_voi_gate} is used.
\end{theorem}
Intuitively, a harmful candidate can pass the gate only if its estimate errs by more than its own radius, an event whose probability is capped by $\delta_{j,t,C}$; summing these caps over all candidates bounds the total failure budget. The gate therefore acts as a tail-probability control on harmful imports, with $\delta$ as a user-chosen budget, and the guarantee is exactly as strong as the calibration of the radius $r_{j,t}$ feeding it (see Remark~\ref{rem:gate_scope}).

\paragraph{Discovery.} The last guarantee models discovery as a race: each branch independently rolls a die every round, and the first success ends the race. Assumption~\ref{ass:discovery_hazard} only asks that each branch's per-round success probability is bounded below by some $\lambda_j>0$.

\begin{assumption}[Discovery hazards]\label{ass:discovery_hazard}
Before discovery, active branches discover a $P^\star$-consistent discriminator independently conditional on history, with per-round probability at least $\lambda_j$.
\end{assumption}

\begin{theorem}[Parallel discovery]\label{thm:discovery}
Under Assumption~\ref{ass:discovery_hazard},
\begin{equation}
\Pr(\tau_{\mathrm{disc}}>t)\le\underbrace{\prod_j(1-\lambda_j)^t}_{\substack{\text{no branch discovers}\\\text{in }t\text{ rounds}}}\le\exp\Bigl(-t\sum_j\lambda_j\Bigr).
\label{eq:discovery_tail}
\end{equation}
\end{theorem}
The tail probability of the discovery time decays exponentially at rate $\sum_j\lambda_j$: running $K$ branches multiplies the effective hazard, so the expected time to first discovery shrinks roughly by a factor of $K$ when the per-branch hazards are small and comparable.

\newpage
\section{Proofs of the Theoretical Guarantees}
\label{app:methodology_proofs}

In this section, we give the complete proofs of the results stated in \cref{subsec:guarantees}.

\subsection{Proof of parallel mutual-information coverage}
\label{app:proof_coverage}

\begin{proof}[Proof of \cref{thm:parallel_coverage}.]
	Write the round-$t$ system-history increment as $Z_t=(h_{1:K,t},\,Y_{1:K,t},\,W_t)$, where $W_t$ collects every remaining logged increment of round $t$ (routing decisions, verification outcomes, and monitoring records). We telescope \eqref{eq:system_information} over rounds and then lower-bound each round's contribution.

	\emph{Telescoping the system information over rounds.} Since $H_T^{\mathrm{sys}}=(Z_1,\ldots,Z_T)$, the chain rule for mutual information gives
	\begin{equation}
		\mathcal{U}_{\methodm}(T)=I(P;H_T^{\mathrm{sys}})
		=\sum_{t=1}^T I(P;Z_t\mid H_{t-1}^{\mathrm{sys}})
		=\sum_{t=1}^T \mathbb E\!\left[I(P;Z_t\mid\mathcal F_{t-1})\right],
		\label{eq:proof_coverage_telescope}
	\end{equation}
	where $I(P;Z_t\mid\mathcal F_{t-1})$ denotes the \emph{history-wise} conditional mutual information, i.e., the $\mathcal F_{t-1}$-measurable random variable
	\[
		I(P;Z_t\mid\mathcal F_{t-1})
		:=
		\mathrm D_{\mathrm{KL}}\!\left(p_{P,Z_t\mid H_{t-1}^{\mathrm{sys}}}\,\middle\|\,p_{P\mid H_{t-1}^{\mathrm{sys}}}\otimes p_{Z_t\mid H_{t-1}^{\mathrm{sys}}}\right),
	\]
	and the last equality in \eqref{eq:proof_coverage_telescope} is the tower property: $I(P;Z_t\mid H_{t-1}^{\mathrm{sys}})$ is by definition the expectation of this kernel over the random history. Every such kernel is bounded by $H_{\pi_0}(P)\le\log|\bar{\mathcal P}|<\infty$ under the standing finite-universe setting, so the telescoping and the interchange with expectation are legitimate. Adaptivity of the design poses no difficulty here---conditioning on the full $\mathcal F_{t-1}$ is precisely what makes each summand a well-defined information quantity, regardless of how the round-$t$ experiments were selected.

	\emph{Hypothesis selection carries no extra information.} Each branch selects $h_{j,t}$ from its own history, which is a sub-$\sigma$-algebra of $\mathcal F_{t-1}$, possibly augmented by exogenous randomization independent of $P$ given $\mathcal F_{t-1}$. Hence $I(P;h_{1:K,t}\mid\mathcal F_{t-1})=0$, and applying the chain rule to $Z_t$ conditionally on $\mathcal F_{t-1}$,
	\[
		I(P;Z_t\mid\mathcal F_{t-1})
		=
		\underbrace{I(P;h_{1:K,t}\mid\mathcal F_{t-1})}_{=\,0}
		+\,I(P;Y_{1:K,t}\mid\mathcal F_{t-1},h_{1:K,t})
		+\,I(P;W_t\mid\mathcal F_{t-1},h_{1:K,t},Y_{1:K,t}).
	\]
	Both surviving terms are nonnegative; keeping only the observation term,
	\begin{equation}
		I(P;Z_t\mid\mathcal F_{t-1})
		\;\ge\;
		I(P;Y_{1,t},\ldots,Y_{K,t}\mid\mathcal F_{t-1},h_{1:K,t}).
		\label{eq:proof_coverage_observation}
	\end{equation}
	This is the precise sense in which routing and logging increments are ``free'': discarding them can only weaken the bound.

	\emph{Lower-bounding each active round.} Whether round $t$ is active is decided from the history before the round is executed, so the indicator $\mathbf 1\{t\in\mathcal A_T\}$ is $\mathcal F_{t-1}$-measurable. On an active round, Assumption~\ref{ass:complementarity} applies to the right-hand side of \eqref{eq:proof_coverage_observation} and Assumption~\ref{ass:info_rate} floors each summand:
	\[
		I(P;Z_t\mid\mathcal F_{t-1})
		\;\ge\;
		\underline\gamma_T\sum_{j=1}^K I_{j,t}
		\;\ge\;
		\underline\gamma_T\sum_{j=1}^K\mu_j.
	\]
	On an inactive round we retain only nonnegativity of conditional mutual information. Combining the two cases,
	\[
		I(P;Z_t\mid\mathcal F_{t-1})
		\;\ge\;
		\mathbf 1\{t\in\mathcal A_T\}\,\underline\gamma_T\sum_{j=1}^K\mu_j
		\qquad\text{almost surely}.
	\]

	\emph{Summing the per-round bounds over the horizon.} Inserting into \eqref{eq:proof_coverage_telescope} and using that $\underline\gamma_T$ and the $\mu_j$ are deterministic,
	\[
		\mathcal{U}_{\methodm}(T)
		\;\ge\;
		\underline\gamma_T\sum_{j=1}^K\mu_j\;\mathbb E\!\left[\sum_{t=1}^T\mathbf 1\{t\in\mathcal A_T\}\right]
		\;=\;
		\mathbb E|\mathcal A_T|\,\underline\gamma_T\sum_{j=1}^K\mu_j
		\;\ge\;
		(1-\chi_T)T\,\underline\gamma_T\sum_{j=1}^K\mu_j,
	\]
	where the last step uses $|\mathcal A_T|\ge(1-\chi_T)T$ almost surely. Every step holds history-wise, so no independence assumption across rounds or branches enters the argument.
\end{proof}

\begin{remark}[The pre-saturation clause]
	\label{rem:pre_saturation}
	Mutual information about $P$ never exceeds the prior entropy, $\mathcal U_{\methodm}(T)\le H_{\pi_0}(P)$, so a per-round floor $I_{j,t}\ge\mu_j>0$ cannot hold indefinitely: Assumption~\ref{ass:info_rate}, and hence the bound of \cref{thm:parallel_coverage}, is consistent only for horizons satisfying $(1-\chi_T)T\,\underline\gamma_T\sum_{j=1}^K\mu_j\le H_{\pi_0}(P)$. The qualifier ``before entropy saturation'' in the theorem statement refers to this regime; once the posterior concentrates, the per-round increment is governed by the residual entropy $H(P\mid\mathcal F_{t-1})$ rather than by any constant floor, and the rate must be re-derived accordingly.
\end{remark}

\subsection{Proof of branch concentration under discounted imports}
\label{app:proof_concentration}

\begin{proof}[Proof of \cref{thm:concentration}.]
	The proof has three moves: writing out the discounted log-contrast explicitly, bounding the posterior tail through the posterior odds, and inserting the simultaneous contrast bound.

	\emph{Writing out the discounted log-contrast.} Dividing \eqref{eq:power_posterior} evaluated at $P$ by the same expression at $P^\star$ eliminates the normalizer $Z_{j,T}$ and yields the posterior-odds identity
	\begin{equation}
		\frac{p_{j,T}(P)}{p_{j,T}(P^\star)}
		=
		\frac{\pi_0(P)}{\pi_0(P^\star)}
		\exp\!\big[-L_{j,T}(P)\big],
		\label{eq:proof_odds_identity}
	\end{equation}
	where the accumulated local and discounted-import log-likelihood ratio in favor of $P^\star$ is
	\begin{equation}
		L_{j,T}(P)
		:=
		\sum_{(h,y)\in\mathcal D_{j,T}}
		\log\frac{p_j(y\mid h,P^\star)}{p_j(y\mid h,P)}
		+
		\sum_{(i,h,y,\alpha)\in\mathcal I_{j,T}}
		\alpha\,
		\log\frac{q_{j\leftarrow i}(y\mid h,P^\star)}{q_{j\leftarrow i}(y\mid h,P)}.
		\label{eq:proof_contrast_explicit}
	\end{equation}
	This is exactly the quantity of Assumption~\ref{ass:contrast}: local observations enter with weight $1$ and imports enter with their discount $\alpha_{j\leftarrow i}$, all weights lie in $[0,1]$, and each weight is computed before the corresponding factor is injected into the posterior, hence is predictable with respect to the branch history. Consequently $N^{\mathrm{eff}}_{j,T}$, the sum of these weights, counts one full effective sample per local observation and $\alpha$ per import. The finiteness clause of Assumption~\ref{ass:contrast} ensures that every logarithm in \eqref{eq:proof_contrast_explicit} is finite almost surely on data generated under $P^\star$, so the odds identity is well defined with no zero-over-zero ambiguity.

	\emph{Bounding the posterior tail through the odds.} The odds identity above also shows $p_{j,T}(P^\star)>0$ almost surely: its unnormalized mass is a product of $\pi_0(P^\star)>0$ with likelihood factors that are positive almost surely on data generated under $P^\star$ by the finiteness clause. Since moreover $p_{j,T}(P^\star)\le1$, dividing by it can only enlarge each term, so
	\begin{equation}
		1-p_{j,T}(P^\star)
		=
		\sum_{P\ne P^\star}p_{j,T}(P)
		\;\le\;
		\sum_{P\ne P^\star}\frac{p_{j,T}(P)}{p_{j,T}(P^\star)}
		=
		\sum_{P\ne P^\star}\frac{\pi_0(P)}{\pi_0(P^\star)}\exp\!\big[-L_{j,T}(P)\big],
		\label{eq:proof_odds_sum}
	\end{equation}
	where the last equality is \eqref{eq:proof_odds_identity}. No probabilistic argument has been used: \eqref{eq:proof_odds_sum} holds pathwise on every realization of the evidence.

	\emph{Inserting the simultaneous contrast bound.} Let $\mathcal E_{\mathrm{conc}}$ denote the event in \eqref{eq:discounted_contrast_bound}, which has probability at least $1-\delta$ and holds simultaneously over all rounds $T$ and all $P\ne P^\star$. On $\mathcal E_{\mathrm{conc}}$, for every $P\ne P^\star$,
	\[
		L_{j,T}(P)
		\;\ge\;
		\kappa_j(P)\,N^{\mathrm{eff}}_{j,T}-\mathfrak r_j(N^{\mathrm{eff}}_{j,T},\delta)
		\;\ge\;
		\kappa_j\,N^{\mathrm{eff}}_{j,T}-\mathfrak r_j(N^{\mathrm{eff}}_{j,T},\delta),
	\]
	because $\kappa_j(P)\ge\kappa_j:=\min_{P\ne P^\star}\kappa_j(P)$ and $N^{\mathrm{eff}}_{j,T}\ge0$. Substituting into \eqref{eq:proof_odds_sum} and using $\pi_0(P)\le1$ for each of the $|\bar{\mathcal P}|-1$ wrong principles,
	\begin{align}
		1-p_{j,T}(P^\star)
		&\;\le\;
		\frac{\exp\!\big\{-\kappa_jN^{\mathrm{eff}}_{j,T}+\mathfrak r_j(N^{\mathrm{eff}}_{j,T},\delta)\big\}}{\pi_0(P^\star)}
		\sum_{P\ne P^\star}\pi_0(P)
		\notag\\
		&\;\le\;
		\frac{|\bar{\mathcal P}|-1}{\pi_0(P^\star)}
		\exp\!\big\{-\kappa_jN^{\mathrm{eff}}_{j,T}+\mathfrak r_j(N^{\mathrm{eff}}_{j,T},\delta)\big\},
		\label{eq:proof_concentration_final}
	\end{align}
	which is the claim; the prefactor is finite because $\pi_0(P^\star)>0$ by Assumption~\ref{ass:finite_positive}. Since $\mathcal E_{\mathrm{conc}}$ is uniform in $T$, the same conclusion holds with $N^{\mathrm{eff}}_{j,\tau}$ in place of $N^{\mathrm{eff}}_{j,T}$ at any adaptive routing or stopping time $\tau$: no optional-sampling correction is required.
\end{proof}

\begin{remark}[Constants and sharpness]
	\label{rem:concentration_constants}
	Two features of \eqref{eq:proof_concentration_final} affect its use. First, the prior enters only through $1/\pi_0(P^\star)$: with a uniform prior on $M=|\bar{\mathcal P}|$ principles the prefactor is $M(M-1)$, so concentration is fastest relative to the universe size when the prior already places non-negligible mass on $P^\star$. Second, the first line of \eqref{eq:proof_concentration_final} shows that the sharper, prior-dependent constant $\sum_{P\ne P^\star}\pi_0(P)=1-\pi_0(P^\star)$ is available; \cref{thm:concentration} states the looser $|\bar{\mathcal P}|-1$ form, which is prior-independent and matches the union-bound structure of the argument. Finally, the bound is vacuous until $\kappa_jN^{\mathrm{eff}}_{j,T}$ exceeds $\mathfrak r_j(N^{\mathrm{eff}}_{j,T},\delta)+\log\!\big((|\bar{\mathcal P}|-1)/\pi_0(P^\star)\big)$; it is informative only once the effective evidence count has cleared this crossover, which is the formal content of the informal claim that an import counts only at its discounted weight.
\end{remark}

\subsection{Proof of negative-transfer control}
\label{app:proof_gate}

\begin{proof}[Proof of \cref{thm:negative_transfer}.]
	Let $\mathcal C_T$ denote the collection of all triples $(j,t,C)$ for which candidate $C$ is presented to branch $j$'s gate during the run. For each $(j,t,C)\in\mathcal C_T$ define the calibration event
	\[
		\mathcal E_{j,t,C}
		:=
		\big\{\,|\widehat V_{j,t}(C)-V_{j,t}(C)|\le r_{j,t}(C,\delta_{j,t,C})\,\big\},
		\qquad
		\Pr(\mathcal E_{j,t,C}^c)
		=
		\mathbb E\!\left[\Pr\!\left(\mathcal E_{j,t,C}^c\,\middle|\,\mathcal F_{t-1},C\right)\right]
		\le
		\delta_{j,t,C},
	\]
	where the conditional bound is \eqref{eq:voi_calibration} and the outer expectation removes the conditioning. Conditioning on $(\mathcal F_{t-1},C)$ is what legitimizes treating an adaptively generated candidate as fixed at the moment it is gated; this is the only place where Assumption~\ref{ass:voi_calibration} is used.

	\emph{The gate is safe on each calibration event.} On $\mathcal E_{j,t,C}$, if the true net value satisfies $V_{j,t}(C)\le\eta$, then
	\[
		\widehat V_{j,t}(C)-r_{j,t}(C,\delta_{j,t,C})
		\;\le\;
		V_{j,t}(C)\;\le\;\eta,
	\]
	so the strict inequality in \eqref{eq:safe_voi_gate} fails and the candidate is rejected. Contrapositively, any \emph{accepted} candidate with $V_{j,t}(C)\le\eta$ must lie in $\mathcal E_{j,t,C}^c$; hence
	\[
		\big\{\text{some accepted candidate has }V_{j,t}(C)\le\eta\big\}
		\;\subseteq\;
		\bigcup_{(j,t,C)\in\mathcal C_T}\mathcal E_{j,t,C}^c.
	\]

	\emph{Union-bounding over the adaptive candidate stream.} Whether a triple $(j,t,C)$ is ever gated is itself determined by the history and the generated candidate, i.e., the event $\{(j,t,C)\in\mathcal C_T\}$ is $\sigma(\mathcal F_{t-1},C)$-measurable. Hence each calibration statement can be restricted to the realized stream without changing its bound:
	\[
		\Pr\!\big(\mathcal E_{j,t,C}^c\cap\{(j,t,C)\in\mathcal C_T\}\big)
		=
		\mathbb E\!\left[\mathbf 1\{(j,t,C)\in\mathcal C_T\}\,\Pr\!\left(\mathcal E_{j,t,C}^c\,\middle|\,\mathcal F_{t-1},C\right)\right]
		\;\le\;
		\delta_{j,t,C}.
	\]
	A union bound over the stream then gives
	\[
		\Pr\!\big(\text{any accepted import has }V_{j,t}(C)\le\eta\big)
		\;\le\;
		\Pr\!\left(\bigcup_{(j,t,C)\in\mathcal C_T}\mathcal E_{j,t,C}^c\right)
		\;\le\;
		\sum_{j,t,C}\delta_{j,t,C},
	\]
	where the last sum runs over the per-candidate budgets allocated along the realized stream, exactly as in the theorem statement. Adaptivity of the stream is therefore harmless: the failure probabilities add regardless of how the sequence of candidates was generated.
\end{proof}

\begin{remark}[Budgeting the failure probability]
	\label{rem:gate_budget}
	If a deterministic bound $M_T$ on the number of gated candidates is known, the uniform allocation $\delta_{j,t,C}=\delta/M_T$ caps the total failure probability at $\delta$. For an unbounded adaptive stream the budget is spent across the sequence---for example, assigning $\delta_m=6\delta/(\pi^2 m^2)$ to the $m$-th candidate---so that $\sum_{m\ge1}\delta_m=\delta$ and the guarantee holds for the entire run at level $\delta$.
\end{remark}

\begin{remark}[What the theorem does and does not control]
	\label{rem:gate_scope}
	The guarantee is stated relative to the \emph{defined} net value \eqref{eq:true_value}, i.e., the entropy-reduction potential net of the certified transfer cost, and it is only as strong as the radius $r_{j,t}$ entering the gate. As delimited in \cref{subsec:theoretical_framework}, the deterministic implementation's exact-enumeration radius $r=0$ certifies the numerical update but not model misspecification or residual context shift; an externally validated or replication-based radius is required for the theorem to bind in deployment. In particular, bounded rewards alone do not make harmful imports rare---the control comes entirely from the calibration of $\widehat V_{j,t}$.
\end{remark}

\subsection{Proof of parallel discovery under hazards}
\label{app:proof_discovery}

\begin{proof}[Proof of \cref{thm:discovery}.]
	Let $D_{j,s}$ denote the event that branch $j$ discovers a $P^\star$-consistent discriminator in round $s$, and let
	\[
		N_s
		:=
		\{\tau_{\mathrm{disc}}>s\}
		=
		\bigcap_{s'\le s}\bigcap_j \overline{D_{j,s'}}
	\]
	be the event that no branch has discovered by the end of round $s$; note $N_s\in\mathcal F_s$.

	\emph{One round of non-discovery.} On the event $N_{s-1}$ every branch is still in the pre-discovery regime, so Assumption~\ref{ass:discovery_hazard} applies: conditionally on $\mathcal F_{s-1}$, the discovery events $D_{j,s}$ of the active branches are independent with $\Pr(D_{j,s}\mid\mathcal F_{s-1})\ge\lambda_j$. Therefore, on $N_{s-1}$,
	\[
		\Pr\!\left(N_s\,\middle|\,\mathcal F_{s-1}\right)
		=
		\Pr\!\left(\textstyle\bigcap_j \overline{D_{j,s}}\,\middle|\,\mathcal F_{s-1}\right)
		=
		\prod_j\big(1-\Pr(D_{j,s}\mid\mathcal F_{s-1})\big)
		\;\le\;
		\prod_j(1-\lambda_j),
	\]
	where the middle equality is the conditional independence across branches. The product ranges over the branches active in round $s$; a branch that is inactive in round $s$ is equivalently covered by reading its hazard floor as $0$ for that round, which contributes a factor $1$, so the displayed bound holds with the $\lambda_j$ of the branches that are active throughout the pre-discovery regime. The conditional independence is exactly where the argument uses that the branches' discovery mechanisms share no common failure mode beyond the observed history.

	\emph{Iterating the one-round bound over rounds.} Since $N_s\subseteq N_{s-1}$ and $N_{s-1}\in\mathcal F_{s-1}$, the conditional probability $\Pr(N_s\mid\mathcal F_{s-1})$ vanishes off $N_{s-1}$; hence the one-round bound applies through the tower property,
	\[
		\Pr(N_s)
		=
		\mathbb E\!\left[\mathbf 1\{N_{s-1}\}\,\Pr(N_s\mid\mathcal F_{s-1})\right]
		\;\le\;
		\Big(\prod_j(1-\lambda_j)\Big)\,\Pr(N_{s-1}).
	\]
	Starting from $\Pr(N_0)=1$ and iterating over $s=1,\ldots,t$,
	\[
		\Pr(\tau_{\mathrm{disc}}>t)
		=
		\Pr(N_t)
		\;\le\;
		\prod_j(1-\lambda_j)^t.
	\]

	\emph{Converting the product bound to exponential form.} Since $\log(1-x)\le-x$ for $x\in[0,1)$ (if some $\lambda_j=1$ the product is $0$ and the claim is trivial),
	\[
		\prod_j(1-\lambda_j)^t
		=
		\exp\!\Big(t\sum_j\log(1-\lambda_j)\Big)
		\;\le\;
		\exp\!\Big(-t\sum_j\lambda_j\Big),
	\]
	which is \eqref{eq:discovery_tail}.
\end{proof}

\begin{remark}[Expectation bound and the role of independence]
	\label{rem:discovery_expectation}
	Summing the tail bound yields an expectation bound at no extra cost: since $\tau_{\mathrm{disc}}$ is a positive integer-valued random variable,
	\[
		\mathbb E[\tau_{\mathrm{disc}}]
		=
		\sum_{t\ge0}\Pr(\tau_{\mathrm{disc}}>t)
		\;\le\;
		\sum_{t\ge0}\Big(\prod_j(1-\lambda_j)\Big)^t
		=
		\frac{1}{1-\prod_j(1-\lambda_j)},
	\]
	by the geometric-series identity. With a common hazard $\lambda$ this reads $\mathbb E[\tau_{\mathrm{disc}}]\le 1/(1-(1-\lambda)^K)\approx 1/(K\lambda)$ for small $\lambda$---the near-$K$-fold speedup of parallel search. The conditional-independence clause of Assumption~\ref{ass:discovery_hazard} is essential for this conclusion: were the discovery events perfectly coupled, the one-round non-discovery probability would be $1-\max_j\lambda_j$ instead of $\prod_j(1-\lambda_j)$, and no speedup would follow. The theorem therefore certifies acceleration only to the extent that the branch construction genuinely diversifies the discovery mechanisms, which is a measurable property of the design rather than a consequence of running $K$ copies.
\end{remark}

\newpage
\section{Theorem--Experiment Alignment}
\label{app:theory_empirical_alignment}

The four guarantees of \cref{subsec:guarantees} certify information-theoretic mechanisms---coverage rate, posterior concentration, gate failure probability, and discovery hazards---whereas the run logs record quality, time, and token surrogates. The alignment therefore runs at the level of \emph{predicted signatures}: each theorem implies a falsifiable pattern in the reported measurements, checked against the tables and figures of \cref{subsec:performance_analysis,subsec:efficiency_analysis,subsec:ablation_study}. The four alignments below state the signature and its evidence; each closes with the conclusion the evidence supports.

\noindent \textbf{(\cref{thm:parallel_coverage}) Per-round coverage advantage appears as the wall-clock and equal-time leads.} The bound $(1-\chi_T)\underline\gamma_T\sum_j\mu_j$ predicts two signatures: the shared budget is exhausted sooner, and an equal-time read shows a quality lead. Both hold---\method completes the shared 72-evaluation budget $1.03{\times}$--$2.85{\times}$ faster than PiEvo on all six tasks (mean $1.80{\times}$; Figure~\ref{fig:efficiency}(a)), $\Delta\mathrm{SQ}@T_c$ averages $+8.3$ SQ points and is positive in $15$ of $18$ task--seed pairs (Figure~\ref{fig:efficiency}(c)), and \method attains the best average APD on both backbones ($52.0$ and $58.1$; Table~\ref{tab:apd_auoc}), the $\underline\gamma_T$ side of the bound.
\begin{takeaway}
	\textbf{Takeaway.} \method exhausts the shared evaluation budget $1.03{\times}$--$2.85{\times}$ faster than PiEvo on all six tasks, leads by $+8.3$ SQ on average at equal time, and attains the widest exploration (APD) on both backbones.
\end{takeaway}

\noindent \textbf{(\cref{thm:concentration}) imports help only at their discounted weight.} Since the theorem counts an import toward $N^{\mathrm{eff}}_{j,T}$ at weight $\alpha_{j\leftarrow i}$, two signatures must hold jointly: calibrated sharing raises attainment, and stripping the discount ($\alpha{\equiv}1$) damages it. Both hold---the shared arm's best-of-seed SQ exceeds the Isolated arm on all six tasks ($+13.6$ MBO, $+13.3$ Promoter, $+2.5$ SPO, $+0.5$ NHO and TMC, and $+7.9$ AMP with $p{<}0.05$; Table~\ref{tab:ablation_sharing_arms}), removing the discount collapses quality by $19.0$ SQ, the largest knockout effect (Table~\ref{tab:ablation_module_knockouts}), and the flow itself is real: $244$ accepted imports, $205$ net-unique, with uplift of $+15.4$ (MBO), $+25.4$ (AMP), and $+43.3$ (TMC) SQ points (Table~\ref{tab:sharing_diagnostics}).
\begin{takeaway}
	\textbf{Takeaway.} Calibrated sharing improves best-of-seed quality over the Isolated arm on all six tasks, while applying the same imports at full weight ($\alpha{\equiv}1$) costs $19.0$ SQ---the largest single-module degradation---and the accepted imports themselves carry observed uplift of up to $+43.3$ SQ.
\end{takeaway}

\noindent \textbf{(\cref{thm:negative_transfer}) the gate behaves as a tail-probability control.} The theorem bounds the \emph{probability} of a harmful accepted import, so disabling the gate should fatten the lower tail, not move the mean. Removing the gate admits $137$ imports against $42$, yet the mean moves only $-2.1$ SQ while the standard deviation more than doubles ($4.6\to11.1$) with two of three seeds degraded (Table~\ref{tab:ablation_module_knockouts}); attainment shows no monotone response to the margin $\eta$ even though $\eta$ is the strongest admission-volume actuator (Figure~\ref{fig:ablation_hparams}); and on Promoter the gate admits a single import across the three seeds, so its selectivity is genuine and the guarantee non-vacuous (Table~\ref{tab:sharing_diagnostics}).
\begin{takeaway}
	\textbf{Takeaway.} Removing it admits over $3{\times}$ more imports ($42\to137$) and more than doubles the outcome spread ($4.6\to11.1$) while shifting the mean by only $-2.1$ SQ---the signature of protection against rare harmful imports.
\end{takeaway}

\noindent \textbf{(\cref{thm:discovery}): the discovery rescue is gated by hazard diversity, not by $K$.} The tail bound $\Pr(\tau_{\mathrm{disc}}>t)\le\prod_j(1-\lambda_j)^t$ predicts the largest advantage where a single branch's hazard $\lambda$ is near zero: on SPO, systems without principle branching (Vanilla MAS, The AI Scientist v1) collapse to $0.0\%$ under both backbones while \method sustains $15.3\%$--$30.0\%$, $1.5{\times}$--$2.8{\times}$ PiEvo ($10.4\%$--$10.6\%$; Table~\ref{tab:main_compare_8tasks}). The branch-count sweep confirms the caveat of Remark~\ref{rem:discovery_expectation}: at fixed budget the response to $K$ is an inverted-U on MBO peaking at $K{=}3$, flat on AMP, and a plateau matching the PiEvo reference on TMC---landscape-gated, not $K$-gated (Figure~\ref{fig:ksweep}).
\begin{takeaway}
	\textbf{Takeaway.} \method sustains $15.3\%$--$30.0\%$ on SPO, where single-branch systems stall at $0.0\%$, and the response to branch count at fixed budget peaks at $K{=}3$---the gain comes from diversified discovery hazards, not from running more branches.
\end{takeaway}

\newpage
\section{Implementation Details of the Coordination Loop}
\label{app:loop_implementation}

\method instantiates the three quantities of \cref{subsec:theoretical_framework}---the trust score $\rho_{i,t}$, the transfer cost $\Delta^{\mathrm{imp}}_{j,t}$, and the value estimate $\widehat V_{j,t}$---and assembles them into the per-round loop of Algorithm~\ref{algo:method}:

\textbf{Trust $\rho_{i,t}=1/(1+\bar e_{i,t})$.} We set $\bar e_{i,t}$ as the source's mean absolute standardized GP residual~\citep{Pu2026PrincipleEvolvableSD}, which discounts sources that mispredict their own measurements. The $s_{ij,t}$ scores context transferability to target $j$ from task identity, the declared shared observation model, and subspace compatibility, and $v_{i,t}=1$ after independent target replication and $0.5$ otherwise. Missing estimates fail closed at zero. The target snapshot supplies the predictive entries $f_{j,P}(h)$ and $\sigma^2_{j,P}(h)$ for every supported hypothesis and principle.

\textbf{Cost $\Delta^{\mathrm{imp}}_{j,t}(C)$.}  It prices the transfer mechanics as constants normalized to the entropy scale: parsing $c_{\mathrm{read}}$, verification of unreplicated records $c_{\mathrm{verify}}$, posterior fitting of novel records $c_{\mathrm{fit}}$, and context shift
$
    c_{\mathrm{shift}}=(1-s_{ij,t})c_{\mathrm{shift}}^{\max}.
$
Context shift is the only state-dependent term; driven by the same setting-match score as the discount, it makes a less compatible source simultaneously downweighted in $\alpha$ and more expensive to import.

\textbf{Value $V_{j,t}(C)$.} It is a conditional expectation over residual randomness; the core evaluates it exactly over the finite principle universe by cloning the target posterior, applying $(h,y)$ at weight $\alpha$, and measuring the induced entropy change. To align this information score with the task's maximization objective, it is weighted by the source outcome's empirical-rank relevance $w_{\mathrm{rel}}(y)\in[w_{\min},1]$:
\[
\widehat V_{j,t}(C)=w_{\mathrm{rel}}(y)\bigl[H(p_{j,t})-H(\widetilde p^+_{j,t})\bigr]-\lambda\widehat\Delta^{\mathrm{imp}}_{j,t}(C),
\]
where $\widetilde p^+_{j,t}$ is the hypothetical clone, committed only if the candidate passes the gate. $\widehat V_{j,t}(C)$ departs from Eq.~\eqref{eq:true_value} only through $w_{\mathrm{rel}}$, the cost constants, and exact enumeration over the finite $\bar{\mathcal P}$ in place of the expectation.

\noindent \textbf{Redundancy cap $R_{\max}$.} Routing deduplicates each branch's per-round admitted set by text: after the gate and ranking, candidates are admitted greedily in ranked order, and a candidate is dropped if its text has token-Jaccard similarity $J(A,B)=\nicefrac{|A\cap B|}{|A\cup B|}\ge R_{\max}$ with an already-admitted one, keeping only the highest-ranked near-duplicate per round (default $R_{\max}{=}0.3$; $R_{\max}{=}1$ disables the cap, admitting only exact-text duplicates---the knock-out arm of Table~\ref{tab:ablation_module_knockouts}). This is distinct from the evidence pool's exact-hash deduplication: the cap is an approximate, per-round, per-branch filter at routing time.

\noindent \textbf{Monitoring.} The core logs accepted and rejected candidates, discount components, duplicate rates, replication outcomes, predictive-density checks, posterior entropy changes, and coordination time; rejected records remain in the pool without altering target posteriors.

\noindent \textbf{Gaussian realization of the discounted update.} Under Gaussian predictive models with fixed observation noise and a sample weight acting as a precision multiplier, the discounted update of Eq.~\eqref{eq:power_posterior} contributes
$
-\frac{\alpha_{j\leftarrow i,t}(h)}{2\sigma_{\mathrm{obs}}^2}
\bigl(y-f_{j,P}(h)\bigr)^2+\mathrm{const}
$
to the target log posterior: an accepted import is a local observation with variance $\sigma_{\mathrm{obs}}^2/\alpha$.

\noindent \textbf{Default constants and reproducibility.}
Table~\ref{tab:coordination_constants} lists every constant of the coordination loop. The relevance weight is the source outcome's mid-rank percentile within the source branch's observed outcomes, lifted to $[w_{\min},1]$: $w_{\mathrm{rel}}(y)=w_{\min}+(1-w_{\min})\,\mathrm{pct}(y)$ with $\mathrm{pct}(y)=\bigl(\#\{y'<y\}+\nicefrac{1}{2}\max(\#\{y'=y\}-1,0)\bigr)/(n-1)$ and $w_{\min}=0.1$; a candidate implausible under the entire target posterior (log predictive density below $-25$) is rejected outright. The per-candidate confidence budget follows the alpha-spending schedule $\delta_m=6\delta/(\pi^2 m^2)$ with total budget $\delta=0.05$ (Remark~\ref{rem:gate_budget}). The per-task principle universe $\bar{\mathcal P}$, prior $\pi_0$, and GP outcome models follow PiEvo~\citep{Pu2026PrincipleEvolvableSD} unchanged. All runs use three fixed seeds per cell; rejected candidates and surrogate or endpoint failures score $0.0$ and are excluded from trajectories as infrastructure events (\cref{sec:benchmark}). Code, configuration files, and per-run logs will be released.

\begin{table}[h!]
	\centering
	\caption{\textbf{Default constants of the coordination core.} All values are fixed across tasks, backbones, and ablations unless the ablation varies them explicitly.}
	\label{tab:coordination_constants}
	\small
	\setlength{\tabcolsep}{5pt}
	\begin{tabular}{lll}
		\toprule
		\textbf{Symbol} & \textbf{Default} & \textbf{Role} \\
		\midrule
		$\eta$ & $0.0$ & safe-VoI gate margin (Eq.~\eqref{eq:safe_voi_gate}) \\
		$B$ & $3$ & routing quota per branch per round \\
		$R_{\max}$ & $0.3$ & redundancy cap (token-Jaccard) \\
		$\lambda$ & $1.0$ & import cost weight (Eq.~\eqref{eq:true_value}) \\
		$\delta$ & $0.05$ & total false-import budget (alpha-spending) \\
		$c_{\mathrm{read}}$ & $5\times10^{-4}$ & parsing cost \\
		$c_{\mathrm{verify}}$ & $2.5\times10^{-4}$ & verification cost (unreplicated records) \\
		$c_{\mathrm{fit}}$ & $2.5\times10^{-4}$ & posterior fitting cost \\
		$c_{\mathrm{shift}}^{\max}$ & $2\times10^{-2}$ & maximal context-shift surcharge \\
		$w_{\min}$ & $0.1$ & relevance floor of $w_{\mathrm{rel}}(y)$ \\
		$\varepsilon$ & $10^{-9}$ & IDS ratio denominator guard (Eq.~\eqref{eq:ids_ratio}) \\
		\bottomrule
	\end{tabular}
\end{table}

Algorithm~\ref{algo:method} assembles these quantities into one coordination round.

\begin{algorithm}[h!]
    \caption{\method coordination core (one round $t$)}
    \label{algo:method}
    \begin{algorithmic}[1]
        \REQUIRE branches $\{1,\dots,K\}$ with posteriors $\{p_{j,t-1}\}$, evidence pool $H^{\mathrm{sys}}_{t-1}$; margin $\eta$, routing quota $B$
        \ENSURE updated posteriors $\{p_{j,t}\}$, extended pool $H^{\mathrm{sys}}_{t}$, logged decisions
        \STATE \textbf{Collect:} each branch $i$ publishes its round-$t$ tested pairs as candidate records $C=(i,h,y,E_i,m)$;
        \STATE \textbf{Merge:} validate, deduplicate, and append records to the append-only pool; no posterior is touched;
        \FORALL{targets $j$ and pooled records $C$ from sources $i\neq j$}
            \STATE $\alpha_{j\leftarrow i,t}(h)\leftarrow\operatorname{clip}_{[0,1]}\!\left(\rho_{i,t}\,s_{ij,t}\,v_{i,t}\right)$ \COMMENT{Eq.~\eqref{eq:discount_factor}}
            \STATE clone $p_{j,t-1}$, enumerate the discounted update, set $\widehat V_{j,t}(C)$ with $r=0$ \COMMENT{Eq.~\eqref{eq:true_value}}
            \STATE $\Delta^{\mathrm{imp}}_{j,t}(C)\leftarrow c_{\mathrm{read}}+c_{\mathrm{verify}}+c_{\mathrm{fit}}+(1-s_{ij,t})\,c_{\mathrm{shift}}^{\max}$
        \ENDFOR
        \STATE \textbf{Route:} admit $C$ to $j$ only if $\widehat V_{j,t}(C)-r>\eta$; rank admitted records by $\nicefrac{\Delta^{\mathrm{imp}}_{j,t}(C)^2}{\max\{\widehat V_{j,t}(C)-r,0\}+\varepsilon}$ and keep the top $B$ \COMMENT{Eqs.~\eqref{eq:safe_voi_gate}--\eqref{eq:ids_ratio}}
        \FORALL{targets $j$}
            \STATE \textbf{Inject:} multiply $p_{j,t}$ by $q_{j\leftarrow i}(y\mid h,P)^{\alpha}$ per admitted import; record the decision \COMMENT{Eq.~\eqref{eq:power_posterior}}
        \ENDFOR
        \STATE \textbf{Monitor:} log diversity and audit diagnostics feeding the scoring of round $t{+}1$.
    \end{algorithmic}
\end{algorithm}

\newpage
\section{Ablation of Hyperparameters in \method}

\noindent \textbf{Branch-count ablation.} As shown in Figure~\ref{fig:ksweep}, the response of \method tends to be landscape-dependent: an inverted-U on MBO peaking at $K{=}3$, flat on AMP, and on TMC a rise into a plateau that matches or exceeds the PiEvo reference from $K{=}3$ onward.

\begin{figure*}[h!]
	\centering
	\includegraphics[width=0.95\textwidth]{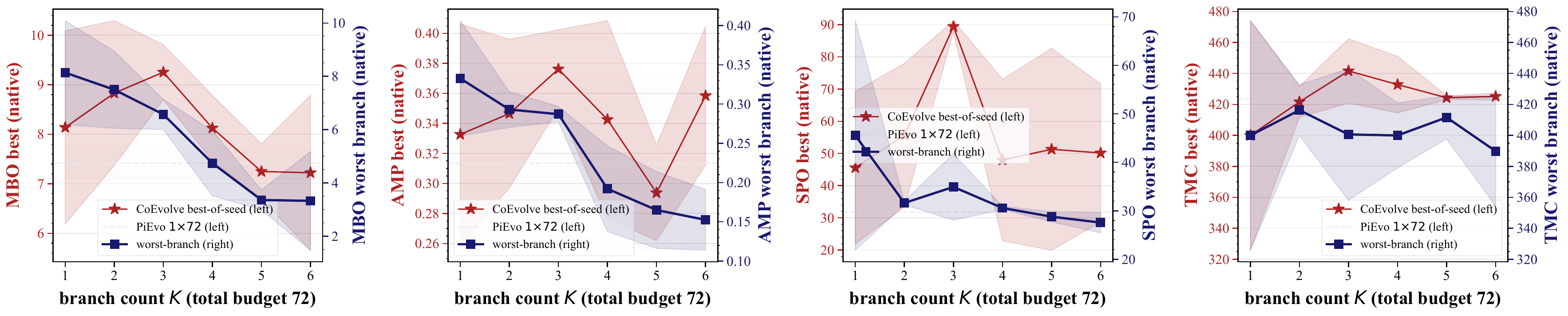}
	\caption{
        \textbf{Branch-count ablation.} 
        In $K{=}5$, we use unequal per-branch budgets $\{15,15,14,14,14\}$ summing to the same total of 72.
        The \textcolor{gray}{gray} dashed line and band mark PiEvo~\citep{Pu2026PrincipleEvolvableSD}.}
	\label{fig:ksweep}
\end{figure*}

\noindent \textbf{Hyperparameter sensitivity.} We ablate the safe-VoI gate margin $\eta$, the redundancy cap $R_{\max}$ (extended to the fully-disabled point $R_{\max}{=}1$ from Table~\ref{tab:ablation_module_knockouts}), the import cost weight $\lambda$, and the per-round routing quota in Figure~\ref{fig:ablation_hparams}. The best-of-seed SQ peaks at the default of $R_{\max}$ and the routing quota, sits on a plateau spanning $\lambda{\in}[1,2]$, and shows no monotone response to $\eta$. This is consistent with the gate's role as protection against \emph{rare} harmful imports, not an average-case quality lever (Table~\ref{tab:ablation_module_knockouts}).

\begin{figure*}[h!]
	\centering
	\includegraphics[width=0.95\textwidth]{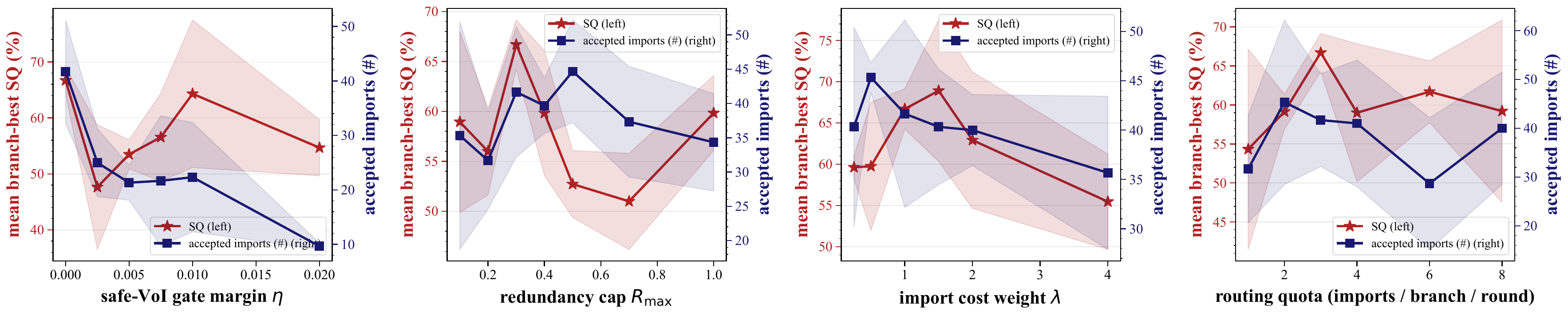}
	\caption{
        \textbf{Hyperparameter sensitivity.} 
        Each panel varies exactly one coordination hyperparameter and pairs the \emph{best-of-seed SQ} with the accepted-import volume.}
	\label{fig:ablation_hparams}
\end{figure*}

\newpage
\section{Ablation of Evidence Transfer in \method}
\label{sec:sharing_diagnostics}

\begin{table*}[h!]
	\centering
	\caption{\textbf{Ablation of Evidence transfer.} 
    $\Delta$ is the difference of arm means with $^\star$ marking Welch's $t$ significance at $p<0.05$; \textcolor[HTML]{7D2AD1}{purple} favors sharing. Promoter's Isolated arm runs the current $T{=}1.1$ protocol.}
	\label{tab:ablation_sharing_arms}
	\small
	\setlength{\tabcolsep}{4pt}
	\resizebox{0.8\textwidth}{!}{\begin{tabular}{l ccc cccc}
		\toprule[1.5pt]
		& \multicolumn{3}{c}{\textbf{Worst-branch SQ (\%)}} & \multicolumn{4}{c}{\textbf{Best-of-seed SQ (\%)}}\\
		\cmidrule(lr){2-4} \cmidrule(lr){5-8}
		\textbf{Task} & \method shared & Isolated & $\Delta$ & \method shared & Isolated & Independent & $\Delta$ \\
		\midrule

		\rowcolor{rowcolor}
		MBO & $54.8$\,{\scriptsize\color{gray}$\pm\,4.7$} & $45.3$\,{\scriptsize\color{gray}$\pm\,10.0$} & $\textcolor[HTML]{7D2AD1}{+9.4}$ & $77.1$\,{\scriptsize\color{gray}$\pm\,4.6$} & $63.5$\,{\scriptsize\color{gray}$\pm\,7.9$} & $53.6$\,{\scriptsize\color{gray}$\pm\,13.6$} & $\textcolor[HTML]{7D2AD1}{+13.6}$ \\

		NHO & $90.8$\,{\scriptsize\color{gray}$\pm\,2.0$} & $89.9$\,{\scriptsize\color{gray}$\pm\,2.7$} & $\textcolor[HTML]{7D2AD1}{+0.9}$ & $94.9$\,{\scriptsize\color{gray}$\pm\,0.1$} & $94.4$\,{\scriptsize\color{gray}$\pm\,0.6$} & $95.0$\,{\scriptsize\color{gray}$\pm\,0.0$} & $\textcolor[HTML]{7D2AD1}{+0.5}$ \\

		\rowcolor{rowcolor}
		SPO & $11.7$\,{\scriptsize\color{gray}$\pm\,2.3$} & $10.4$\,{\scriptsize\color{gray}$\pm\,0.4$} & $\textcolor[HTML]{7D2AD1}{+1.3}$ & $30.0$\,{\scriptsize\color{gray}$\pm\,0.6$} & $27.5$\,{\scriptsize\color{gray}$\pm\,4.9$} & $25.5$\,{\scriptsize\color{gray}$\pm\,4.2$} & $\textcolor[HTML]{7D2AD1}{+2.5}$ \\

		TMC & $80.1$\,{\scriptsize\color{gray}$\pm\,8.5$} & $84.8$\,{\scriptsize\color{gray}$\pm\,0.2$} & $\textcolor[HTML]{D14D4D}{-4.7}$ & $88.3$\,{\scriptsize\color{gray}$\pm\,4.2$} & $87.8$\,{\scriptsize\color{gray}$\pm\,4.5$} & $87.7$\,{\scriptsize\color{gray}$\pm\,4.6$} & $\textcolor[HTML]{7D2AD1}{+0.5}$ \\

		\rowcolor{rowcolor}
		AMP & $28.7$\,{\scriptsize\color{gray}$\pm\,1.0$} & $19.8$\,{\scriptsize\color{gray}$\pm\,1.4$} & $\textcolor[HTML]{7D2AD1}{+8.9}$$^\star$ & $37.6$\,{\scriptsize\color{gray}$\pm\,2.6$} & $29.7$\,{\scriptsize\color{gray}$\pm\,2.0$} & $32.3$\,{\scriptsize\color{gray}$\pm\,0.6$} & $\textcolor[HTML]{7D2AD1}{+7.9}$$^\star$ \\

		Promoter & $56.2$\,{\scriptsize\color{gray}$\pm\,14.9$} & $52.4$\,{\scriptsize\color{gray}$\pm\,4.8$} & $\textcolor[HTML]{7D2AD1}{+3.8}$ & $86.4$\,{\scriptsize\color{gray}$\pm\,0.4$} & $73.1$\,{\scriptsize\color{gray}$\pm\,10.9$} & $70.5$\,{\scriptsize\color{gray}$\pm\,13.1$} & $\textcolor[HTML]{7D2AD1}{+13.3}$ \\
		\bottomrule[1.5pt]
	\end{tabular}
	}
\end{table*}

\noindent \textbf{Per-task view of evidence transfer.}
Table~\ref{tab:ablation_sharing_arms} expands Figure~\ref{fig:ablation_sharing_arms} (Section~\ref{subsec:ablation_study}) into exact numbers: best-of-seed SQ for all three arms (shared, Isolated, Independent) and worst-branch SQ for the two coordinated arms, with $\Delta$ the shared-minus-Isolated margin. Figure~\ref{fig:ablation_sharing_arms_systemic} shows the same comparison with mean branch-best SQ, to see whether sharing lifts the whole branch population rather than only the champion lineage.

\begin{figure}[h!]
	\centering
	\includegraphics[width=0.55\textwidth]{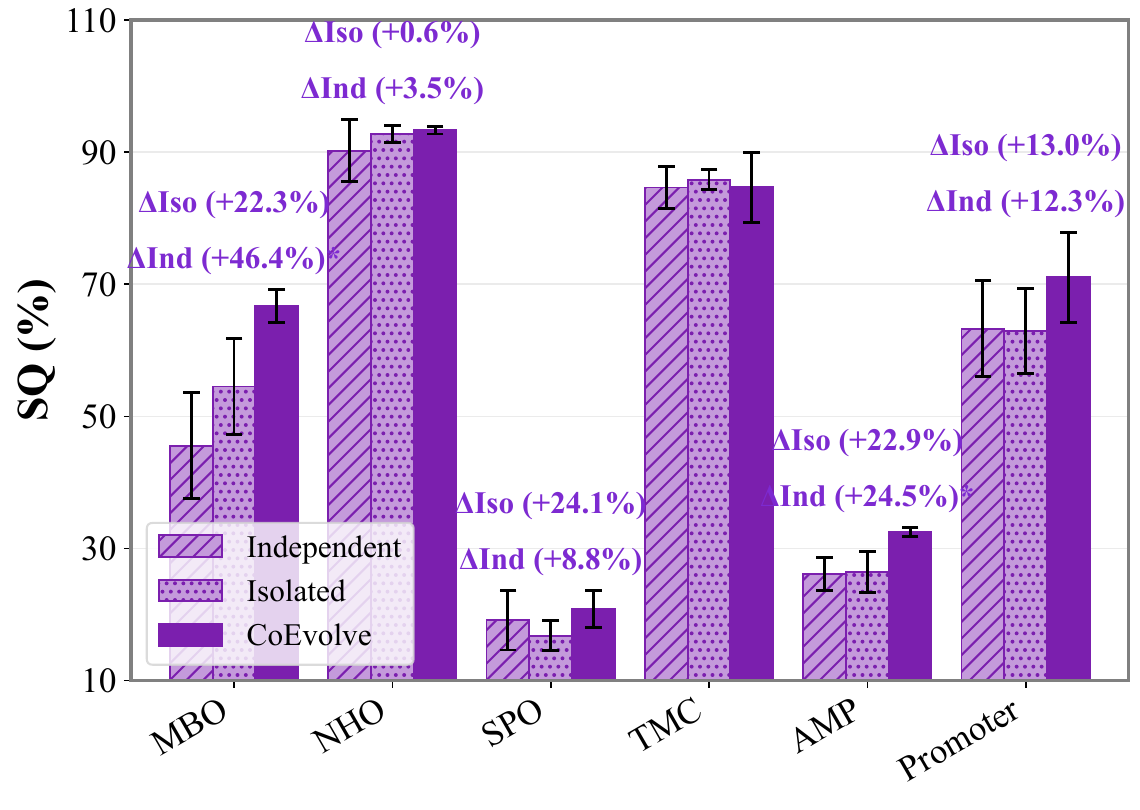}
	\caption{\textbf{Marginal value of evidence sharing at the systemic level.} 
    }
	\label{fig:ablation_sharing_arms_systemic}
\end{figure}

On two tasks the Independent--Isolated--\method ordering is not monotone. On AMP the Independent arm ($32.3 \pm 0.6$) edges out Isolated ($29.7 \pm 2.0$) within the seed noise, while the shared arm remains clearly best ($37.6 \pm 2.6$; Welch's $t$, $p{<}0.05$): coordination without sharing neither helps nor hurts there, and the entire gain comes from the evidence flow itself. 

NHO is the opposite extreme: a compact, exploitation-type landscape (a seven-dimensional continuous search space with a smooth surrogate) in which every arm saturates the same ceiling ($94.4$--$95.0$ SQ), so little complementary coverage remains to transfer and the Independent arm nominally matches \method ($95.0$ vs $94.9$, within noise).

On NHO, coordination's remaining benefit is convergence speed: in the dedicated NHO campaign both arms (\method and Independent) reach the identical ceiling ($g$-factor $1.90$) but \method converges $1.5\times$ faster in wall-clock ($120$ vs $184$ minutes), and the same signature appears on the Terra backbone as the $+10.5$ equal-time gap of Figure~\ref{fig:efficiency}~(c). On such compact landscapes, collaboration shows up as efficiency.

\noindent \textbf{Ablation of modules in \method.}
Table~\ref{tab:ablation_module_knockouts} gives the exact numbers behind the module ablation of \cref{subsec:ablation_study} (Figure~\ref{fig:ablation_module_knockouts}): each row disables exactly one transfer safeguard, and the full-system reference row is the main-table arm. Off discounting or redundancy control degrades both portfolio and worst-branch quality. Disabling the VoI gate admits over $3\times$ as many imports and degrades two of three seeds with markedly higher variance. This is consistent with the gate protecting against rare harmful imports, not average-case ones (see Section~\ref{sec:methodology}).

\begin{table}[h!]
	\centering
	\caption{\textbf{Ablation of modules in \method.}}
	\label{tab:ablation_module_knockouts}
	\small
	\setlength{\tabcolsep}{3.5pt}
	\resizebox{0.8\columnwidth}{!}{\begin{tabular}{l l cc c c c}
		\toprule[1.5pt]
		\textbf{Arm} & \textbf{knocked-out module} & \textbf{SQ} & \textbf{worst-br.} & $\Delta$\textbf{SQ} & \textbf{imports} & $\bar\alpha$ \\
		\midrule

		\rowcolor{rowcolor}
		\method (full) & all modules on (main-table arm) & $77.1$\,{\scriptsize\color{gray}$\pm\,4.6$} & $54.8$\,{\scriptsize\color{gray}$\pm\,4.7$} & -- & 42 & $0.23$ \\

		$-$ VoI gate & safe-VoI gate off ($\eta{\to}-\infty$): every routed import accepted & $75.0$\,{\scriptsize\color{gray}$\pm\,11.1$} & $52.8$\,{\scriptsize\color{gray}$\pm\,15.3$} & $\textcolor[HTML]{D14D4D}{-2.1}$ & 137 & $0.24$ \\

		\rowcolor{rowcolor}
		$-$ discounting & $\alpha{\equiv}1$: imports applied at full weight & $58.1$\,{\scriptsize\color{gray}$\pm\,4.4$} & $43.3$\,{\scriptsize\color{gray}$\pm\,9.5$} & $\textcolor[HTML]{D14D4D}{-19.0}$ & 58 & $1.00$ \\

		$-$ redundancy control & $R_{\max}{=}1$: redundancy cap disabled & $67.5$\,{\scriptsize\color{gray}$\pm\,4.4$} & $51.5$\,{\scriptsize\color{gray}$\pm\,7.8$} & $\textcolor[HTML]{D14D4D}{-9.6}$ & 34 & $0.23$ \\

		\rowcolor{rowcolor}
		$-$ sharing (Isolated) & coordination on, zero evidence flow & $63.5$\,{\scriptsize\color{gray}$\pm\,7.9$} & $45.3$\,{\scriptsize\color{gray}$\pm\,10.0$} & $\textcolor[HTML]{D14D4D}{-13.6}$ & 0 & -- \\
		\bottomrule[1.5pt]
	\end{tabular}
	}
	\vspace{-7pt}
\end{table}

Table~\ref{tab:sharing_diagnostics} reports the transfer-effects analysis of \cref{subsec:ablation_study}: the per-branch cost of budget splitting and the import-level quantities that show whether cross-branch evidence sharing actually flowed, was deduplicated, and is associated with observed quality gains.

\begin{table*}[h!]
	\centering
	\caption{\textbf{Ablation of evidence transfer with net value.}
    \emph{accepted} $=$ evidence-node$\times$branch applications; \emph{net unique} $=$ distinct evidence hashes delivered; \emph{dedup} $=$ deliveries per unique evidence (delivery redundancy); \emph{uplift} $=$ pooled-best SQ gain over evaluation windows containing a delivery, averaged over seeds: a window-level, associational statistic (deliveries are value-selected, so uplift does not isolate a causal transfer effect), unlike the branch-level, immediate \emph{inert} rate, with which it need not agree; \emph{inert} $=$ fraction of accepted imports with zero immediate running-best change. TMC and Promoter aggregate only $2$ and $1$ accepted imports. All import columns aggregate the three main-table seeds.}
	\label{tab:sharing_diagnostics}
	\small
	\setlength{\tabcolsep}{4pt}
	\resizebox{0.8\textwidth}{!}{\begin{tabular}{l ccc cc c c c}
		\toprule[1.5pt]
		& \multicolumn{3}{c}{\textbf{Per-branch SQ (\%)}} & \multicolumn{2}{c}{\textbf{Imports}} & & & \\
		\cmidrule(lr){2-4} \cmidrule(lr){5-6}
		\textbf{Task} & \method worst-branch & PiEvo & $\Delta$ & accepted & net unique & \textbf{dedup} & \textbf{uplift (SQ pts)} & \textbf{inert} \\
		\midrule

		\rowcolor{rowcolor}
		MBO & $54.8$\,{\scriptsize\color{gray}$\pm\,4.7$} & $61.8$\,{\scriptsize\color{gray}$\pm\,2.7$} & $\textcolor[HTML]{D14D4D}{-7.0}$ & 120 & 99 & $1.26\times$ & $+15.41$ & $80\%$ \\

		NHO & $90.8$\,{\scriptsize\color{gray}$\pm\,2.0$} & $94.9$\,{\scriptsize\color{gray}$\pm\,0.1$} & $\textcolor[HTML]{D14D4D}{-4.1}$ & 19 & 17 & $1.29\times$ & $+0.15$ & $89\%$ \\

		\rowcolor{rowcolor}
		SPO & $11.7$\,{\scriptsize\color{gray}$\pm\,2.3$} & $10.6$\,{\scriptsize\color{gray}$\pm\,0.0$} & $\textcolor[HTML]{7D2AD1}{+1.1}$ & 6 & 6 & $1.00\times$ & $+3.64$ & $67\%$ \\

		TMC & $80.1$\,{\scriptsize\color{gray}$\pm\,8.5$} & $84.7$\,{\scriptsize\color{gray}$\pm\,0.0$} & $\textcolor[HTML]{D14D4D}{-4.6}$ & 2 & 2 & $2.50\times$ & $+43.29$ & $100\%$ \\

		\rowcolor{rowcolor}
		AMP & $28.7$\,{\scriptsize\color{gray}$\pm\,1.0$} & $31.4$\,{\scriptsize\color{gray}$\pm\,5.8$} & $\textcolor[HTML]{D14D4D}{-2.7}$ & 96 & 80 & $1.26\times$ & $+25.41$ & $81\%$ \\

		Promoter & $56.2$\,{\scriptsize\color{gray}$\pm\,14.9$} & $54.4$\,{\scriptsize\color{gray}$\pm\,6.3$} & $\textcolor[HTML]{7D2AD1}{+1.8}$ & 1 & 1 & $1.00\times$ & $+0.00$ & $100\%$ \\

		\midrule
		\rowcolor{rowcolor}
		\bfseries Total / Avg. & & & & 244 & 205 & & & \\
		\bottomrule[1.5pt]
	\end{tabular}
	}
	\vspace{-7pt}
\end{table*}

\newpage
\section{Statistical Analysis}
\label{app:stats}

All experiments run three seeds per cell, which limits the power of any single per-task test. We therefore (i) base the headline comparisons on aggregate paired tests over the 18 task--seed pairs of the GPT-5.6-Terra backbone, where per-seed trajectories are logged for both \method and PiEvo, and (ii) report per-task tests for completeness.

\noindent \textbf{Paired per-seed tests (\method vs PiEvo, Terra backbone).}
Table~\ref{tab:paired_tests} lists the per-seed SQ differences. Aggregating across tasks, the paired gain is positive in 15 of the 17 non-tied pairs (two-sided sign test $p{=}0.002$; counting the exact tie against \method, $p{=}0.008$; Wilcoxon signed-rank $p{=}0.001$). The wall-clock speedup exceeds $1\times$ in 17 of 18 pairs (sign test $p{=}1.5\times10^{-4}$), and the equal-time gap $\Delta\mathrm{SQ}@T_c$ is positive in 15 of 17 non-tied pairs ($p{<}0.01$; Wilcoxon $p{=}8\times10^{-4}$).

\begin{table}[h!]
	\centering
	\caption{\textbf{Paired per-seed SQ differences, \method minus PiEvo (GPT-5.6-Terra).} Two-sided paired $t$ test per task over three seeds; the aggregate row pools all 18 task--seed pairs.}
	\label{tab:paired_tests}
	\small
	\setlength{\tabcolsep}{6pt}
	\begin{tabular}{lcc}
		\toprule
		\textbf{Task} & \textbf{Per-seed SQ differences} & \textbf{Paired $t$ $p$-value} \\
		\midrule
		MBO & $+5.2,\ +6.6,\ +39.3$ & $0.266$ \\
		NHO & $+9.4,\ +6.8,\ +6.9$ & $0.012$ \\
		SPO & $-0.1,\ +14.5,\ +0.3$ & $0.415$ \\
		TMC & $+2.5,\ \phantom{+}0.0,\ +1.8$ & $0.194$ \\
		AMP & $+2.0,\ +5.0,\ +5.0$ & $0.057$ \\
		Promoter & $+2.9,\ -2.8,\ +8.0$ & $0.481$ \\
		\midrule
		Aggregate (18 pairs) & 15 positive, 2 negative, 1 tie & sign $p{<}0.01$; Wilcoxon $p{=}0.001$ \\
		\bottomrule
	\end{tabular}
\end{table}

\noindent \textbf{Summary-statistic Welch tests.}
Table~\ref{tab:welch_tests} reports two-sided Welch $t$ tests computed from the rounded mean $\pm$ std of Table~\ref{tab:main_compare_8tasks} ($n{=}3$ per cell), for guidance; the underlying per-seed values accompany the code release. Against PiEvo, 4 of 12 task--backbone pairs reach $p{<}0.05$; under a Bonferroni correction over the 12 pairs ($\alpha{=}0.0042$), SPO (Gemma, $p{=}3\times10^{-4}$) and NHO (Terra, $p{=}0.004$) survive. We therefore state the closed-world claim at the level of best average SQ with the aggregate paired test above, not per-task superiority on every task.

\begin{table}[h!]
	\centering
	\caption{\textbf{Welch $t$ tests on final SQ from the summary statistics of Table~\ref{tab:main_compare_8tasks}} (two-sided, $n{=}3$ per cell, computed from rounded means and stds). $^\star$\,$p{<}0.05$; $^{\star\star}$\,survives Bonferroni over the 12 pairs.}
	\label{tab:welch_tests}
	\small
	\setlength{\tabcolsep}{5pt}
	\begin{tabular}{lcccc}
		\toprule
		& \multicolumn{2}{c}{\textbf{vs PiEvo}} & \multicolumn{2}{c}{\textbf{vs InternAgent-1.5}} \\
		\cmidrule(lr){2-3} \cmidrule(lr){4-5}
		\textbf{Task} & \textbf{Gemma} & \textbf{Terra} & \textbf{Gemma} & \textbf{Terra} \\
		\midrule
		MBO & $0.013^\star$ & $0.263$ & $0.005^\star$ & $0.453$ \\
		NHO & $1.000$ & $0.004^{\star\star}$ & $0.300$ & $0.386$ \\
		SPO & $3{\times}10^{-4}\,^{\star\star}$ & $0.414$ & $2{\times}10^{-4}\,^{\star\star}$ & $0.501$ \\
		TMC & $0.276$ & $0.184$ & $0.029^\star$ & $0.445$ \\
		AMP & $0.197$ & $0.058$ & $0.011^\star$ & $0.085$ \\
		Promoter & $0.012^\star$ & $0.526$ & $0.660$ & $0.324$ \\
		\bottomrule
	\end{tabular}
\end{table}

\noindent \textbf{Harness regime.}
On each of the five autoresearch tasks, \method's mean improvement is significantly above the anchor (one-sample $t$ against zero: ALDE $p{=}0.003$, Deconv $p{=}2\times10^{-4}$, InvScat $p{=}0.022$, D2D $p{<}10^{-4}$, MolEdit $p{=}0.006$). Pairwise differences against individual reference arms are directionally uniform but underpowered at three seeds (against PiEvo, $p$ ranges $0.07$ to $0.40$); the interval argument of \cref{subsec:harness}, the worst \method case matching or exceeding every other arm's best case on all five tasks, is the primary evidence there.

\newpage
\section{Harness-Regime Experiment Design}
\label{sec:harness_design}

\subsection{Tasks and baselines}

\noindent \textbf{Tasks and score.}
This regime uses a second suite of five research-workflow tasks sampled from AutoresearchEval~\citep{fei2026autoresearcheval}, disjoint from the closed-world optimization benchmarks of \cref{subsec:experiment_setup}: each task asks for a \emph{method artifact}: an optimizer, a pipeline, or a predictor, not a single candidate, and each ships a published SOTA anchor and a sealed grader running on held-out data:
\begin{enumerate}[leftmargin=16pt]
    \item \textbf{protein-landscape active learning} (ALDE), submitting an optimizer that allocates a 480-measurement screening budget over two four-site protein fitness landscapes;
    \item \textbf{bulk deconvolution} (Deconv), predicting tumor-microenvironment cell-type proportions for real breast-cancer bulk RNA-seq from a disjoint single-cell reference;
    \item \textbf{inverse scattering} (InvScat), reconstructing dielectric maps from complex scattered fields under two receiver-density instances;
    \item \textbf{D2D scheduling} (D2D), activating device-to-device links in a geographic interference network to maximize sum rate; and
    \item \textbf{molecule editing} (MolEdit), editing molecular graphs under chemistry-validity constraints to optimize penalized logP, QED, and similarity-constrained objectives.
\end{enumerate}
The grader aggregates per-instance improvements over the published anchor; we report $\Delta\times100$, where $\Delta>0$ means the best submitted artifact surpasses the anchor and $\Delta<0$ falls short of it.

\noindent \textbf{Budget contract and outcome source.}
The evaluation budget is enforced server-side: the evaluation service counts every request, scored or not, and rejects calls beyond the evaluation budget. All outcomes are taken from the grader's responses; numbers self-reported by the agents are discarded.

\noindent \textbf{Pre-registered task explorativeness.}
Before any run, we ordered the five tasks by an ordinal task-explorativeness index (TEI) read from the published task artifacts alone, including task text, anchor, and grader, never run behavior: \emph{free-oracle throughput} tasks (D2D, MolEdit), where a free local oracle governs the score, lowest; \emph{recipe-saturated} tasks (ALDE), where an established recipe leaves little headroom above the anchor; and \emph{open method spaces} (Deconv, InvScat), where no canonical pipeline dominates, highest. Appendix~\ref{subsec:harness_tei} tests whether the collaboration margin follows this pre-registered ordering.

\noindent \textbf{Harness substrate.}
All methods run on one substrate: a containerized coding agent (with \texttt{DeepSeek-v4-Flash-0731} model as backbone) with read-only task data, an ephemeral workspace, and a scientific Python environment. The executor endpoint is pinned and verified at launch, so no method can fall back to a different model, and every method sees the identical task description. The methods differ only in who decides the next experiment:
\begin{enumerate}[leftmargin=16pt, label=(\alph*)]
    \item \textbf{Agent-only} (Claude Code): the agent receives the task once and thereafter acts autonomously; the driver only harvests submissions, nudges toward budget completion, and restarts stalled sessions;
    \item \textbf{Agent-only} (Codex): the same protocol under a second coding agent, testing whether conclusions depend on one particular harness;
    \item \textbf{Agent-only} (Arbor~\citep{jin2026arbor}): an external autonomous research-harness method that self-verifies against a free local oracle.
    \item \textbf{PiEvo}~\citep{Pu2026PrincipleEvolvableSD}: the PiEvo strategy layer runs driver-side and injects one stage-guidance message (principle, hypothesis, or experiment) per turn while the same agent executes it in the container; the strategy layer is identical across arms (d) and (e), and its model is distinct from the executor model;
    \item \textbf{\method}: $K{=}3$ branches of arm (d), under a server-side-enforced total of $72$ evaluations per run.
\end{enumerate}

\subsection{Measurements}
\label{subsec:harness_measurement}

\noindent \textbf{Wall-clock and token accounting.}
Wall-clock time spans a run's first to last scored evaluation. Tokens are parsed from the transcripts (uuid-deduplicated) in four channels; \emph{billable} $=$ input $+$ output $+$ cache-creation, and the \emph{all-in} count adds cache-read, we report both. Table~\ref{tab:harness_tokens} gives the per-task means; \method spends $1.55\times$ PiEvo's billable tokens on average, with ALDE below $1\times$ because its branches converge early and exhaust the budget sooner.

\begin{table}[h!]
	\centering
	\caption{\textbf{Token accounting in the harness regime} (three-seed means, millions of tokens).}
	\label{tab:harness_tokens}
	\small
	\setlength{\tabcolsep}{5pt}
	\begin{tabular}{lccccc}
		\toprule
		& \multicolumn{2}{c}{\textbf{Billable}} & & \multicolumn{2}{c}{\textbf{All-in}} \\
		\cmidrule(lr){2-3} \cmidrule(lr){5-6}
		\textbf{Task} & \method & PiEvo & \textbf{ratio} & \method & PiEvo \\
		\midrule
		ALDE & $7.5$ & $14.5$ & $0.52\times$ & $350.4$ & $339.1$ \\
		Deconv & $95.4$ & $37.6$ & $2.53\times$ & $181.5$ & $90.3$ \\
		InvScat & $89.9$ & $46.3$ & $1.94\times$ & $363.8$ & $136.6$ \\
		D2D & $11.8$ & $8.8$ & $1.35\times$ & $331.4$ & $192.4$ \\
		MolEdit & $23.3$ & $16.6$ & $1.40\times$ & $628.6$ & $237.1$ \\
		\midrule
		Mean ratio & & & $1.55\times$ & & \\
		\bottomrule
	\end{tabular}
\end{table} 

\noindent \textbf{Normalization for merged curves.}
Cross-task anytime plots min--max normalize scores per task, with $0$ at the task anchor and $1$ at the best final value, so a point on the merged curve reads as the fraction of the task's observed best attained at a given wall-clock or token spend.

\subsection{Efficiency in the Harness Regime}
\label{subsec:harness_efficiency}

\noindent \textbf{Equal-time quality gap compared to PiEvo~\citep{Pu2026PrincipleEvolvableSD}.}
$\Delta@T_c$ reads every PiEvo run's running best at each \method run's completion time $T_c$ from per-evaluation timestamps and averages the difference. It is positive on all five tasks: ALDE $+8.9\%$, Deconv $+31.2\%$, InvScat $+18.4\%$, D2D $+0.1\%$, MolEdit $+3.0\%$ points, mean $+12.3\%$, and is never smaller than the budget-end gap, as shown in Figure~\ref{fig:harness_equal_time_gap}.

\begin{figure}[h!]
	\centering
	\includegraphics[width=0.55\textwidth]{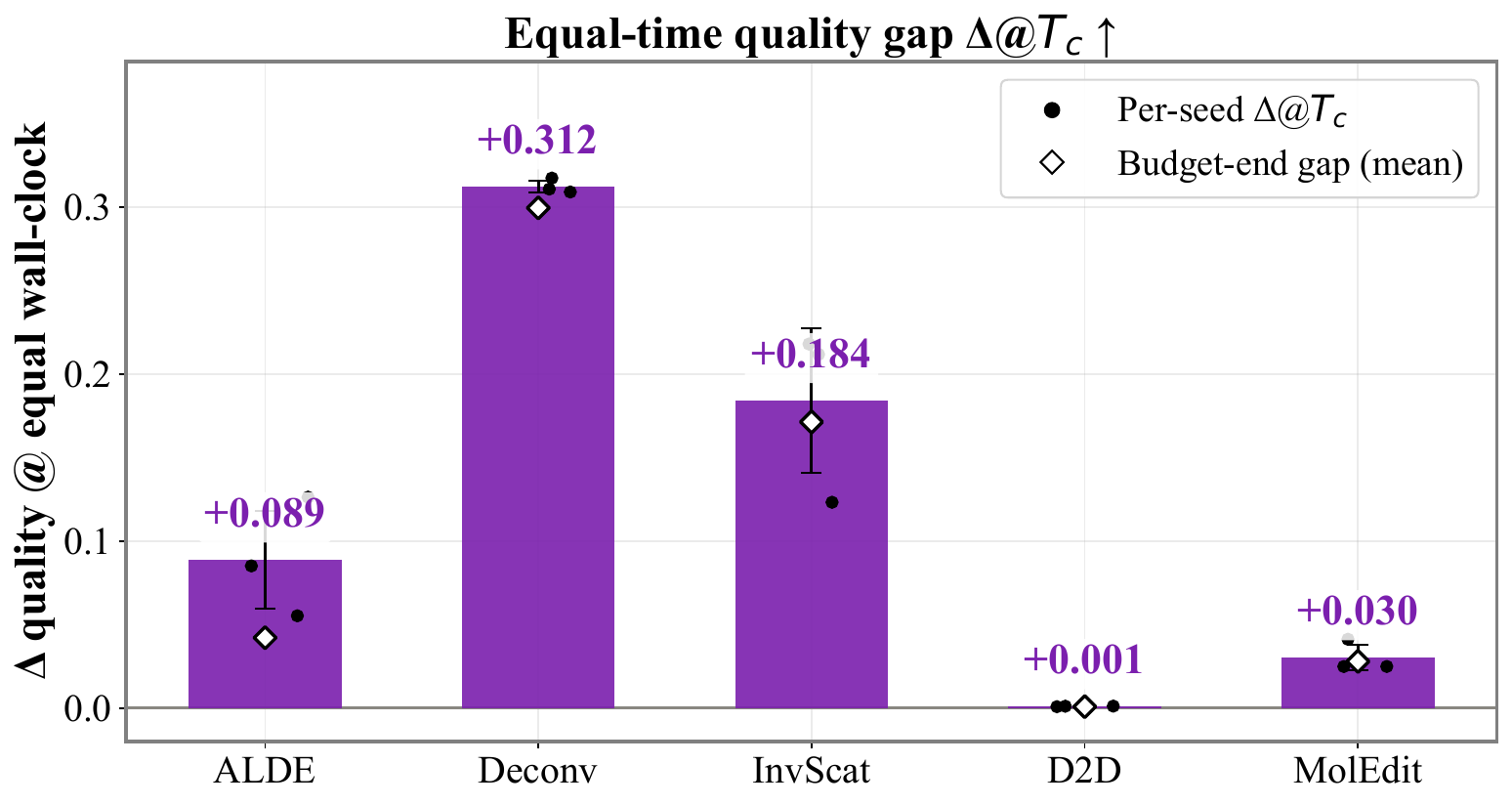}
	\caption{\textbf{Equal-time quality gap in the harness regime.} Black dots: per-seed values; white diamonds: the budget-end gap. The equal-time reading is positive on all five tasks and never below the budget-end gap.}
	\label{fig:harness_equal_time_gap}
\end{figure}

\noindent \textbf{Time-to-threshold and wall-clock compared with PiEvo~\citep{Pu2026PrincipleEvolvableSD}.}
Let $q^*$ be PiEvo's three-seed mean final quality on a task and TTT the hours until a run's running best first reaches $q^*$. All fifteen \method runs reach $q^*$ (zero censoring). Against PiEvo's own time to reach $q^*$ (self-TTT), the per-task median ratio $T_{\mathrm{TTT}}^{\mathrm{PiEvo}}/T_{\mathrm{TTT}}^{\method}$ is $0.9$/$1.2$/$0.8$/$1.3$/$1.9\times$ on ALDE/Deconv/InvScat/D2D/MolEdit (Table~\ref{tab:harness_ttt}). \method is faster on D2D and MolEdit, tied on Deconv, slightly slower on ALDE, and slowest on InvScat ($15.2$ vs $11.9$ hours), where its branches remain below the anchor for most of the budget before a late jump, the breadth-first pattern of \cref{subsec:harness_trajectory}.

\begin{table}[h!]
\centering
\caption{
    \textbf{Wall-clock conversion speed per task in the harness regime.}
    $q^*$ = PiEvo's mean final quality (raw anchor-relative $\Delta$; multiply by $100$ for the scale of Table~\ref{tab:harness_compare}); $T$ = end-to-end wall-clock (h); TTT = per-run hours until the run's running best first reaches $q^*$. \method's self-TTT ratio against PiEvo is $0.8\times$--$1.9\times$, and its end-to-end wall-clock is shorter on three of five tasks.
    }
    \label{tab:harness_ttt}
    \resizebox{0.8\columnwidth}{!}{
        \begin{tabular}{lccccccc}
        \toprule
        Task & $q^*$ & $T_{\mathrm{PiEvo}}$ & $T_{\methodm}$ & $T_{\mathrm{PiEvo}}/T_{\methodm}$ & TTT$_{\mathrm{PiEvo}}$ & TTT$_{\methodm}$ & ratio \\
        \midrule
		ALDE & +0.024 & 32.5 & 20.8 & 1.56$\times$ & 6.64 & 7.67 & 0.9$\times$ \\
		Deconv & +0.154 & 22.4 & 12.1 & 1.85$\times$ & 1.38 & 1.11 & 1.2$\times$ \\
		InvScat & -0.053 & 47.7 & 59.7 & 0.80$\times$ & 11.87 & 15.17 & 0.8$\times$ \\
		D2D & +0.154 & 31.7 & 38.0 & 0.83$\times$ & 17.05 & 13.45 & 1.3$\times$ \\
		MolEdit & +0.013 & 45.9 & 30.8 & 1.49$\times$ & 11.52 & 6.03 & 1.9$\times$ \\
        \bottomrule
        \end{tabular}
    }
\end{table}

\subsection{Trajectory Pattern}
\label{subsec:harness_trajectory}

\noindent \textbf{Breadth-first, then lead in later stage.}
\method leads at the later stage on all five tasks, but not throughout, as shown in Figure~\ref{fig:harness_anytime_evals}. Per-evaluation AUOC agrees: \method is best on ALDE (the only positive value, $+0.006$), Deconv ($0.212$, $4\times$ the runner-up), and InvScat (highest at $-0.090$), ties PiEvo on D2D ($0.1268$ vs $0.1272$, within seed noise, where Arbor is best at $0.1447$), and trails PiEvo on MolEdit.

\begin{figure}[h!]
	\centering
	\includegraphics[width=\textwidth]{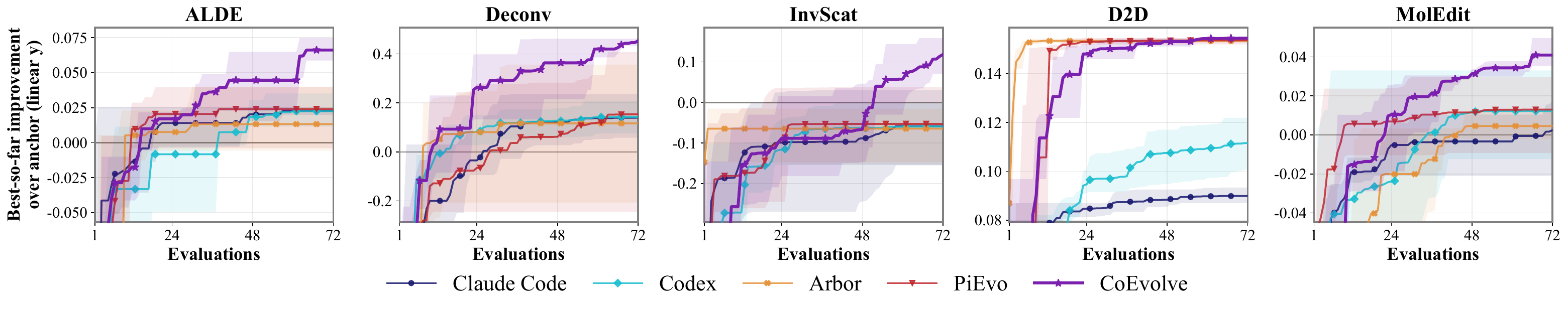}
	\caption{\textbf{Running-best $\Delta$ over the 72-evaluation budget}.}
	\label{fig:harness_anytime_evals}
\end{figure}

\begin{figure}[htbp]
	\centering
	\begin{subcaptionblock}{0.8\textwidth}
		\includegraphics[width=\textwidth]{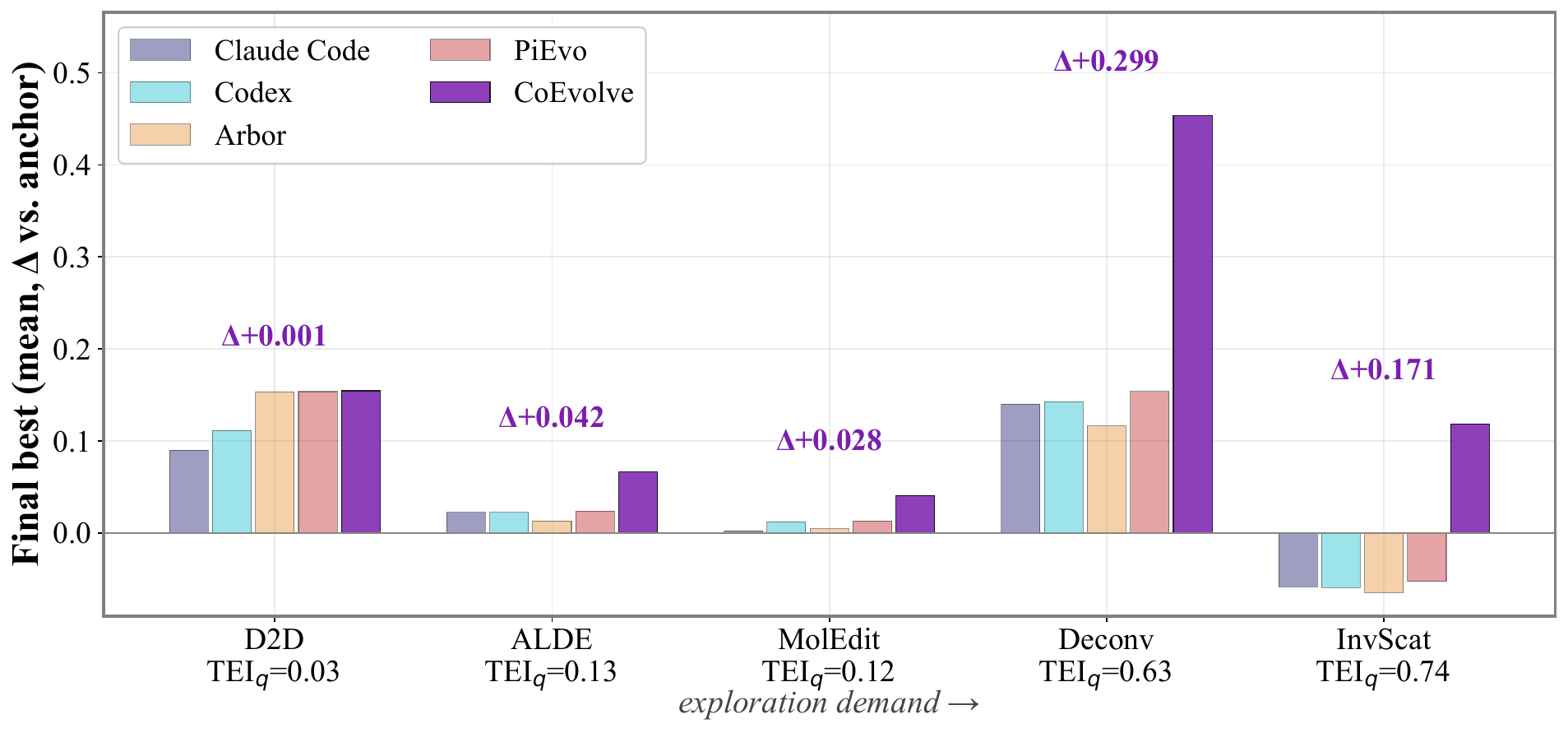}
		\caption{Final $\Delta$ by task, ordered by exploration demand.}
		\label{fig:harness_tei}
	\end{subcaptionblock}

	\vspace{26pt}

	\begin{subcaptionblock}{0.65\textwidth}
		\includegraphics[width=\textwidth]{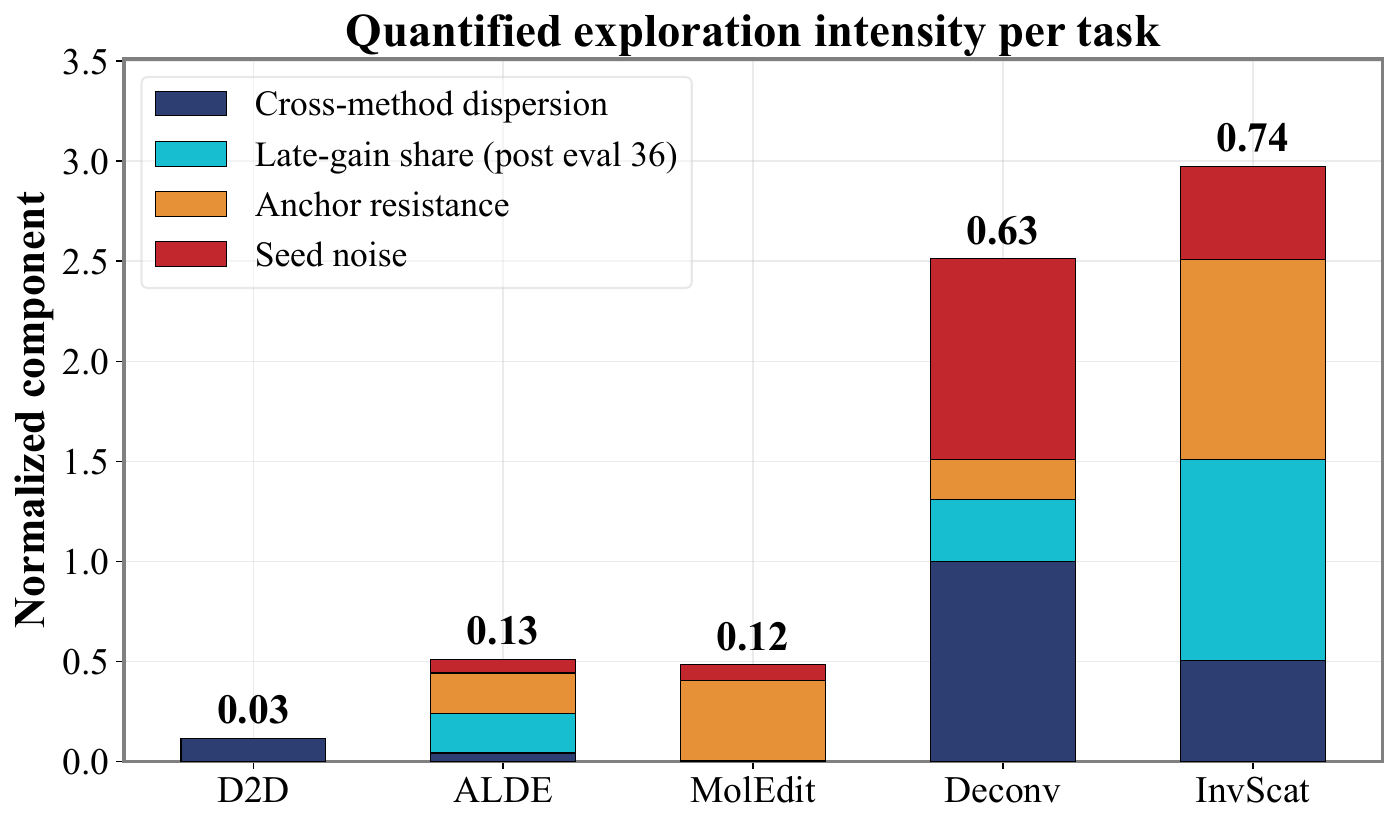}
		\caption{Quantified TEI from four behavioral proxies (illustration).}
		\label{fig:harness_tei_components}
	\end{subcaptionblock}
	\caption{\textbf{Task explorativeness moderates the collaboration margin.}}
	\label{fig:harness_tei_moderation}
\end{figure}

\subsection{Task Explorativeness Analysis}
\label{subsec:harness_tei}

The pre-registered ordering (\cref{sec:harness_design}) predicts that the collaboration margin grows with task explorativeness.

Figure~\ref{fig:harness_tei_moderation}(a) orders the five tasks by increasing exploration demand and compares each arm's final $\Delta$: the four reference arms stay bunched on the free-oracle task (D2D) and collapse toward or below the anchor as demand grows---all four are negative on InvScat---while \method alone stays positive throughout, peaking on the open method space (Deconv) and remaining the unique positive arm on InvScat.
The margin over the runner-up follows the predicted order: $+0.1$/$+2.8$ points on the free-oracle tasks (D2D/MolEdit), $+4.2$ on the recipe-saturated task (ALDE), and $+30.0$/$+17.2$ on the open method spaces (Deconv/InvScat).

Figure~\ref{fig:harness_tei_moderation}(b) quantifies the same ordering from run behavior---four proxies pooled over all runs per task, min--max normalized across tasks and averaged---yielding TEI$_q$ of $0.03/0.13/0.12/0.63/0.74$.
This quantification reads run behavior and is therefore an illustration of, not a replacement for, the pre-registered index.

\begin{remark}
    On free-oracle tasks, as shown in Table~\ref{tab:harness_compare}, throughput converts directly into score and coordination cannot add throughput, whereas open method spaces reward crossing into a better method family, which is what evidence transfer supplies.
\end{remark}

\newpage
\newpage\section{Benchmark Details}
\label{sec:benchmark}

\subsection{Design Principles}
\label{subsec:benchmark_design}

The benchmark suite serves a specific purpose: comparing budget-constrained search architectures. Two design requirements follow.

\noindent \textbf{Every task must be searchable, not memorizable.} A comparison between arms that spend the same 72-query surrogate budget is interpretable only if no arm can match search-attained quality by recalling a textbook answer from parametric memory. The suite enforces this by construction, through one of three mechanisms per task:
\begin{enumerate}[label=\alph*., leftmargin=16pt]
    \item \emph{Banned-answer tasks} (SPO, Promoter, AMP): an admission gate excludes the canonical high-scoring solution family---the record and textbook superconducting backbones in SPO, the yeast GRF/GC-box promoter architecture in Promoter, cationic amphipathic peptides in AMP---so the memorizable answer is unreachable, not merely discouraged.
    \item \emph{Surrogate-defined optima} (NHO, MBO): the objective is the argmax of a specific surrogate's response surface, which appears in no publication and therefore cannot be recalled; domain knowledge supplies priors (drug-likeness heuristics, chiral-geometry intuition) but not the answer, and MBO's gate additionally closes the one channel by which memorized motifs could inflate the score without genuine binding (PAINS assay-interference artifacts).
    \item \emph{Pool-computed optimum} (TMC): the optimum is an explicit computation over a fixed 49-ligand pool under charge neutrality, not a named entity that can be recited. The residual spaces remain scientifically non-trivial: the motif-free Promoter landscape still spans $\sim$7.7--14.7 of $17$, and the admissible SPO families retain electron-doped basins at 80--95\,K (\cref{subsec:promoter_task_benchmark,subsec:spo_task_benchmark}).
\end{enumerate}

\noindent \textbf{Every score must be an admission-gated physical property.} 
All tasks couple the LLM proposal interface to the domain surrogate of the originating study (regression models, semi-empirical physical models, or universal machine-learning force fields; cited per task below). Candidates enter the trajectory only after passing an anti-reward-hacking gate whose components apply as relevant per task: 
\begin{enumerate}[label=(\alph*), leftmargin=16pt]
    \item format and chemical validity;
    \item the task's counterfactual constraint;
    \item sequence complexity on the sequence-design tasks;
    \item a physical plausibility ceiling (oracle clamping)
\end{enumerate}

As an engineering control, it is expected to close the reward-hacking channels we identified for each task, and the per-task subsections enumerate each channel together with its check. The accounting distinguishes the two failure modes because they are different events: a \emph{gate rejection} is a search outcome, so it consumes one evaluation of the fixed budget, scores $0.0$, and can never become the running best, whereas a \emph{surrogate failure} is an infrastructure event, likewise zero-scored and excluded from best-of-seed computation. 

Inadmissible proposals are therefore paid for in opportunity cost and a run that admits no valid candidate scores $0.0$ (e.g., the zero-score cells of Table~\ref{tab:main_compare_8tasks}). Finally, the six tasks are chosen to span distinct search-space types~\citep{Pu2026PrincipleEvolvableSD}: continuous (NHO), molecular graph (MBO), peptide sequence (AMP), nucleotide sequence (Promoter), discrete combinatorial (TMC), and stoichiometric formula (SPO), and distinct landscape regimes, from the near-ceiling continuous task (NHO) through quality discrimination (MBO) and counterfactual boundary search (AMP, Promoter) to the bimodal SPO landscape. \Cref{tab:benchmark_summary} summarizes the six tasks; \cref{subsec:nho_task_benchmark}--\cref{subsec:spo_task_benchmark} give the scientific definition, search space, surrogate, and exact task specification for each.

\begin{table}[h!]
	\centering
	\caption{\textbf{Benchmark summary.} All tasks maximize the gated surrogate score. ``Anti-memorization gate'' = the per-task mechanism that prevents reaching the optimum from parametric memory (the three mechanisms of \cref{subsec:benchmark_design}); scores in native units. \textbf{Ref.\ scale} $= [y_{\mathrm{lo}}, y_{\mathrm{hi}}]$, the domain-defined absolute bounds of the task's score used to normalize SQ (\cref{subsec:evaluation_metrics}); they are fixed properties of the benchmark (physical laws, scale endpoints, or the originating benchmark's theoretical maximum), never derived from any method's runs. On NHO, AMP, Promoter, and TMC the scale is a mathematical or benchmark-defined bound, so $\mathrm{SQ} \le 100\%$ holds by construction; on MBO and SPO it is a domain-anchored reference far above every observed run (best observed: $77.1$ and $30.0$ SQ), so the bound holds empirically with wide margin.}
	\label{tab:benchmark_summary}
	\small
	\setlength{\tabcolsep}{3pt}
	\resizebox{\columnwidth}{!}{\begin{tabular}{llllll}
		\toprule
		\textbf{Task} & \textbf{Domain} & \textbf{Search space} & \textbf{Surrogate score} & \textbf{Anti-memorization gate} & \textbf{Ref.\ scale} \\
		\midrule
		NHO & nanophotonics & 7-dim continuous & $g$-factor $\in[0,2]$ & surrogate-defined optimum & $[0,\,2]$ \\
		MBO & bio-chemistry & SMILES (drug-like) & pChEMBL & surrogate-defined optimum; Lipinski + PAINS & $[0,\,12]$ \\
		AMP & biology & peptide 12--50 aa & AMP probability $\in[0,1]$ & net charge $\le 0$ (cationic recipe banned) & $[0,\,1]$ \\
		Promoter & biology & 50-nt DNA & expression bin $\in[0,17]$ & yeast GRF + GC-box banned & $[0,\,17]$ \\
		TMC & chemistry & 4-of-49 ligand comb. & polarizability ($\le 500$) & pool-computed optimum; charge neutrality, distinct ligands & $[0,\,500]$ \\
		SPO & materials & cuprate formulas & $T_c$ (K) & record/textbook backbones banned (Hg/Tl, BSCCO, YBCO) & $[0,\,298.15]$ \\
		\bottomrule
	\end{tabular}}
\end{table}

\noindent \textbf{Reference-scale design.} Normalizing raw scores by a dataset statistic (e.g.\ a training-set maximum) would make SQ a moving target that changes with the surrogate version, and normalizing by any method's best would destroy cross-paper comparability. We therefore anchor each task's scale endpoints in domain knowledge, in three tiers: 
\begin{enumerate}[label=(\alph*), leftmargin=16pt]
    \item \emph{Mathematical bounds of the score.} NHO's $g$-factor ceiling of $2$, AMP's probability bound $1$, and Promoter's expected value over the $18$ ordinal expression bins, bounded by $17$;
    \item \emph{The originating benchmark's theoretical maximum.} TMC's $500$, enforced at evaluation time by oracle clamping;
    \item \emph{Domain-anchored targets} that are references rather than bounds. MBO's pChEMBL $12$ (femtomolar binding, the practical affinity ceiling of the ChEMBL scale for reversible drug-like molecules~\citep{Gaulton2011ChEMBLAL}) and SPO's $298.15$\,K (the field's room-temperature target)
\end{enumerate}

Consequently $\mathrm{SQ} \le 100\%$ holds by construction on tiers (a)--(b) and empirically, with wide margin, on tier (c): the tier-(c) anchors sit far beyond every observed run (best observed: $77.1$ and $30.0$ SQ).

\begin{table}[h!]
	\centering
	\caption{\textbf{Final quality in native units.} Best-of-seed SQ means of Table~\ref{tab:main_compare_8tasks} multiplied by the reference scale of Table~\ref{tab:benchmark_summary}.}
	\label{tab:native_scores}
	\small
	\setlength{\tabcolsep}{5pt}
	\begin{tabular}{llcccc}
		\toprule
		& & \multicolumn{2}{c}{\textbf{Gemma-4-31B-IT}} & \multicolumn{2}{c}{\textbf{GPT-5.6-Terra}} \\
		\cmidrule(lr){3-4} \cmidrule(lr){5-6}
		\textbf{Task} & \textbf{Unit (scale)} & \method & PiEvo & \method & PiEvo \\
		\midrule
		NHO & $g$-factor ($2$) & $1.898$ & $1.898$ & $1.878$ & $1.724$ \\
		MBO & pChEMBL ($12$) & $9.25$ & $7.42$ & $8.56$ & $6.52$ \\
		AMP & probability ($1$) & $0.376$ & $0.314$ & $0.277$ & $0.238$ \\
		Promoter & expression bin ($17$) & $14.69$ & $9.25$ & $13.96$ & $13.50$ \\
		TMC & polarizability ($500$) & $441.5$ & $423.5$ & $465.0$ & $457.5$ \\
		SPO & $T_c$ in K ($298.15$) & $89.4$ & $31.6$ & $45.6$ & $31.0$ \\
		\bottomrule
	\end{tabular}
\end{table}

\subsection{Nanohelix Optical-Chirality Optimization (NHO)}
\label{subsec:nho_task_benchmark}

The NHO task~\citep{Pu2025PiFlowPS} optimizes the geometry of a helical metamaterial nanostructure and its optical excitation conditions to maximize circular-dichroic response.

\noindent \textbf{Scientific definition.} The quantity of interest is the dissymmetry ($g$-) factor, the normalized difference between left- and right-circularly polarized absorption, evaluated by an electromagnetic simulation surrogate.

\noindent \textbf{Search space.} Unlike the four-dimensional variant in PiFlow, our instantiation uses the full \textbf{seven-dimensional} parameterization: four structural parameters of the helix (fiber radius, helix radius, number of turns, pitch) plus three excitation parameters (wavelength, principal optical axis, propagation direction). The inclusion of the excitation parameters matters physically: the chiral response of a helix is strongly anisotropic~\citep{Savinov2016Toroidal,Rauschenbeutel2017ChiralQO,Faniayeu2020PolarizationCW}, and the global optimum $g\approx 2$ lies on a narrow manifold requiring aligned excitation, so the task is a genuine coupled structure--environment optimization, not geometry tuning alone. The space is bounded to the documented parameter envelope, with a $\pm 0.1$ anti-micro-tune gate against re-querying near-duplicates.

\noindent \textbf{Anti-memorization mechanism} (surrogate-defined optimum). NHO is the suite's only task with a purely continuous search space, and the only one without a banned-family constraint; it needs none, because its optimum is defined by the simulator's response surface and appears in no publication.

\noindent \textbf{Suite role.} NHO anchors the suite's ``near-ceiling'' regime (best-of-seed $\approx 1.90$ of the theoretical $2.0$).

\noindent \textbf{Reference scale.} $[0, 2]$; the dissymmetry factor satisfies $|g| \le 2$ by definition, so $2.0$ is a mathematical ceiling, not an empirical one.

\begin{taskbox}
	\small
	\textbf{Nanohelix Structure Optimization}

	\medskip
	\textbf{[Objective]} Find the parameter combination that maximizes the optical chirality of a helical nanostructure, described by its $g$-factor value $\in [0, 2]$.

	\medskip
	\textbf{[Parameters for Hypothesizing]}
	\begin{itemize}[leftmargin=*, noitemsep, topsep=2pt]
		\item \textbf{fiber-radius} (20--60 nm): radius of the fiber/wire forming the helix.
		\item \textbf{helix-radius} (20--90 nm): distance from the central axis to the center of the helical path.
		\item \textbf{n-turns} (float, 3.0--10.0): number of turns.
		\item \textbf{pitch length} (60--200 nm): axial distance between adjacent turns.
		\item \textbf{wavelength} (float, 400.0--800.0 nm): incident wavelength at which the $g$-factor is evaluated (highest chirality often occurs at 400--500 nm).
		\item \textbf{x\_y\_z} (integer 0/1/2): principal axis of the light source.
		\item \textbf{direction} (integer 0/1): propagation polarity (forward/backward).
	\end{itemize}

	\textbf{[Target Property]} $g$-factor $\in [0,2]$, higher is stronger chirality; computed by the \texttt{characterize\_nanohelix\_gfactor} tool.

	\medskip
	\textbf{[Hypothesis Format]} A single string assigning all seven parameters, e.g.\
	\begin{multicols}{2}
		\begin{itemize}[leftmargin=1.5em, noitemsep, topsep=1pt]
			\item \texttt{fiber\_radius=20.0}
			\item \texttt{helix\_radius=60.0}
			\item \texttt{n\_turns=6.0}
			\item \texttt{pitch=120.0}
			\item \texttt{wavelength=400.0}
			\item \texttt{x\_y\_z=2}
			\item \texttt{forward\_backward=0}
		\end{itemize}
	\end{multicols}
	\vspace{-0.8em}
	Perturbations of $\pm 0.1$ around previous candidates are not allowed.

	\medskip
	\textbf{[Principle Constraints]} Principles must state testable parameter--performance trends without numerical values, and must stay at the parameter level (no quantum/electronic meta-principles).
\end{taskbox}

\subsection{Molecular Bio-Activity Optimization (MBO)}
\label{subsec:mbo_task_benchmark}

MBO~\citep{Pu2025PiFlowPS} searches the space of drug-like small molecules, represented as SMILES strings, for the highest predicted bio-activity.

\noindent \textbf{Scientific definition.} The objective is the predicted pChEMBL value---the negative logarithm of the activity IC$_50$/EC$_50$/Ki/$K_d$---so each unit corresponds to an order of magnitude in potency. The surrogate is a regression model over molecular structures trained on bio-activity data.

\noindent \textbf{Search space.} SMILES strings over drug-like chemical space, delimited by the Lipinski rule-of-five (molecular weight $\le 500$ Da, H-bond donors $\le 5$, acceptors $\le 10$).

\noindent \textbf{Anti-memorization mechanism} (surrogate-defined optimum). The argmax of the trained regressor's response surface appears in no publication, so prior knowledge supplies priors but not the answer. The PAINS pan-assay-interference filter~\citep{Baell2010NewSF} closes the one memorization-adjacent hack: assay-promiscuous motifs (catechols, rhodanines, Michael acceptors) that inflate predicted activity without genuine binding score $0.0$.

\noindent \textbf{Suite role.} MBO is the benchmark's strongest ``quality-discrimination'' task: its landscape rewards structural reasoning in molecular structure-property level.

\noindent \textbf{Reference scale.} $[0, 12]$; pChEMBL $= 12$ corresponds to femtomolar affinity ($10^{-15}$\,M), the practical ceiling of reversible binding for drug-like small molecules---the strongest known drugs reach pChEMBL $\sim 10$--$11$, and essentially no drug-like molecule exceeds $12$.

\begin{taskbox}
	\small
	\textbf{Molecular Bio-Activity Optimization}

	\medskip
	\textbf{[Objective]} Discover a molecule (SMILES) with high bio-activity, quantified by maximizing its predicted pChEMBL value; only bio-activity is considered.

	\medskip
	\textbf{[Candidate Constraints]} (violations score $0.0$)
	\begin{itemize}[leftmargin=*, noitemsep, topsep=2pt]
		\item \textbf{Lipinski rule-of-five}: MW $\le 500$ Da; H-bond donors $\le 5$; H-bond acceptors $\le 10$.
		\item \textbf{PAINS filter}: no pan-assay-interference motifs~\citep{Baell2010NewSF}: false-high apparent activity is a known reward-hack path.
		\item \textbf{Chemical plausibility}: no simple polymers, repetitive chains (e.g.\ polyphenyls), or nonsensical structures; focus on novel scaffolds distinct from tested examples.
	\end{itemize}

	\textbf{[Example Record]} SMILES \texttt{"n1cccc2ccccc12"} $\rightarrow$ pChEMBL $= 2.691$.
\end{taskbox}

\subsection{Atypical Antimicrobial-Peptide Design (AMP)}
\label{subsec:amp_task_benchmark}

\noindent \textbf{Scientific definition.} AMP is an \emph{adversarial} sequence-design task over peptides of 12--50 residues in the twenty-letter amino-acid alphabet. The scoring oracle is the macrel antimicrobial-peptide classifier~\citep{SantosJnior2019MacrelAP}, whose confidence (range $[0,1]$) is maximized; the scientific interest is in probing the classifier's decision boundary away from its saturation region.

\noindent \textbf{Anti-memorization mechanism} (banned-answer). Textbook antimicrobial peptides are cationic (net positive charge) and amphipathic, and the classifier saturates on them. The task therefore inverts the textbook recipe---candidates must be \textbf{acidic} (net charge $\le 0$ at pH 7) with hydrophobic fraction in $[0.30,0.60]$---so that a high score can only be achieved by discovering which alternative physicochemical features (local hydrophobic patches, D/E--hydrophobic alternation geometry, aromatic content) the classifier responds to.

\noindent \textbf{Suite role.} AMP converts a memorization-friendly classification task into a search problem in which neither arm can rely on prior knowledge.

\noindent \textbf{Reference scale.} $[0, 1]$, the mathematical bounds of a classifier probability.

\begin{taskbox}
	\small
	\textbf{Atypical Antimicrobial-Peptide Design}

	\medskip
	\textbf{[Objective]} Discover an \emph{atypical} (acidic) peptide that the AMP classifier scores as antimicrobial, maximizing predicted AMP probability.

	\medskip
	\textbf{[Candidate Constraints]} (violations score $0.0$)
	\begin{itemize}[leftmargin=*, noitemsep, topsep=2pt]
		\item \textbf{Alphabet and length}: one-letter amino-acid code, 20 canonical residues only; length 12--50.
		\item \textbf{Acidic atypical profile}: net charge $\le 0$ at pH 7 (K,R $=+1$; D,E $=-1$; H $=+0.5$); hydrophobic fraction (A,V,L,I,M,F,W,Y,P) in $[0.30, 0.60]$; cationic (K/R-rich) amphipathic helices are rejected.
		\item \textbf{Anti-degeneracy}:
		\begin{itemize}[leftmargin=1.5em, noitemsep, topsep=1pt]
			\item no tandem repeats of any short segment; no exact repeats with period 2--4;
			\item no single-residue run $>4$;
			\item residue-composition Shannon entropy $\ge 2.0$ bits;
			\item 3-mer diversity $\ge 0.65$ (several distinct functional segments, not one copied motif).
		\end{itemize}
	\end{itemize}

	\textbf{[Guiding Hint]} Textbook AMPs use net positive charge to bind anionic membranes; that route is closed. Reason about which other physicochemical features an acidic peptide could use to be classified as antimicrobial.

	\medskip
	\textbf{[Example Record]} \texttt{"DWEFLPKGAHVDEILNWPTS"} (length 20, charge $-4$, hydrophobic fraction 0.50, 3-mer diversity 0.94) $\rightarrow$ AMP probability $0.34$.
\end{taskbox}

\subsection{Promoter Expression Optimization (Promoter)}
\label{subsec:promoter_task_benchmark}

\noindent \textbf{Scientific definition.} The Promoter task searches a 50-nucleotide DNA window for the strongest predicted transcriptional expression, scored by the LegNet surrogate trained on the DREAM2022 yeast-promoter expression challenge~\citep{deAlmeida2022GenerationOS}; the score is an expected expression bin, roughly $[0,17]$.

\noindent \textbf{Anti-memorization mechanism} (banned-answer). The dominant expression drivers in this yeast-trained model are the nucleosome-displacing general regulatory factors (GRFs) REB1, ABF1, and RAP1, together with GC-box-like cores---the memorizable textbook architecture. The task therefore bans the GRF consensus motifs on both strands and all GC-box cores (exact strings in the task box below), making it a counterfactual design problem.

\noindent \textbf{Residual landscape.} The reachable motif-free landscape still spans $\sim$7.7--14.7, so high scores remain achievable but require discovering non-canonical expression drivers: positional effects of partial motifs, A/T- versus G/C-rich region balance, initiator-like elements.

\noindent \textbf{Non-degeneracy gate.} A sequence-complexity gate (GC content $[0.30,0.70]$; no homopolymer run $>6$; no poly(dA:dT) tract $>8$, the yeast nucleosome-exclusion cheat; no di-nucleotide $>35\%$; combined CG+GC di-mer fraction $\le 45\%$; no period-2--6 repeats; no alternating di-nucleotide block $\ge 6$; 3-mer diversity $\ge 0.55$) prevents degenerate repeat padding that inflates the score without genuine design. Sequences of 45--55 nt are centre-cropped to exactly 50 as a pure format normalization.

\noindent \textbf{Generation note.} Because LLM-generated DNA drifts into repetition, this task uses generation temperature $1.1$ (vs.\ $0.6$ elsewhere), determined by an offline yield scan.

\noindent \textbf{Reference scale.} $[0, 17]$; the surrogate outputs an expected value over the DREAM2022 challenge's $18$ ordinal expression bins (indexed $0$--$17$), so $17$ is a mathematical bound of the score, not an empirical maximum.

\begin{taskbox}
	\small
	\textbf{Promoter Expression Optimization (Counterfactual)}

	\medskip
	\textbf{[Objective]} Discover a 50-nt DNA sequence driving strong predicted expression \emph{without} the yeast GRF motifs (REB1/ABF1/RAP1) or any GC-box variant---the memorizable drivers of the yeast-trained surrogate.

	\medskip
	\textbf{[Candidate Constraints]} (violations score $0.0$)
	\begin{itemize}[leftmargin=*, noitemsep, topsep=2pt]
		\item \textbf{Length}: 50-nt window, alphabet \{A,C,G,T\}; 45--55 nt is centre-cropped to 50, anything else scores $0.0$.
		\item \textbf{Motif ban}, on both strands and including near-matches:
		\begin{itemize}[leftmargin=1.5em, noitemsep, topsep=1pt]
			\item REB1: \texttt{TTACCCG} / \texttt{CGGGTAA}
			\item ABF1: \texttt{TCGGTAA} / \texttt{TTACCGA}
			\item RAP1: \texttt{ACACCC} / \texttt{GGGTGT}
			\item GC-box cores: \texttt{GGCG}, \texttt{CGCC}, \texttt{GCGC}, \texttt{CGCG} (subsuming the 6-mers \texttt{GGGCGG} / \texttt{CCGCCC})
		\end{itemize}
		\item \textbf{Non-degeneracy}:
		\begin{itemize}[leftmargin=1.5em, noitemsep, topsep=1pt]
			\item GC content in $[0.30, 0.70]$; no homopolymer run $>6$; no poly(dA:dT) tract $>8$;
			\item no dinucleotide $>35\%$; combined CG+GC di-mer fraction $\le 45\%$;
			\item no period-2--6 repeats; no alternating di-nucleotide block $\ge 6$; 3-mer diversity $\ge 0.55$.
		\end{itemize}
		\item \textbf{Construction procedure}: assemble as five 10-nt blocks with different compositional sketches, then self-check for repeated 6-mers, periodicity, GC balance, and motif bans before submitting.
	\end{itemize}

	\textbf{[Example Record]} (GC 0.58, 3-mer diversity 0.67, no GC-box) $\rightarrow$ expression score $14.38$:

	\smallskip
	\quad\texttt{TGAGGAGCCCGGTACCGGCATACCTCTACAGTGTGTTACTACCGGACGAC}
\end{taskbox}

\subsection{Transition-Metal-Complex Polarizability Optimization (TMC)}
\label{subsec:tmc_task_benchmark}

\noindent \textbf{Scientific definition.} TMC~\citep{Song2025EvaluatingLL} designs a charge-neutral Pd(II) complex by selecting exactly four \emph{distinct} ligands from a fixed pool of 49 (24 monoanions, 25 neutrals); the objective is maximum electronic polarizability (theoretical maximum 500).

\noindent \textbf{Search space.} Charge neutrality imposes the combinatorial structure: exactly two anionic and two neutral ligands, and any repeated ligand scores $0.0$ (re-using a high-scoring ligand is reward-padding), giving $\binom{24}{2}\binom{25}{2} = 82{,}800$ valid combinations, each a discrete choice over aromatic versus aliphatic, halide versus soft donor, and bulky versus compact ligands.

\noindent \textbf{Anti-memorization mechanism} (pool-computed optimum). The optimum is an explicit computation over the fixed pool, not a named entity that can be recited. The scientific content is physical-organic---polarizability grows with delocalized $\pi$-systems (aromatic SMILES atoms), soft heavy atoms (I, Br, S), and extended electron clouds---while the selection rules require parsing charge from bracket syntax, not chemical intuition.

\noindent \textbf{Suite role.} TMC is the suite's exemplar of a purely \emph{combinatorial} space with a smooth but subtle surrogate response.

\noindent \textbf{Reference scale.} $[0, 500]$, the originating benchmark's theoretical maximum polarizability~\citep{Song2025EvaluatingLL}, enforced at evaluation time by oracle clamping.

\begin{taskbox}
	\small
	\textbf{Transition-Metal-Complex Design}

	\medskip
	\textbf{[Objective]} Select exactly four ligands from the provided pool to coordinate with a central Pd$^{2+}$, maximizing polarizability.

	\medskip
	\textbf{[Critical Constraints]} (violations score $0.0$)
	\begin{itemize}[leftmargin=*, noitemsep, topsep=2pt]
		\item \textbf{Charge neutrality}: the complex must be neutral; with a $+2$ center and a pool of $-1$/$0$ ligands, select exactly \textbf{two} charge-$(-1)$ and \textbf{two} charge-$0$ ligands.
		\item \textbf{Pool restriction}: only the 49 tabulated ligands; no external molecules.
		\item \textbf{Distinct ligands}: the four ligands must all be different; repeating a ligand is reward-padding.
		\item \textbf{Charge reading}: charges must be read from the provided table, not inferred from SMILES or chemical names.
	\end{itemize}

	\textbf{[Submission Format]} \texttt{Pd\_\{SMILES\}\_\{SMILES\}\_\{SMILES\}\_\{SMILES\}}: exact pool SMILES joined by underscores, e.g.\ \texttt{Pd\_c1ccccn1\_S(=O)(C)C\_C1=C[C-]=CC=C1\_[N-]=[N+]=[N-]}.

	\medskip
	\textbf{[Ligand Pool]} (24 monoanions, 25 neutrals)

	\smallskip
	\textbf{Monoanions ($-1$):}
	\begin{multicols}{2}
		\scriptsize
		\begin{itemize}[leftmargin=1.4em, noitemsep, topsep=1pt]
			\item saccharinate \texttt{S1(=O)(=O)[N-]C(=O)c2c1cccc2}
			\item \texttt{O=[C-]OC}
			\item cyclopentadienyl \texttt{C1=C[C-]=CC=C1}
			\item \texttt{[I-]}
			\item \texttt{[S-]c1c(c(cc(c1F)F)F)F}
			\item \texttt{C1CC(=O)[N-]C1=O}
			\item \texttt{O=N(=O)[O-]}
			\item \texttt{[C-]1=CC=C(C=C1)F}
			\item \texttt{[CH3-]}
			\item tolyl \texttt{[C-]1=CC=C(C=C1)C}
			\item \texttt{O=N[O-]}
			\item trifluoromethyl \texttt{[C-](F)(F)F}
			\item \texttt{[Cl-]}
			\item \texttt{[C-]\#N}
			\item azide \texttt{[N-]=[N+]=[N-]}
			\item \texttt{[Br-]}
			\item \texttt{[S-]C\#N}
			\item phthalimide \texttt{[N-]1C(=O)c2c(C1=O)cccc2}
			\item \texttt{[C-]1=C(F)C(=C(C(=C1F)F)F)F}
			\item \texttt{c1c(C\#[C-])cccc1}
			\item \texttt{[O-]c1ccccc1}
			\item \texttt{[F-]}
			\item \texttt{[S-]c1ccccc1}
			\item acetyl \texttt{C[C-]=O}
		\end{itemize}
	\end{multicols}

	\textbf{Neutrals ($0$):}
	\begin{multicols}{2}
		\scriptsize
		\begin{itemize}[leftmargin=1.4em, noitemsep, topsep=1pt]
			\item lutidine \texttt{n1c(cc(cc1C)C)C}
			\item \texttt{n1ccc(cc1)C}
			\item \texttt{n1c(cccc1C)C}
			\item isocyanide \texttt{[C-]\#[N+]C(C)(C)C}
			\item \texttt{[C-]\#[N+]C1CCCCC1}
			\item \texttt{[C-]\#[N+]c1c(C)cccc1C}
			\item phosphine \texttt{CP(C)c1ccccc1}
			\item \texttt{CP(C)C}
			\item \texttt{n1cccc(c1)Cl}
			\item \texttt{O} (water)
			\item \texttt{N\#CC}
			\item \texttt{S(C)C}
			\item aminopyridine \texttt{c1ccnc(c1)N}
			\item \texttt{n1ccn(c1)C}
			\item \texttt{CN1[C]N(C)C=C1}
			\item \texttt{n1ccc(cc1)N(C)C}
			\item \texttt{n1[nH]c(cc1C)C}
			\item \texttt{[C-]\#[O+]}
			\item morpholine \texttt{O1CCNCC1}
			\item \texttt{NCC}
			\item \texttt{CS(=O)C}
			\item \texttt{N}
			\item \texttt{c1ccccn1}
			\item \texttt{S(=O)(C)C}
			\item \texttt{C(C)NCC}
		\end{itemize}
	\end{multicols}

	\textbf{[Example Records]}
	\begin{itemize}[leftmargin=1.5em, noitemsep, topsep=1pt]
		\item \texttt{Pd\_c1ccccn1\_CP(C)C\_[Cl-]\_[Br-]} $\rightarrow$ $198.45$
		\item \texttt{Pd\_c1ccccn1\_S(=O)(C)C\_C1=C[C-]=CC=C1\_[N-]=[N+]=[N-]} $\rightarrow$ $235.25$
	\end{itemize}
\end{taskbox}

\subsection{Superconductor Critical-Temperature Optimization (SPO)}
\label{subsec:spo_task_benchmark}

\noindent \textbf{Scientific definition.} SPO~\citep{Pu2025PiFlowPS} searches stoichiometric cuprate formulas for the highest critical temperature $T_c$ (K).

\noindent \textbf{Search space.} The explorable families are three copper-oxide systems (La-Sr-Cu-O, La-Ba-Cu-O, Nd-Ce-Cu-O) extended by genuine two-site co-doping (rare-earth-site plus Cu-site, or two rare earths), with ratios to two decimal places.

\noindent \textbf{Anti-memorization mechanism} (banned-answer). The counterfactual admission gate is the strictest in the suite: a candidate scores $0.0$ if it combines Y, Ba, and Cu (the YBCO backbone---the single most-recited textbook superconductor---at any stoichiometry), if it combines Bi, Sr, Ca, and Cu (the BSCCO backbone, at any stoichiometry), or if it bears Hg or Tl on the Ba-Ca-Cu backbone (the Hg-1223/Tl-1223/Tl-2223 record phases, including degenerate near-miss substitutions such as a tiny Pb/Bi replacement of the apical cation), because these are the memorizable textbook answers that the $T_c$ surrogate reproduces uncritically. Formulas with fewer than four distinct elements, or ratios beyond two decimal places, likewise score $0.0$.

\noindent \textbf{Suite role.} SPO provides the benchmark's clearest \emph{bimodal landscape}: surrogate-scored basins near 27--35\,K (conventional La-based) coexist with electron-doped T$'$-phase basins near 80--95\,K, and which basin a search trajectory locks onto is decided by early exploration---the mechanism behind the diversity-collapse failure mode analyzed in the experiments.

\noindent \textbf{Reference scale.} $[0, 298.15]$\,K, the task's stated target of room-temperature superconductivity---the defining physical goal of the field.

\begin{taskbox}
	\small
	\textbf{Superconductor Critical-Temperature Optimization}

	\medskip
	\textbf{[Objective]} Discover a material (within the given copper-oxide systems) with $T_c$ as high as possible, approaching room temperature (298.15 K). Only element-first, ratio-second formulas are considered; no environmental conditions.

	\medskip
	\textbf{[Suggested Systems]} La-Sr-Cu-O, La-Ba-Cu-O, Nd-Ce-Cu-O, plus genuine two-site co-doping (rare-earth-site + Cu-site, or two rare earths) that does not recreate a forbidden backbone.

	\medskip
	\textbf{[Counterfactual Constraint]} (violations score $0.0$) Record-$T_c$ and textbook-memorizable phases are forbidden:
	\begin{itemize}[leftmargin=*, noitemsep, topsep=2pt]
		\item no Hg- or Tl-bearing Ba-Ca-Cu-O (the Hg-1223/Tl-1223/Tl-2223 record phases; a small Pb/Bi substitution of the apical cation is a degenerate hack);
		\item no formula combining Bi, Sr, Ca, Cu (the BSCCO backbone, any stoichiometry);
		\item no formula combining Y, Ba, Cu (the YBCO backbone, any stoichiometry).
	\end{itemize}

	\textbf{[Format]} (violations score $0.0$)
	\begin{itemize}[leftmargin=*, noitemsep, topsep=2pt]
		\item a single stoichiometric-formula string with $\ge 4$ distinct elements;
		\item ratios with $\le 2$ decimal digits (excessive precision is numerical reward-hacking);
		\item no qualifiers or layer annotations.
	\end{itemize}

	\textbf{[Example Record]} \texttt{"La1.85Sr0.15Cu1O4"} $\rightarrow$ $T_c = 30.8$ K.
\end{taskbox}

\end{document}